\documentclass[11pt,fleqn]{article}
\pdfoutput=1
\usepackage{fullpage}

\usepackage{titlesec}
\titleformat{\subsubsection}[runin]{\normalfont\bfseries}{\thesubsubsection}{0.7em}{}[.]

\usepackage{amssymb,amsfonts,amsmath,amssymb,mathtools,mathrsfs,sansmath,nicefrac}
\usepackage[mathscr]{euscript}

\usepackage[table]{xcolor}
\usepackage{graphicx}
\usepackage{pgfplots}
\pgfplotsset{compat=1.16}
\usepackage{tikz}
\usepackage{tikz-3dplot}
\usetikzlibrary{decorations.markings,arrows}

\usepackage[ruled,vlined,linesnumbered]{algorithm2e}

\usepackage{mdframed}

\newcommand{\vc}[1]{{\boldsymbol{#1}}}
\newcommand{\mt}[1]{{\boldsymbol{#1}}}
\renewcommand{\cal}[1]{\mathcal{#1}}

\newcommand{\norm}[1]{\lVert#1\rVert}

\newcommand{\abs}[1]{\lvert#1\rvert}
\newcommand{\op}[1]{\operatorname{#1}}

\newcommand{\T}{\mathfrak{T}}   
\newcommand{\Kc}{\cal{K}}     

\newcommand{\FC}{\mathfrak{F}}

\newcommand{\etaone}{\overline{\eta}}
\newcommand{\etamix}{\eta^\#}
\newcommand{\etabnd}{\eta^\&}

\newcommand{\psilow}{\psi_{\rm L}}
\newcommand{\psiup}{\psi_{\rm H}}

\newcommand{\msh}{\mathfrak{W}_1} 
\newcommand{\mshbar}{\mathfrak{W}}

\newcommand{\cc}{\mathtt{c}}

\newcommand{\argmax}{\mathop{\op{argmax}}}
\newcommand{\Argmin}{\mathop{\op{Argmin}}}
\newcommand{\Argmax}{\mathop{\op{Argmax}}}

\newcommand{\dd}{\mathrm{d}}

\newcommand{\indic}{\delta}
\newcommand{\gauge}{\xi}
\newcommand{\rc}{r^\circ}

\newcommand{\Rext}{\overline{\mathbb{R}}}

\usepackage{amsthm,thmtools}
\theoremstyle{plain}
\newtheorem{proposition}{Proposition}
\newtheorem{definition}[proposition]{Definition}
\newtheorem{lemma}[proposition]{Lemma}
\newtheorem{corollary}[proposition]{Corollary} 
\newtheorem{remark}[proposition]{Remark}

\newtheoremstyle{setup}{-\topsep}{}{\normalfont}{}{\bfseries}{.}{.5em}{}
\theoremstyle{setup}
\newmdtheoremenv[backgroundcolor=black!10!white]{setup}{Setup}

\newtheoremstyle{theorem}{-\topsep}{}{\normalfont}{}{\bfseries}{.}{.5em}{}
\theoremstyle{theorem}
\newmdtheoremenv[backgroundcolor=cyan!10!white]{theorem}{Theorem}

\newcommand\blfootnote[1]{
  \begingroup
  \renewcommand\thefootnote{}\footnote{#1}
  \addtocounter{footnote}{-1}
  \endgroup
}

\newcommand{\para}{\medskip$\diamond$ }

\numberwithin{equation}{section}

\usepackage[backref=page]{hyperref}
\hypersetup{
    colorlinks=true,
    linkcolor=blue!50!black,
    filecolor=magenta,      
    urlcolor=cyan,
    citecolor=green!50!black,
}

\title{Persistent Reductions in Regularized Loss Minimization\\ for Variable Selection}

\author{Amin Jalali\blfootnote{Date: November 2020. Email: \texttt{amjalali16@gmail.com}}}
\date{}

\begin{document}
\maketitle
\pagenumbering{gobble}

\begin{abstract}
In the context of regularized loss minimization with polyhedral gauges, we show that for a broad class of loss functions (possibly non-smooth and non-convex) and under a simple geometric condition on the input data it is possible to efficiently identify a subset of features which are guaranteed to have zero coefficients in all optimal solutions in all problems with loss functions from said class, before any iterative optimization has been performed for the original problem. This procedure is standalone, takes only the data as input, and does not require any calls to the loss function. Therefore, we term this procedure as a persistent reduction for the aforementioned class of regularized loss minimization problems. This reduction can be efficiently implemented via an extreme ray identification subroutine applied to a polyhedral cone formed from the datapoints. We employ an existing output-sensitive algorithm for extreme ray identification which makes our guarantee and algorithm applicable in ultra-high dimensional problems.  
\end{abstract}

\newpage 

\setcounter{tocdepth}{2}
\tableofcontents

\newpage 

\pagenumbering{arabic}
\section{Introduction}\label{sec:intro}

\paragraph{The Problem Statement.}
In this work, we consider regularized loss minimization problems where the regularization function is a {\em polyhedral} gauge function;  
\begin{align}\label{prob:loss-reg}
	\min_{\vc{\beta}} ~~ f( \mt{X}\vc{\beta}) + \eta\, r(\vc{\beta}), 
\end{align}
where $\mt{X}\in\mathbb{R}^{n\times p}$, $f:\mathbb{R}^n\to \Rext$ is proper and lsc, $\eta>0$, and $r:\mathbb{R}^p\to\Rext_+$ is a level-bounded polyhedral gauge function; i.e., a non-negative positively homogeneous convex function with $r(\vc{0})=0$ where sublevel sets are polytopes. For example, $r$ can be the (weighted) $\ell_1$ norm, ordered weighted $\ell_1$ norms, sum-of-top-$k$ norms, their duals, as well as sum, max, or infimal convolution of these norms, among many more examples. On the other hand, we allow $f$ to be quite general and only require some mild conditions we will present in \autoref{cond:f-star}. For example, all convex functions with a minimizer and all star-convex functions satisfy our condition. We denote the corresponding `reference' point of $f$ (to be defined later; a minimizer in the case of convex functions) by $\vc{y}$.

It is possible to simplify the form of \eqref{prob:loss-reg} for our purposes without any loss in generality. Given $\mt{M}\in\mathbb{R}^{p\times q}$, the convex hull of the columns of $\mt{M}$ is given by $\op{conv}(\mt{M}) = \{ \mt{M}\vc{\beta}:~ \vc{\beta}\in\mathbb{R}_+^q ,~ \vc{1}^T\vc{\beta} = 1 \}$, 
and the corresponding gauge function is given by 
\begin{align}
	r(\vc{\theta})
	= \gauge (\vc{\theta}; \op{conv}(\mt{M}) )
	\coloneqq& \min\{ \lambda\geq 0:~ \vc{\theta} \in \lambda \cdot \op{conv}(\mt{M}) \} \nonumber\\
	=& \min\{ \lambda \geq 0:~ \vc{\theta} = \lambda \mt{M}\vc{\beta} ,~ \vc{1}^T\vc{\beta} = 1,~  \vc{\beta}\in\mathbb{R}_+^q   \} \nonumber\\
	=& \min\{  \vc{1}^T\vc{\beta} :~ \vc{\theta} = \mt{M}\vc{\beta},~ \vc{\beta}\in\mathbb{R}_+^q     \} .
	\label{eq:poly-reg}
\end{align}
Every level-bounded polyhedral gauge function has a representation as in the above for some matrix~$\mt{M}$. 
Therefore, in studying \eqref{prob:loss-reg}, it suffices to understand the following optimization problem, 
\begin{align}\label{prob:loss-reg-lin-pos}
	\min_{\vc{\beta}} ~~ f( \mt{X}\vc{\beta}) + \eta \langle \vc{1},\vc{\beta}\rangle ~~ \op{subject~to} ~~\vc{\beta} \in\mathbb{R}_+^p ,
\end{align}
Observe that by plugging $\mt{X}\mt{M}$ in place of $\mt{X}$ in \eqref{prob:loss-reg-lin-pos} we recover the more general problem in \eqref{prob:loss-reg}, using \eqref{eq:poly-reg}. Therefore, we work with $r(\vc{\beta}) = \vc{1}^T\vc{\beta} + \indic (\vc{\beta}; \mathbb{R}_+^p)$ from now on, where $\indic$ denotes the indicator function for a set. 
\label{rem:generalize}
On the other hand, and starting from \eqref{prob:loss-reg-lin-pos}, all of our results are readily generalizable to when $\langle\vc{1},\vc{\beta}\rangle$ is replaced with $\langle\vc{d},\vc{\beta}\rangle$ for an arbitrary vector $\vc{d}\in\mathbb{R}_+^p$ with $d_i>0$. 
Note that, \eqref{prob:loss-reg-lin-pos} subsumes regularization with the $\ell_1$ norm when $\begin{bmatrix}-\mt{X}&\mt{X}\end{bmatrix}$ is used in place of $\mt{X}$. 
Moreover, further generalizations with matrix-weighted $\ell_1$ norm (structured $\ell_1$ norm) or polyhedral cone constraints are possible. We leave these extensions to the reader. 
Finally, note that if the given gauge function has unbounded level sets we can always re-define $f$, by optimizing over the unbounded directions, to arrive at \eqref{prob:loss-reg-lin-pos}. To simplify our discussions, we have removed this possibility by making a level-boundedness assumption in \eqref{prob:loss-reg}.

\paragraph{The Guarantee.}
Let us provide a high-level description of our main guarantee and algorithm. To simplify the presentation, only in this part, let us absorb $\eta$ into $f$ by dividing the objective in \eqref{prob:loss-reg-lin-pos} by $\eta>0$; hence, without any loss in generality, we can consider problems of the following form; 
\begin{align}\label{prob:loss1F}
	\min_{\vc{\beta}} ~~ \overline{f}( \mt{X}\vc{\beta}) + \langle \vc{1},\vc{\beta}\rangle ~~ \op{subject~to} ~~\vc{\beta} \in\mathbb{R}_+^p .
\end{align} 

Corresponding to any choice of $\mt{X}\in\mathbb{R}^{n\times p}$ and $\vc{y}\in\mathbb{R}^n$, we first specify a broad class of loss functions, namely $\FC(\mt{X},\vc{y})$, whose description encompasses all of our assumptions on $\overline{f}$, or equivalently on $f$ and $\eta$. 
We then show that there exists an efficient procedure (we provide one) that takes $\mt{X}$ and $\vc{y}$ as inputs and, if they satisfy a geometric condition (\autoref{condn:ver-pos}), it outputs a subspace $S(\mt{X},\vc{y})$ which {\em contains all} optimal solutions of the optimization problem in \eqref{prob:loss1F} {\em for all} $\overline{f}\in \FC(\mt{X},\vc{y})$. This subspace could then be used to reduce the feasible set for optimization by augmenting \eqref{prob:loss1F} with a new constraint of the form $\vc{\beta}\in S(\mt{X},\vc{y})$. In other words, {\em for any} choice of $\overline{f}\in \FC(\mt{X},\vc{y})$, the following problem will be equivalent to \eqref{prob:loss1F};
\begin{align}\label{prob:loss1F-S}
	\min_{\vc{\beta}} ~~ \overline{f}( \mt{X}\vc{\beta}) + \langle \vc{1},\vc{\beta}\rangle ~~ \op{subject~to} ~~\vc{\beta} \in\mathbb{R}_+^p ,~ \vc{\beta}\in S(\mt{X},\vc{y}). 
\end{align}
This reduction procedure, is exact, can be executed before any optimization of the original problem which involves $\overline{f}$, and is universal for all loss functions $\overline{f}$ in the class $\FC(\mt{X},\vc{y})$. Therefore, we term such reduction as a {\em persistent reduction} for the corresponding problem class. 

In this work, we are concerned with polyhedral gauge functions which allows for turning the original problem in \eqref{prob:loss-reg} into an equivalent formulation as in \eqref{prob:loss-reg-lin-pos}. Considering the task of identifying the support of optimal solutions for \eqref{prob:loss-reg-lin-pos} for the purpose of {\em variable selection}, the proposed procedure identifies a subspace $S(\mt{X},\vc{y})$ which is aligned with the coordinate axes. In such case, we are in fact identifying a {\em superset} of the support for all optimal solutions of \eqref{prob:loss1F} for all loss functions $\overline{f}\in \FC(\mt{X},\vc{y})$. Therefore, instead of a subspace, we equivalently provide a set $\cal{I}\subseteq \{1,\ldots,p\}$, as the superset of all supports. In that case, the augmented problem is in the following format; 
 \begin{align}\label{prob:loss1F-I}
	\min_{\vc{\beta}} ~~ f( \mt{X}\vc{\beta}) + \eta \langle \vc{1},\vc{\beta}\rangle ~~ \op{subject~to} ~~\vc{\beta} \in\mathbb{R}_+^p ,~ \beta_i = 0 ~~ \forall i\not\in \cal{I}, 
\end{align}
where we switched back to the notation in \eqref{prob:loss-reg-lin-pos} for later reference. 
For much of our presentation we stick to the formulation in \eqref{prob:loss-reg-lin-pos} and plainly state the conditions (instead of gathering all under~$\FC$); while the main theorem, namely \autoref{thm:main-detailed} contains the full description of the class, we also provide theorems about specific loss functions in the class without referring to the whole class; e.g., see \autoref{thm:main} or \autoref{sec:cor}. Nonetheless, we come back to this abstraction in \autoref{def:class}, right after stating our main theorem, where we define $\FC(\mt{X},\vc{y})$ and provide one of its subsets.

\label{parag:3q-beginning}
With this result we now face two important questions on practicality and applicability, namely (i) whether the resulting $S(\mt{X},\vc{y})$ or $\cal{I}$ are small enough for the reduction to worth the overhead; and, (ii) whether the geometric condition on $\mt{X}$ and $\vc{y}$ is natural in practical settings. 
Moreover, it is natural to ask whether the class $\FC(\mt{X},\vc{y})$ is big enough for the result to be interesting and applicable beyond a few special loss functions. We indirectly discuss these questions as we go along in \autoref{sec:intro}, but we provide a summary of on page \pageref{sec:concerns} and come back to each point after we provide the corresponding formal description.

\paragraph{A Special Case of the Main Theorem.} 
We now state a special case of the main result (which is given in \autoref{sec:conds-thms}) to illustrate the nature of our guarantee, make the forthcoming discussions more concrete, and build intuition. In fact, as the subspace-outputting procedure is universal for the aforementioned class of loss functions, even this specialized theorem fully demonstrates the procedure (up  to a slight modification for the sake of presentation.) 
Let us consider  a simple loss (least-squares loss), a simple data configuration (standardized design), and a simple regularization (gauge function for the simplex.) 
Given $\mt{X} = \begin{bmatrix} \vc{x}_1 & \cdots & \vc{x}_p \end{bmatrix}\in\mathbb{R}^{n\times p}$, $\vc{y}\in\mathbb{R}^{n}$, and $\eta>0$, consider  
\begin{align}\label{eq:lasso-pos}
	\min_{\vc{\beta}} ~\{ \norm{ \mt{X}\vc{\beta} - \vc{y} }_2^2 
	+ \eta \vc{1}^T\vc{\beta} :~	\vc{\beta} \in\mathbb{R}_+^p \}. 
\end{align}
Our main result, namely \autoref{thm:main-detailed} in \autoref{sec:conds-thms}, implies the following guarantee.  

\begin{theorem}[Main result; simplified]\label{thm:main}
Given $\mt{X} = \begin{bmatrix} \vc{x}_1& \cdots& \vc{x}_p\end{bmatrix}\in\mathbb{R}^{n\times p}$, 
$\vc{y}$ in the column space of $\mt{X}$, and $\eta>0$, consider \eqref{eq:lasso-pos}. 
Assume all $\vc{x}_i$'s are on the unit sphere, and 
\begin{align}\label{eq:cond-eta-thm1}
\eta < 2  \norm{\vc{y}}_2 - 2 (\max_{i\in[p]}\, \vc{y}^T\vc{x}_i).
\end{align} 
Consider the convex cone $\T \coloneqq \{ \sum_{i=1}^p \lambda_i (\vc{x}_i - \vc{y}/\norm{\vc{y}}_2 ) :~ \lambda_i \geq 0 \}$ and assume for some $\cal{I}\subseteq [p]$, 
if $\{t(\vc{x}_i-\vc{y}/\norm{\vc{y}}_2):~ t\geq 0\}$ is an extreme ray of $\T$ then $i\in\cal{I}$. 
Then, $\cal{I}$ is a superset of the supports of all optimal solutions for \eqref{eq:lasso-pos}; i.e., we can augment \eqref{eq:lasso-pos} with the constraints $\beta_i=0$, for all $i\not\in\cal{I}$, to arrive at an equivalent optimization problem. 
\end{theorem}
Using this guarantee, we can omit features outside of $\cal{I}$ from the optimization problem without changing any of the optimal solutions except for a zero-padding. Moreover, to identify such features, we do not need to solve the optimization problem \eqref{eq:lasso-pos} itself but only require an {\em extreme ray identification} routine for a polyhedral cone which we form from the datapoints. We elaborate on this reduction and on this routine in the sequel.

To assess the restrictiveness of the interval assumption on $\eta$ in \autoref{thm:main}, in \autoref{plt:sph-compare-etacv-add}, we examine whether the $5$-fold cross-validated $\eta_{\rm cv}$'s for a family of lasso problems belong to the interval or not. For this, we use {\tt cv.glmnet} and we specialize \autoref{thm:main} to lasso by symmetrizing $\mt{X}$ (see page \pageref{rem:generalize}) which changes the requirement on $\eta$ in \eqref{eq:cond-eta-thm1} to $\eta\leq 2\norm{\vc{y}}_2-2\norm{\vc{y}^T\mt{X}}_\infty$. Furthermore, since cross-validation sweeps over the complete range of $\eta$, without any loss in generality, we consider $\norm{\vc{y}}_2=1$. See \autoref{sec:exp-eta-cv-details} for details of the experiments. 
Observe that the boundaries are consistent with the regime $p=\exp(O(n^\kappa))$, $\kappa>0$, which makes our guarantee appealing in ultra-high dimensional problems \cite{fan2008sure}. Boundaries in the plots for $\eta_{\rm 1se}$ shift to the right but seem to demonstrate a similar effect.

\begin{figure}[h!]
    \centering
	    \begin{tikzpicture}
	  		\node[inner sep=0pt] (A)  {\includegraphics[width=.23\textwidth]{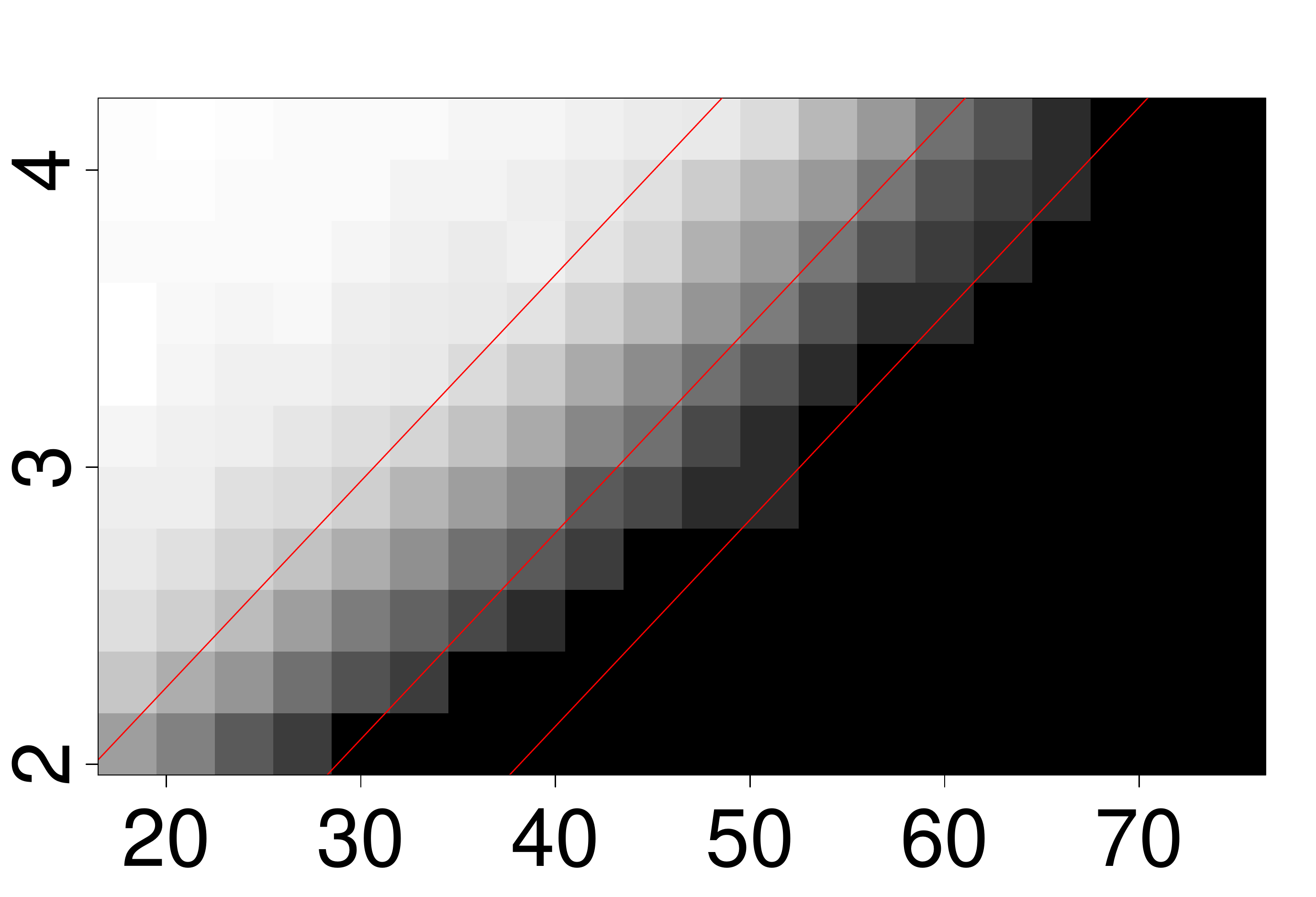}};
	  		\node[xshift=2.5pt,yshift=-11pt] (B) at ($(A.south)!.1!(A.north)$) {\scriptsize $n$};
			\node[rotate=90,yshift=12pt] (C) at ($(A.west)!.05!(A.east)$) {\scriptsize $\log_{10}(p)$};
	 	\end{tikzpicture}	
	    \begin{tikzpicture}
	  		\node[inner sep=0pt] (A)  {\includegraphics[width=.23\textwidth]{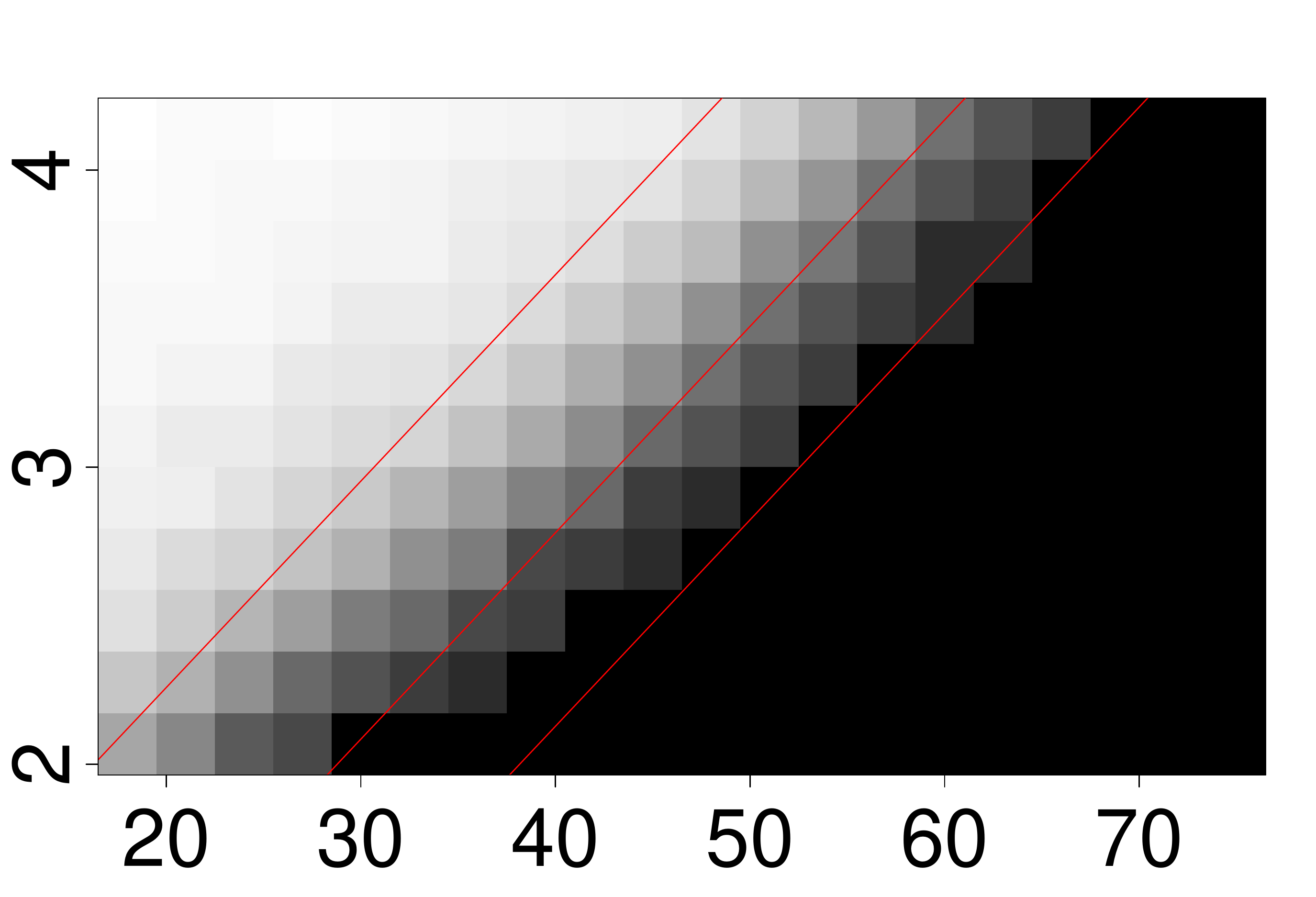}};
	  		\node[xshift=2.5pt,yshift=-11pt] (B) at ($(A.south)!.1!(A.north)$) {\scriptsize $n$};
	 	\end{tikzpicture}	
	    \begin{tikzpicture}
	  		\node[inner sep=0pt] (A)  {\includegraphics[width=.23\textwidth]{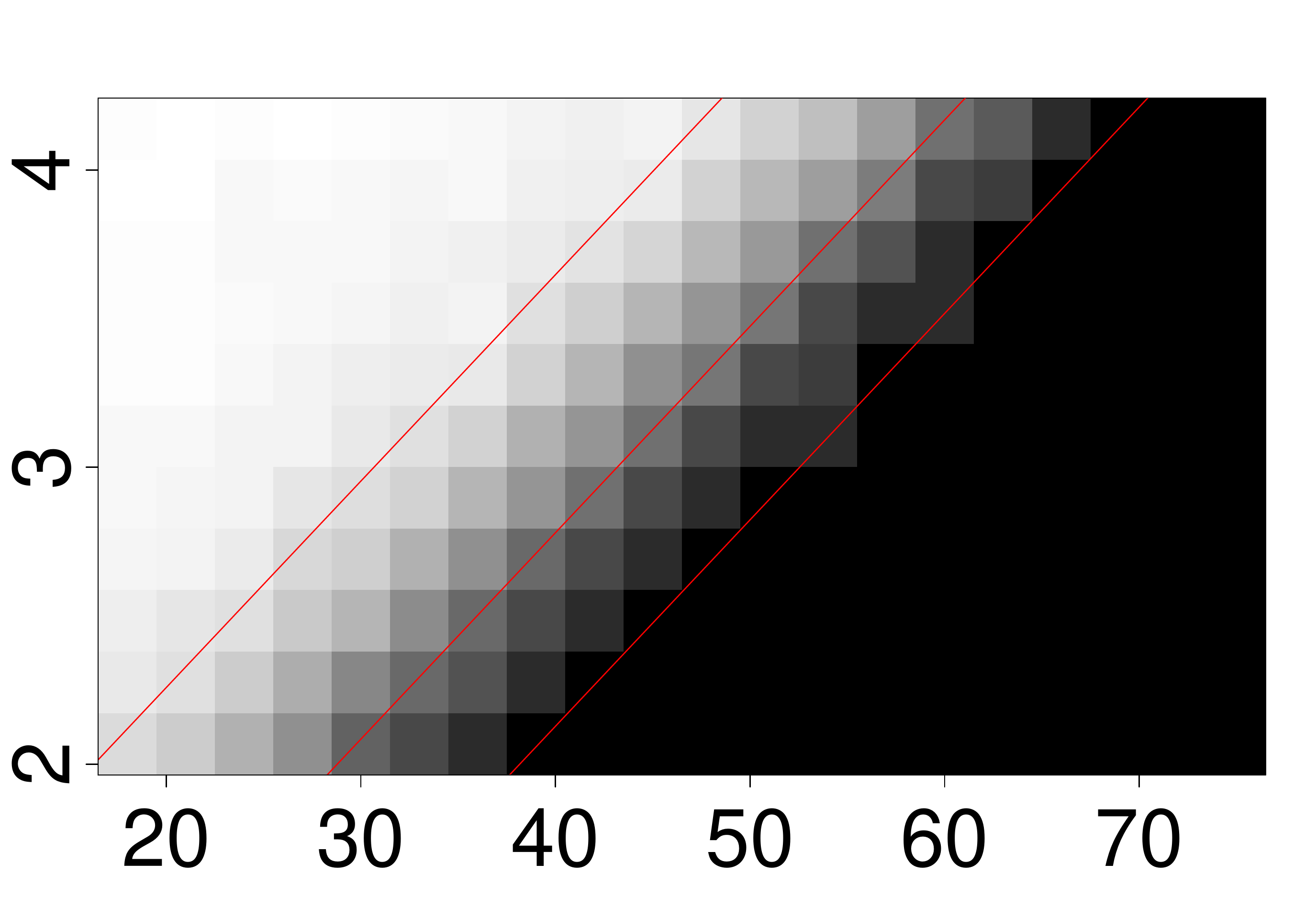}};
	  		\node[xshift=2.5pt,yshift=-11pt] (B) at ($(A.south)!.1!(A.north)$) {\scriptsize $n$};
	 	\end{tikzpicture}	
	    \begin{tikzpicture}
	  		\node[inner sep=0pt] (A)  {\includegraphics[width=.23\textwidth]{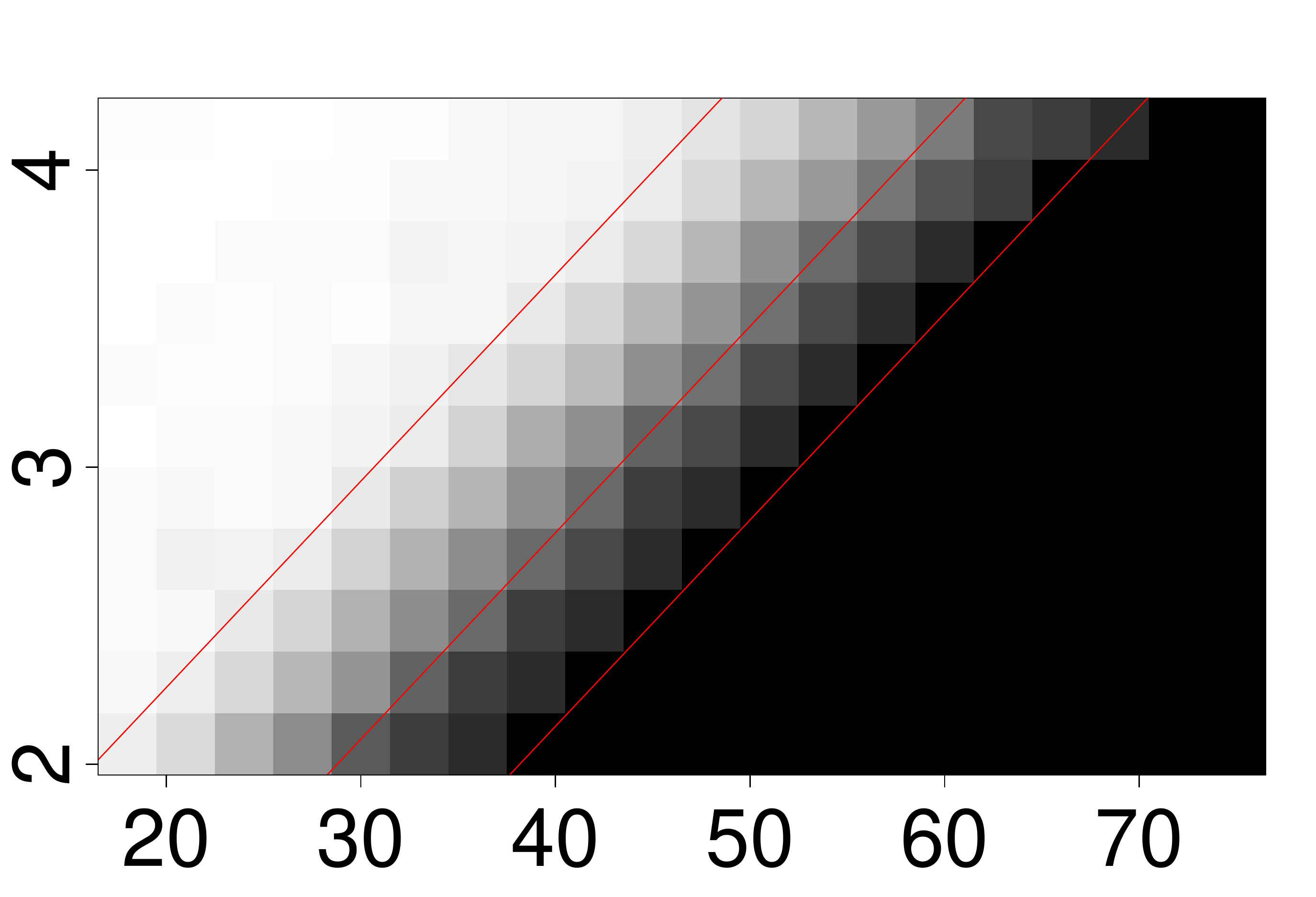}};
	  		\node[xshift=2.5pt,yshift=-11pt] (B) at ($(A.south)!.1!(A.north)$) {\scriptsize $n$};
	 	\end{tikzpicture}	
	\caption{The percentage of cases, out of $1000$ random trials, where the $5$-fold cross-validated $\eta_{\rm cv}$ for a lasso problem satisfies the requirement of our theorem. Black corresponds to $1$ and white corresponds to $0$. The horizontal and vertical axes correspond to $n$ and $\log_{10}(p)$, respectively. The lines depict $p=\exp(0.16 n-1.5)$, $p=\exp(0.16 n)$, $p=\exp(0.16 n+2)$. From left to right: $\vc{y}\sim N(\mt{X}\vc{\beta}, \sigma^2\mt{I}_n)$, for $\sigma=0.01$, $0.1$, and $1$, and $\vc{y}\sim N(\vc{0}, \mt{I}_n)$, where in all experiments, the columns of $\mt{X}$ and $\vc{y}$ were normalized before being fed into the solver. $\vc{\beta}$ draws its $k=\op{round}(\sqrt{p})$ nonzero entries independently from $N(\vc{0},\mt{I}_k)$ before being normalized to unit $\ell_2$ norm.
}
	\label{plt:sph-compare-etacv-add}
\end{figure}
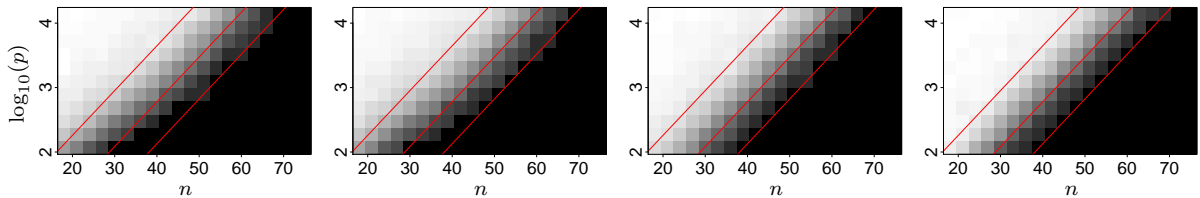

\paragraph{A Persistent Reduction.}
Beyond the simple case of quadratic loss in \eqref{eq:lasso-pos}, and for a fairly broad class of loss functions which could be non-convex, discontinuous, extended real-valued, or non-smooth, we show that if the positioning of the columns of $\mt{X}$ and $\vc{y}$, in $\mathbb{R}^n$, is favorable (as formalized in \autoref{condn:ver-pos}) and if $\eta$ is smaller than an explicit threshold (and not small in the sense of a limit; as formalized in \autoref{thm:main-detailed} or as in the weaker form in \autoref{condn:eta}) then there is a simple pre-processing procedure for certifying zero entries in all optimal solutions through a subroutine that finds the extreme rays of a polyhedral cone; similar to \autoref{thm:main}. 
Note that this reduction can be directly turned into an algorithm (\autoref{proc:pseudo}). 
Any extreme ray identification routine may be employed, but we provide an appealing example in \autoref{alg:clarkson1994more}, from \cite{clarkson1994more, Ottmann1995Enumerating}, which runs in $O(sp)$ time where $s$ is the actual number of extreme rays; see \cite{chan1996output} and \autoref{sec:ray-id}. 
The resulting reduction procedure in \autoref{proc:pseudo} is {\em exact}, it can be performed before any iterative optimization of the original problem, and it is universal for all gauge-regularized loss minimization problems as in \eqref{prob:loss-reg} where the loss satisfies some mild conditions and $\eta$ is below an explicit threshold, leading us to claim: 
\begin{quote}
This work establishes that {\em extreme ray identification has a persistent reduction property} for a broad class of loss minimization problems with polyhedral regularization.
\end{quote}
We discuss in \autoref{sec:connections} how we have arrived at this nomenclature. In short, there are connections to {\em dimensionality reduction}, {\em geometric data summarization} \cite{badoiu2002approximate, agarwal2002approximation,clarkson2010coresets, feldman2020core}, {\em screening} and {\em safe elimination} \cite{fan2010selective, Ghaoui-feature-elim}, {\em persistent relaxations} in integer programming \cite{nemhauser1975vertex,hammer1984roof}, and {\em persistency} in statistics \cite{greenshtein2004persistence}. On the other hand, there are some intrinsic differences between our goal and our method and those of the aforementioned lines of work. We elaborate on these connections in \autoref{sec:connections}.

\begin{algorithm}
 \KwData{$\mt{X}\in\mathbb{R}^{n\times p}$, $\vc{y}\in\mathbb{R}^n$, the knowledge that \eqref{prob:loss-reg-lin-pos} satisfies the conditions of \autoref{thm:main-detailed}}
 \KwResult{The set of minimizers for \eqref{prob:loss-reg-lin-pos}.}
 Find the indices $\cal{I}\subseteq [p]$ corresponding to extreme rays of $\T \coloneqq \{ \sum_{i=1}^p \lambda_i (\vc{x}_i - \vc{y} ) :~ \lambda_i \geq 0 \}$; e.g., via \autoref{alg:clarkson1994more}\;
 Remove the columns of $\mt{X}$ not indexed by $\cal{I}$ and solve \eqref{prob:loss-reg-lin-pos} with the reduced matrix\;
 Map all solutions to $\mathbb{R}^p$ by zero-padding at coordinates not in $\cal{I}$.
 \caption{Persistent Reduction of \eqref{prob:loss-reg-lin-pos} via Extreme Ray Identification}
 \label{proc:pseudo}
\end{algorithm}

As we discuss in \autoref{sec:ray-id}, the extreme ray identification can be performed in an efficient manner as in \autoref{alg:clarkson1994more} where the runtime scales linearly with $p$ and the coefficient depends polynomially on the {\em actual number of extreme rays}; hence it is an {\em output-sensitive} procedure.   
The extreme ray identification procedure could be trivially distributed through a split-apply-combine approach; a parallel to the composability property \cite{indyk2014composable} for coresets, where we combine sets of extreme rays to get a superset of the optimal supports (which can be further reduced by finding the extreme rays of the union.) Many more speedups and approximations for our choice of extreme ray identification module are possible, e.g., \cite{chan1996output}, and other implementations could also be used, but we postpone these to future works.

\newcommand{\extrayid}{\mathtt{ExtRayID}}

\begin{algorithm}
 \KwData{
 		$\cal{Z} = \{\vc{z}_1,\vc{z}_2,\ldots, \vc{z}_p \}
 		\subset\mathbb{R}^n\setminus\{\vc{0}\}$, 
 		$\vc{g}\in\op{rel\,int}(\cal{Z}^\star)$
 		\tcp*[f]{\footnotesize i.e., $\langle\vc{g},\vc{z}_k\rangle >0 ~\forall k\in[p]$}
 		}
 \KwResult{$\cal{I}$ where $\vc{z}_i$ is an extreme ray of $\op{pos}\op{conv}(\cal{Z})$ if and only if $i\in\cal{I}$ }
 Initialize: $\cal{I} = \emptyset$, $\cal{R} = \{1,\dots,p\}$ \;
 \For {$j\in \cal{R}$}{
		\eIf{$\vc{z}_j\in \op{pos}(\{ \vc{z}_i:~i\in\cal{I} \})$}{\label{line:if-in-pos}
			$\cal{I}\gets \cal{I}\cup\{j\}$\label{line:Ij}
			\tcp*[r]{\footnotesize keeping the equivalent extreme rays}
			$\cal{R} \gets \cal{R} \setminus \{j\}$
		}{ 
 		  $(\theta, \vc{v})\gets 
 		  (\max_{\vc{v}}, \argmax_{\vc{v}})
 		  \{ \langle \vc{v}, \vc{z}_j \rangle :~ 
 		  \langle \vc{v}, \vc{z}_i \rangle \leq \vc{0}~\forall i\in\cal{I},~ 
 		  \langle \vc{v}, \vc{z}_j \rangle \leq 1 \}$ \;
 		  
		  \eIf{$\theta=0$}{
		   $\cal{R} \gets \cal{R} \setminus \{j\}$ \;
		   }(\tcp*[f]{\footnotesize 
		   $\theta=1$, $\vc{v}\in \cal{Z}_{\cal{I}}^\circ$, $\vc{v}\not\in \cal{Z}^\circ$}){
		   $\cal{J}^{j} \gets \Argmax_{k\not\in \cal{I}} 
		   (\langle\vc{v}, \vc{z}_k\rangle / \langle\vc{g}, \vc{z}_k\rangle)$\label{line:Jj}
		   \tcp*[r]{\footnotesize defining a face of $\cal{Z}$}
		   $\cal{I}^{j} \gets \extrayid (\{ \vc{z}_k:~k\in\cal{J}^j \} , \vc{g})$\label{line:extrayid-Ij} \;
		   $\cal{I} \gets \cal{I} \cup \cal{I}^{j}$\label{line:IIj} \;
		   $\cal{R} \gets \cal{R} \setminus \cal{J}^{j}$ \;
		  } 		
		 }
	}

 \caption{$\extrayid$ for Extreme Ray Identification (\cite{clarkson1994more}) }
 \label{alg:clarkson1994more}
\end{algorithm}

Our procedure in \autoref{proc:pseudo} may be used for problems of the form \eqref{prob:loss-reg-lin-pos} with streaming, distributed, or dynamically generated features \cite{perkins2003online, muthukrishnan2005data, zhou2006streamwise, wu2012online}, ultra-high dimensional problems in statistics as in genetic microarrays or medical imaging applications \cite{hall2005geometric, candes2007dantzig, fan2008sure,fan2009ultrahigh,fan2010selective}, 
problems requiring fine discretization \cite{bhaskar2013atomic}, 
semi-infinite programming \cite{blankenship1976infinitely, still2001discretization}, 
or when many similar problems with perturbations of the same data set are being solved,  among others. \label{parag:app}
In these examples, either $n$ is fixed and $p$ is growing, or $p$ grows at a much faster rate than $n$; e.g., exponential in powers of $n$ as in the ultra-high dimensional regime of  \cite{fan2008sure}. 
As an example, consider a problem in ultra-high dimensions with a complicated loss function. Using \autoref{proc:pseudo}, we can reduce this problem to solving $p$ linear or quadratic programs (for identifying the extreme rays), whose sizes scale with the true number of extreme rays and not $p$, followed by a potentially much smaller problem with the original loss function. These linear/quadratic programs are amenable to warm-starting, use of special data structures, etc; e.g., see \cite{chan1996output}. We view this as a major computational appeal. 
Note that some optimization algorithms solve linear programs in their iterations; e.g., the Frank-Wolfe algorithm \cite{frank1956algorithm} over a polyhedral domain. However, the output-sensitive running time of our procedure as well as the need for working with the (complicated) loss only over a reduced space could provide an important computational advantage.

\paragraph{Main Contributions.} 
(i) We establish that the extreme rays of a polyhedral cone index all coordinates that could be nonzero in any of the optimal solutions for \eqref{prob:loss-reg}. This aspect of our contribution can be seen by specializing the proofs for the quadratic loss; e.g., see \autoref{thm:LS-gendata}. The main geometric assumption on data, in \autoref{condn:ver-pos}, also reveals itself in this special case. 
(ii) 
We show that the aforementioned phenomena applies to a broad class of loss functions. The corresponding proofs are elementary, geometric, and intuitive; provided in \autoref{sec:props}. 
We provide weaker versions of these requirements in \autoref{condn:f-bnd} and \autoref{condn:eta} to aid in specializing the theorem to specific problems; e.g., as in \autoref{sec:cor}. 
(iii) Our condition on datapoints in \autoref{condn:ver-pos} is quite natural and is satisfied in many statistical or signal processing setups; e.g., when we standardize the columns of the design matrix. We offer geometric insights on \autoref{condn:ver-pos} in relation to the facial structure of $\op{conv}(\mt{X})$ as well as necessary and sufficient conditions for this assumption, in \autoref{sec:disc-geom}. 
Lastly, (iv) \autoref{proc:pseudo} offers a {\em novel computational technique} for regularized regression with complicated loss functions.

From a technical standpoint, our guarantee (i) relies on the growth of the loss function (in a certain sense exemplified in \autoref{condn:f-bnd} via an envelope condition) compared to the gauge regularizer, and (ii) assumes a mild shape condition on the loss which generalizes star-convexity. Our condition on $\eta$ seems to be mild (e.g., see \autoref{plt:sph-compare-etacv-add}) but also improvable; we provide some concrete evidence in \autoref{rem:eta-bound}. The main geometric condition on data is however both simple and intriguing. We elaborate on this condition in \autoref{sec:originT} and take a Euclidean geometric approach rather than relying on notation from variational analysis. Our result could be viewed as a substantial generalization of the result in 
\cite{jalali2017subspace} (Supplementary Material; Lemma 17) which used the cone $\T$ to provide a guarantee for $\ell_1$ minimization with an affine constraint whose matrix has normalized columns. We review this result in \autoref{sec:other-PR} and provide further insight using the results developed in \autoref{sec:originT}.

Note that the effect of regularization has been mostly illustrated and studied through the notion of tangent cones in the feature space, namely $\mathbb{R}^p$; e.g., see illustrations such as \cite[Figure 3.11]{hastie2009elements}. 
Such studies commonly rely on KKT conditions for optimality but focus on probabilistic models of data to provide guarantees for recovery. In this work, we consider similar regularized loss minimization problems (allowing for non-convexity, discontinuity, set constraints, etc, as well) but focus on $\mathbb{R}^n$. This allows us to establish a new set of properties for the optimal solutions and derive the aforementioned persistent reductions. 
On the other hand, unlike most safe screening methods, our algorithm does not operate in the dual space.

\paragraph{Practicality and Applicability.}\label{sec:concerns}
Let us revisit the questions we mentioned on page \pageref{parag:3q-beginning} on practicality and applicability of our reduction. We elaborate on these questions here and provide a summary of our arguments but we also come back to each point after we provide the corresponding formal description. We acknowledge that these questions could be pursued further, theoretically and empirically, and beyond the arguments we provide here.

\para
First, the practicality of this reduction considering the overhead caused by the extreme ray identification routine might be in question; i.e., are $S(\mt{X},\vc{y})$ and $\cal{I}$ small enough to be interesting? 
To get an intuition, we can examine the number of extreme rays in random polytopes. For example, when the columns of $\mt{X}$ and $\vc{y}$ have been drawn uniformly at random from the unit sphere, \cite[Remark 1.9]{kabluchko2020beta} establishes that the expected number of extreme rays approaches a constant, which is only a function of $n$, as $p$ grows to infinity. In preliminary experiments, we observe a similar saturation behavior when $\mt{X}$ is being generated as above and $\vc{y}=\mt{X}\vc{\beta}$ for a sparse $\vc{\beta}$. 
However, as mentioned on page \pageref{parag:app}, the reduction procedure enables working with streaming, distributed, or dynamically generated features, ultra-high dimensional problems, extremely finely (dynamically) discretized problems, and many more setups in which common algorithms for solving \eqref{prob:loss-reg-lin-pos} become inapplicable. Also see our brief overview and comparison with screening methods in \autoref{sec:connections}. Examination of real datasets is an absolute next step.

\para
Secondly, it might not be clear how restrictive the assumptions defining $\FC(\mt{X},\vc{y})$ are; in \autoref{thm:main-detailed} or \autoref{def:class}. We argue that; (i) the assumptions on $f$ (a weak form of star-convexity, as well as a growth condition) are quite mild for any practical fidelity or loss function. Star-convex functions have recently appeared in a variety of learning, statistics, and optimization studies (e.g., see \cite{nesterov2006cubic, lee2016optimizing, hinder2020near}, subsequent works, and references therein) and different forms of growth conditions are somewhat standard in many learning and optimization scenarios. (ii) The interval requirement for $\eta$ is affected by properties of $f$ and $\mt{X}$. We examine this requirement in the case of lasso with Gaussian design and additive Gaussian noise and determine regimes of $p$ and $n$ for which a K-fold cross validated parameter falls into our requirement; see \autoref{plt:sph-compare-etacv-add}. We observe that with $p\leq \exp(O(n^\kappa))$, $\kappa>0$, (the ultra-high dimensional regime \cite{fan2008sure}) the cross-validated parameter satisfies our requirement. Nonetheless, as mentioned in \autoref{sec:exp-eta-cv-details}, similar experiments can be performed to examine $\eta$ chosen by other tuning methods, beyond $K$-fold cross-validation, in connection to our requirement. Furthermore, we can go beyond the provided specification of $\FC$, and examine the support of optimal solutions from cross-validations in connection to the subspace $S(\mt{X},\vc{y})$, to in fact examine the boundaries of $\FC$ and the universality of the reduction via $S(\mt{X},\vc{y})$. With a similar goal, for the case of squared loss (or lasso), we show that the true upper bound on $\eta$ is larger than the simple upper bound given in the theorem (due to a simplistic choice we make in our proofs.) See \autoref{rem:eta-bound} for details.

\para
Third, it is not a priori clear how restrictive the geometric assumption on data in \autoref{condn:ver-pos} might be in practice. While postponing the statement of \autoref{condn:ver-pos} to the next section, we note that with standardization, a common practice in statistics, \autoref{condn:ver-pos} holds (we also need non-repetitive columns.) The standardization can be performed onto the boundary of any strictly convex set (e.g., an ellipsoid) and not just the unit sphere; \autoref{lem:ver-on-bdry}. Moreover, many matrices $\mt{X}$ in problems in signal processing (frames, dictionaries, etc) have equal-norm columns, hence automatically satisfy \autoref{condn:ver-pos}. Going one step further, we dedicate a section, \autoref{sec:originT}, to providing further understanding of this condition in relation to the facial structure of $\op{conv}(\mt{X})$.


\section{The Main Results}\label{sec:main}
We begin by reviewing notations and preliminary facts from optimization, almost entirely from \cite{rockafellar2009variational}, in \autoref{sec:notation}; expert readers may safely skip this section. In \autoref{sec:conds-thms}, we state our main conditions in \autoref{cond:f-star} and \autoref{condn:ver-pos} as well as our main theorem, namely \autoref{thm:main-detailed}. We follow up in \autoref{sec:cor} with a few corollaries which simplify the statement of \autoref{thm:main-detailed} and address a few examples; the least-squares loss, $\ell_q^q$ loss for $q\geq 1$, and a Bregman loss. The three main propositions constituting the proof of \autoref{thm:main-detailed} will be presented in \autoref{sec:props}.

\subsection{Preliminaries}\label{sec:notation}

We use the $\star$ superscript for different purposes: when used for a function it denotes the convex conjugate, when used for a norm it denotes the dual norm, when used for an optimization variable it denotes an optimal value this variable takes with respect to said problem, and when used for a cone it denotes the dual cone. We use a regular (not a superscript) symbol $\star$ to denote epi-multiplication. 

Boldface letters denote vectors ($\vc{y},\vc{x}_i,\vc{\beta},\dots$) and matrices ($\mt{X}$). Normal font is used to denote the entries of vectors ($\beta_i$), scalars ($\lambda$, $\alpha$, $\psilow$), and scalar-valued functions ($f,r,\dots$). Uppercase letters mostly denote sets ($C,A,\dots$, except for $F$ which denotes a function). 
For a natural number $p\geq 1$, define $[p] = \{1,\ldots,p\}$. The nonnegative real line is denoted by $\mathbb{R}_+$. 
The extended real line is denoted by $\Rext = \mathbb{R} \cup \{+\infty\}$. Similarly, $\Rext_+ = \mathbb{R}_+ \cup \{+\infty\}$.

For a given set $C$, denote by $\indic(\cdot\,; C)$ the indicator function for $C$. 
Its convex hull is denoted by $\op{conv}(C)$. 
Its positive hull is defined as $\op{pos}(C) \coloneqq \{\vc{0}\}\cup\{\lambda \vc{x}:~ \vc{x}\in C,~ \lambda>0\}$. 
For a convex set $C$, the tangent cone to $C$ at $\vc{x}$ is given by $T(\vc{x}; C) =\op{cl}\{\vc{w}:~ \vc{x}+\lambda\vc{w}\in C \text{ for some } \lambda>0\}$; \cite[Theorem 6.9]{rockafellar2009variational}. 
For a given matrix $\mt{X}\in\mathbb{R}^{n\times p}$, denote by $\op{col}(\mt{X})$ the column space of $\mt{X}$ in $\mathbb{R}^n$, namely 
$\op{col}(\mt{X}) = \{ \mt{X}\vc{\lambda}:~ \vc{\lambda}\in\mathbb{R}^p \}\subseteq \mathbb{R}^n$. 
Similarly, consider 
$\op{col}_+(\mt{X}) \coloneqq \{ \mt{X}\vc{\lambda}:~ \vc{\lambda}\in\mathbb{R}_+^p \}\subseteq \mathbb{R}^n$. 
Note that $\op{col}_+(\mt{X}) = \op{pos}(\op{conv}(\{\mt{X}\vc{e}_i:~i\in [p]\}))$ where $\vc{e}_i$ denotes the $i$-th standard basis vector in $\mathbb{R}^p$. 
Corresponding to \eqref{prob:loss-reg}, we define 
$\cal{L}(\vc{\beta},\eta) \coloneqq f(\mt{X}\vc{\beta}) + \eta\, r(\vc{\beta}) -\eta$. 

\paragraph{Epi-multiplication.} For $f:\mathbb{R}^n\to \Rext$, define its epi-multiplication by a scalar $\alpha>0$ as 
\begin{align*}
(\alpha \star f)(\vc{u}) \coloneqq \alpha f(\frac{1}{\alpha}\vc{u}). 
\end{align*}
Note that $(\alpha \star f)^\star = \alpha f^\star$. 

\paragraph{Gauge Functions.} A function $r:\mathbb{R}^p\to \Rext$ is called a {\em gauge} if it is a non-negative positively homogeneous convex function with $r(\vc{0})=0$; \cite[Section 15]{rockafellar1970convex}. Equivalently, $r$ is a gauge function if $r(\vc{\theta}) = \gauge(\vc{\theta}; C) \coloneqq \inf\{\lambda\geq 0:~ \vc{\theta}\in\lambda C\}$ for some non-empty convex set $C$. Note that $C$ is not unique but can be chosen as $C = \{\vc{\theta}:~ r(\vc{\theta})\leq 1\}$. 
The {\em polar} of a gauge $r$ is defined as $r^\circ(\vc{\theta}) \coloneqq \inf\{\lambda\geq 0:~ \langle \vc{\beta},\vc{\theta} \rangle \leq \lambda r(\vc{\beta})~ \forall \vc{\beta} \}$. Therefore, for all $\vc{\beta}\in \op{dom}r$ and all $\vc{\theta}\in\op{dom}r^\circ$, we have 
\begin{align}\label{eq:gauge-polar-ineq}
\langle \vc{\beta},\vc{\theta} \rangle \leq  r(\vc{\beta}) \cdot r^\circ(\vc{\theta}) .
\end{align}
As an example, $r(\vc{\beta}) = \vc{1}^T\vc{\beta} + \indic (\vc{\beta}; \mathbb{R}_+^p)$ is a gauge function with $r^\circ(\vc{\theta}) = \max\{\theta_i:~i\in[p]\}$.

\paragraph{
Optimality Conditions.}

\begin{lemma}
\label{lem:suff-attain}
Suppose $f:\mathbb{R}^n\to\mathbb{R}$ is proper, lsc, and level-bounded, with $\op{dom}f = \mathbb{R}^n$, and $C\subseteq\mathbb{R}^n$ is nonempty and closed. Then, $\min_{\vc{u}\in C}f(\vc{u})$ is finite and the solution set is nonempty and compact. 
\end{lemma}
\begin{proof}[Proof of \autoref{lem:suff-attain}]
The set $C$ being nonempty and closed is equivalent to its indicator function $\indic(\cdot\,; C)$ being proper and lsc, respectively. Since $f$ is proper and lsc, using \cite[Proposition 1.39]{rockafellar2009variational} and the fact that $\op{dom} f = \mathbb{R}^n$, we have that $f (\cdot)+\indic(\cdot\,; C)$ is proper and lsc. Moreover, since $f$ is level-bounded, $f (\cdot)+\indic(\cdot\,; C)$ is level-bounded; \cite[Exercise 1.41]{rockafellar2009variational}. Therefore, applying \cite[Theorem 1.9]{rockafellar2009variational} on $f(\cdot)+\indic(\cdot\,; C)$ implies that its optimal value is finite and the solution set is nonempty and compact. 
\end{proof}

\begin{lemma}
\label{lem:uniqueness-general}
Suppose $f:\mathbb{R}^n\to\mathbb{R}$ is proper and lsc with $\op{dom}f = \mathbb{R}^n$. 
Suppose $r:\mathbb{R}^p\to\Rext$ is proper, lsc, and level-bounded. 
Suppose $\mt{X}\in\mathbb{R}^{n\times p}$. 
Then, $\min_{\vc{\beta}}  f( \mt{X}\vc{\beta}) +  r(\vc{\beta})$ is finite and the solution set is nonempty and compact. 
\end{lemma}
\begin{proof}[Proof of \autoref{lem:uniqueness-general}]
Since $f$ is lsc and proper, $f(\mt{X}\cdot)$ is lsc and proper; \cite[Exercise 1.40(a)]{rockafellar2009variational}. 
Since $f(\mt{X}\cdot)$ and $r$ are both proper and lsc, using \cite[Proposition 1.39]{rockafellar2009variational}, we have that $f(\mt{X}\cdot)+r$ is lsc. On the other hand, since domain of $f(\mt{X}\cdot)$ is $\mathbb{R}^p$ and $r$ is proper, $f(\mt{X}\cdot)+r$ is proper. Moreover, since $r$ is level-bounded, $f(\mt{X}\cdot)+r$ is level-bounded; \cite[Exercise 1.41]{rockafellar2009variational}. Therefore, applying \cite[Theorem 1.9]{rockafellar2009variational} on $f(\mt{X}\cdot)+r$ implies that its optimal value is finite and the solution set is nonempty and compact. 
Note that the composition of $f$ and the linear function defined by $\mt{X}$ may not be level-bounded by itself. 
\end{proof}

For a function $f:\mathbb{R}^n\to \Rext$ and a point $\vc{u}\in \mathbb{R}^n$ at which $f$ is finite, the {\em subderivative} function $\dd f(\vc{u}):\mathbb{R}^n\to \Rext$ is defined as \cite[Definition 8.1]{rockafellar2009variational}, 
\[\dd f(\vc{u})(\vc{w}) \coloneqq \liminf_{\tau \searrow 0, \vc{v}\to\vc{w}}\frac{1}{\tau}(f(\vc{u}+\tau \vc{v})-f(\vc{u})).\] 
The necessary optimality condition in \eqref{eq:opt-conds-pos} can be established under different sets of conditions. Here, we choose subdifferential regularity.

\paragraph{Subdifferential Regularity.}
A function $f:\mathbb{R}^n\to \Rext$ is called {\em subdifferentially regular} (or in short, {\em regular}) at $\vc{u}$ if $f(\vc{u})$ is finite and the epigraph of $f$ is Clarke regular (see \cite[Definition 6.4]{rockafellar2009variational}) at $(\vc{u}, f(\vc{u}))$ as a subset of $\mathbb{R}^n \times \mathbb{R}$; quoted from \cite[Definition 7.25]{rockafellar2009variational}. 
For example, a proper lsc convex function $f:\mathbb{R}^n\to \Rext$ is regular; 
any smooth function is regular; a sum of separable functions, each of which regular, is regular; among many more. 
See \autoref{sec:facts} for a non-exhaustive list of examples of regular functions. 

For a function $f:\mathbb{R}^n\to \Rext$ and a point $\vc{u}\in \mathbb{R}^n$ at which $f$ is regular (for example, for a proper lsc convex function and a point $\vc{u}$ in its domain), 
the subderivative function is given by 
\begin{align}\label{eq:def-subderivative}
\dd f(\vc{u})(\vc{w}) = \sup\{ \langle \vc{g},\vc{w}:~ \vc{g}\in \partial f(\vc{u}) \rangle \} .
\end{align}
For example, see \cite[Definition 8.1, Theorem 8.30]{rockafellar2009variational}. In other words, regularity implies a very useful duality between subgradients and subderivatives.

\begin{lemma}[Optimality Condition]\label{lem:opt-conds-pos}
Consider the optimization problem in \eqref{prob:loss-reg-lin-pos}.  
Assume:
\begin{itemize}
\item The objective is proper and lsc, 
and the problem is feasible; 
\item $f$ is subdifferentially regular on $\op{col}_+(\mt{X})$;

\item Either $\mt{X}$ is of rank $n$, or, $f$ is convex and $\op{dom}f$ cannot be separated from the range of $\mt{X}$. 
\end{itemize}
Consider any local minimum $\vc{\beta}^\star$. Then, 
\begin{align}\label{eq:opt-conds-pos}
\dd f( \mt{X}\vc{\beta}^\star )(\mt{X}\vc{\beta}^\star ) + \eta \vc{1}^T\vc{\beta}^\star \geq 0.
\end{align} 
\end{lemma}
\begin{proof}[Proof of \autoref{lem:opt-conds-pos}]
Indicator function $\indic(\cdot\,; C)$ being proper and lsc is equivalent to $C$ being nonempty and closed, respectively. Therefore, the first assumption implies that $\tilde{f}(\vc{\beta}) \coloneqq f (\mt{X}\vc{\beta})+\indic(\vc{\beta}; \mathbb{R}_+^p) + \eta \langle \vc{1},\vc{\beta}\rangle$ is proper and lsc. 
Since $\mathbb{R}_+^p$ is Clarke regular, its indicator function is regular; e.g., see \cite[Example 7.28]{rockafellar2009variational}. 
Since $f$ is proper and lsc, if the third condition is satisfied then by \cite[Theorem 10.6 or Exercise 10.7]{rockafellar2009variational}, regularity of $f$ implies regularity of $f(\mt{X}\cdot)$. 
Then, by \cite[Corollary 10.9]{rockafellar2009variational}, for any $\vc{\beta}\in\mathbb{R}_+^p$ for which $\mt{X}\vc{\beta}\in \op{dom}f$ we have 
$\partial \tilde{f}(\vc{\beta}) = \mt{X}^T (\partial f)( \mt{X} \vc{\beta}) + \eta \vc{1} + \partial \indic (\vc{\beta}; \mathbb{R}_+^p)$, where we used \cite[Exercise 10.7]{rockafellar2009variational}. 
Consider any local optimal solution $\vc{\beta}^\star$. Since the objective is proper, generalized Fermat's rule \cite[Theorem 10.1]{rockafellar2009variational} implies that $\vc{0} \in \partial \tilde{f}(\vc{\beta}^\star)$ and $\dd \tilde{f}(\vc{\beta}^\star)\geq 0$. Using \eqref{eq:def-subderivative} and \cite[Exercise 8.14]{rockafellar2009variational} we get the claimed inequality.
\end{proof}

\begin{lemma}\label{lem:eta-upperlimit-general}
Consider \eqref{prob:loss-reg} where $f:\mathbb{R}^n\to \Rext$ with $\vc{0}\in\op{dom}f$, $\eta>0$, and $r:\mathbb{R}^n\to\Rext_+$ is a nonnegative function which vanishes at $\vc{0}$. 
Assume that the following supremum exists 
	\begin{align*}
	\eta_0 \coloneqq 
	\sup_{\vc{\beta}}\, \{ { (f(\vc{0}) - f( \mt{X}\vc{\beta})) }/{ r(\vc{\beta}) } :~
	 \vc{\beta} \in \op{dom}(r) \setminus \{\vc{0}\} \}. 
	\end{align*}
	Then, $\eta\geq \eta_0$ implies $\vc{\beta}^\star = 0$ is an optimal solution.
\end{lemma}
\begin{proof}[Proof of~\autoref{lem:eta-upperlimit-general}]
Note that $\eta \geq \eta_0$ implies 
$f(\mt{X}\vc{\beta}) + \eta r(\vc{\beta}) \geq f(\mt{X}\vc{\beta}) + \eta_0 r(\vc{\beta}) \geq f(\vc{0}) = f(\vc{0}) + \eta \, r(\vc{0})$ for all nonzero $\vc{\beta}\in\op{dom}(r)$, where we used $r(\vc{0})=0$. Therefore, $\vc{0}$ is an optimal solution. 
\end{proof}

\begin{lemma}\label{lem:eta-0-convex}
Consider \eqref{prob:loss-reg} where $f:\mathbb{R}^n\to\Rext$ is convex and $r:\mathbb{R}^n\to\Rext_+$ is a gauge function where $r(\vc{\beta})=0$ implies $\vc{\beta}=\vc{0}$. Then, $\eta > \inf\{ r^\circ(-\mt{X}^T\vc{g}):~ \vc{g}\in\partial f(\vc{0})\}$ implies $\vc{\beta}^\star = \vc{0}$ is the unique solution.   
\end{lemma}
\begin{proof}[Proof of \autoref{lem:eta-0-convex}]
Convexity of $f$ and the inequality in \eqref{eq:gauge-polar-ineq} yield $f(\mt{X}\vc{\beta})
\geq f(\vc{0}) + \langle \vc{g}, \mt{X}\vc{\beta}\rangle
\geq f(\vc{0}) - 
r^\circ(-\mt{X}^T\vc{g})\cdot r(\vc{\beta})$ for all $\vc{g}\in\partial f(\vc{0})$ and all $\vc{\beta}\in\op{dom} (r)$. Therefore, $f(\mt{X}\vc{\beta})+ \eta r(\vc{\beta}) > f(\vc{0})$ for all $\beta$ with $r(\vc{\beta})>0$. This establishes the claim. 
\end{proof}

\subsection{The Assumptions and The Main Theorem}\label{sec:conds-thms}
Before presenting our results, we discuss our two main assumptions here.

\paragraph{On the Loss.}
We require the following assumption on the loss function. This condition holds for any convex function which attains its minimum, and for any star-convex function \cite[Definition 1]{nesterov2006cubic}, but the requirements are weaker than convexity and star-convexity and seem to be mild for any {\em fidelity} or {\em loss} function in estimation. Recall that a set $C\subseteq \mathbb{R}^n$ is {\em star-convex} if there exists a point $\vc{u}_0\in C$ such that the line segments between $\vc{u}_0$ and any $\vc{u}\in C$ are contained in $C$. 

\begin{setup}
\label{cond:f-star}
Given a function $f:\mathbb{R}^n\to\Rext$ whose domain is star-convex with respect to some $\vc{y}\in\mathbb{R}^n$, assume the followings hold: 
\begin{itemize}

\item 
There exists some $\psilow\in \mathbb{R}$ for which 
\[
f( \lambda \vc{u} + (1-\lambda)\vc{y} )  - \psilow 
\leq \lambda \cdot 
( f(\vc{u})   - \psilow )
\] 
for all $\vc{u}\in \op{dom}f$ and all $\lambda \in (0,1)$. 
Equivalently, there exists some $\psilow\in \mathbb{R}$ for which $\op{epi}f\coloneqq \{(\vc{u},t):~f(\vc{u})\leq t\}$ is star-convex with respect to $(\vc{y},\psilow)$. 

\item $f$ is non-decreasing on all open line segments going out of $\vc{y}$; i.e., 
\[
f( \lambda \vc{u} + (1-\lambda)\vc{y} ) \leq f(\vc{u})
\] for all $\vc{u}\in \op{dom}f$ and all $\lambda\in (0,1)$. 
\end{itemize}
\end{setup}

Consider the sublevel set $S_{\leq} = \{\vc{u}\in\op{dom}f:~ f(\vc{u}) \leq \psilow\}$ and the superlevel set $S_{\geq} = \{\vc{u}\in\op{dom}f:~ f(\vc{u}) \geq \psilow\}$. Then, on $S_{\leq}$ the second condition implies the first and on $S_{\geq}$ the first assumption implies the second. Therefore, \autoref{cond:f-star} can be equivalently stated using these sets.

Next, suppose $f$ satisfies the first condition of \autoref{cond:f-star} with some $\psilow$. Then; 
\begin{itemize}

\item  
$\psilow = f(\vc{y})=\inf f$ is equivalent to $f$ being {\em star-convex} with respect to $\vc{y}$; see \autoref{sec:facts} for a list of examples. Note that any convex function with a nonempty set of minimizers is star-convex hence satisfies \autoref{cond:f-star} with $\psilow = \inf f$ and $\vc{y}$ being any of its minimizers. As another example, any positively homogeneous function of any order $c\geq 1$ is star-convex with respect to the origin. 

\item  
When $\psilow > \inf f$, then \autoref{cond:f-star} allows for a more relaxed behavior within the level set $S_{\leq} = \{\vc{u}\in\op{dom}f:~ f(\vc{u}) \leq \psilow\}$. See \autoref{fig:f-star} for an illustration. 

\item The first condition of \autoref{cond:f-star} can be viewed as an additive relaxation of star-convexity, while multiplicative relaxations also exist in the literature. 

\end{itemize}

Lastly, consider the closures of a partition of $[0, \infty)$ into nonempty intervals and consider functions defined on each of these closed segments where each satisfies the inequality in the first part of \autoref{cond:f-star} for some $\psilow$. Then, for the maximum of these values of $\psilow$, the concatenation of these functions also satisfies the same inequality. See \autoref{fig:f-star} for an illustration.

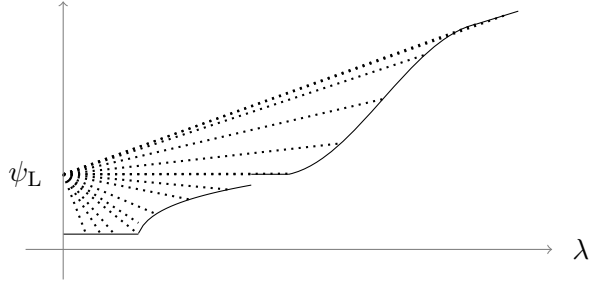
\begin{figure}[h!]
	\centering
	
	\begin{tikzpicture}
    
    \def\yz{-1.3}  
	\def\ymin{-1.7}
	\def\plo{-.3}

	\newcommand{\drawCut}[2]{
		\draw[dotted, thick] (0,\plo) -- (#1,#2);
	}	

	\newcommand*{\ExtractCoordinate}[1]{\path (#1); \pgfgetlastxy{\XCoord}{\YCoord};}	
	\newcommand{\drawBezstar}[6]{
		\draw[postaction={decorate,decoration={markings,
    			mark=at position 0.25 with {\coordinate (A);},
    			mark=at position 0.50 with {\coordinate (B);},
    			mark=at position 0.75 with {\coordinate (C);},
    			mark=at position 0.01 with {\coordinate (D);}
  						}}] 
  						(#1,#2) .. controls (#3,#4) and ({#1+#5-#3},{#2+#6-#4}) .. (#5,#6);
		\ExtractCoordinate{$(A)$};
		\drawCut{\XCoord}{\YCoord}
		\ExtractCoordinate{$(B)$};
		\drawCut{\XCoord}{\YCoord}
		\ExtractCoordinate{$(C)$};
		\drawCut{\XCoord}{\YCoord}
		\ExtractCoordinate{$(D)$};
		\drawCut{\XCoord}{\YCoord}
	}

	
	\draw[domain=0:1,  variable=\x, 
			decoration={post length=3mm,
                 pre length=3mm,markings,
    			mark=at position 0.01 with {\coordinate (A);},
    			mark=at position 0.50 with {\coordinate (B);},
    			mark=at position 0.95 with {\coordinate (C);}
  						}, postaction={decorate} ] 
  			plot ({\x}, {log2(1+1)/log2(3.5)-1.65   });

	\ExtractCoordinate{$(A)$};
	\drawCut{\XCoord}{\YCoord}
	\ExtractCoordinate{$(B)$};
	\drawCut{\XCoord}{\YCoord}
	\ExtractCoordinate{$(C)$};
	\drawCut{\XCoord}{\YCoord}
	\drawCut{1}{{log2(1+1)/log2(3.5)-1.65}}
	\drawCut{1}{{log2(1+1)/log2(3.5)-1.5}}

	\draw[domain=1:2.5, variable=\x, 
			decoration={post length=3mm,
                 pre length=3mm,markings,
    			mark=at position 0.05 with {\coordinate (A);},
    			mark=at position 0.50 with {\coordinate (B);},
    			mark=at position 0.99 with {\coordinate (C);}
  						}, postaction={decorate} ] 
				plot ({\x}, {sqrt(sqrt(\x-.95))-1.56}); 
	\ExtractCoordinate{$(A)$};
	\drawCut{\XCoord}{\YCoord}
	\ExtractCoordinate{$(B)$};
	\drawCut{\XCoord}{\YCoord}
	\ExtractCoordinate{$(C)$};
	\drawCut{\XCoord}{\YCoord}

	\draw (2.5,-.3)--(3,-.3);
	\drawCut{2.5}{-.3}
	\drawBezstar{3}{-.3}{3.9}{-.1}{5.5}{1.7}

	\draw (5.5,1.7) -- ({5.5*1.1},{1.7*1.1});
	\drawCut{5.5}{1.7}
	\drawCut{5.5*1.07}{1.7*1.07}
	
	\draw [gray, ->] (-.5,\yz) -- (6.5,\yz) 
		node[label={[black,label distance=0cm]0:$\lambda$}] {} ;
	\draw [gray, ->] (0,\ymin) -- (0,2);
	\node[label={[label distance=0cm]180:$\psilow$}] (psilow) at (0,\plo) {};

	\end{tikzpicture}

	\caption{An example of $f(\lambda \vc{u} + (1-\lambda)\vc{y})$, for some fixed $\vc{u}$, where $f$ satisfies \autoref{cond:f-star}. }
	\label{fig:f-star}	
\end{figure}

\paragraph{On Data.} Next, we state our assumption on data, namely the design matrix $\mt{X}$ and a  point $\vc{y}$ related to the reference point used by the loss function as characterized in \autoref{cond:f-star}.   
\begin{setup}[Vertex Non-cover Condition]\label{condn:ver-pos}
Given $\mt{X}= \begin{bmatrix}
\vc{x}_1 & \cdots & \vc{x}_p
\end{bmatrix}\in\mathbb{R}^{n\times p}$ and $\vc{y}\in\mathbb{R}^n$, assume 
$\vc{y}\neq \vc{x}_i$ for all $i\in[p]$ and 
\[
\op{ver}(\op{conv}( \{\vc{x}_1, \ldots, \vc{x}_p, \vc{y}\} )) = 
\{\vc{x}_1, \ldots, \vc{x}_p, \vc{y}\}. 
\]
\end{setup}
Let us provide a sufficient condition for \autoref{condn:ver-pos} to illustrate the naturalness of \autoref{cond:f-star} in various settings. 
Note that when the output of an estimator is not equivariant under scaling of the input data, it is common to consider a {\em standardization} step on the input; e.g., see \cite{hastie2009elements}. In such cases, \autoref{condn:ver-pos} amounts to simply insuring that $\vc{y}$ is not equal to any of $\vc{x}_1, \ldots, \vc{x}_p$. Similarly, when working with equal-norm frames, a finite difference matrix, matrices encoding flow constraints, among many other examples, the columns of $\mt{X}$ are already normalized and \autoref{condn:ver-pos} holds naturally when $\vc{y}$ is being appropriately normalized (or $\eta$ is below a threshold; as we will see in \autoref{thm:main-detailed} and \autoref{lem:perspective-props}.)  
More generally, when the columns of the design matrix $\mt{X}$ belong to the boundary of a compact convex set with a `curved' boundary, \autoref{condn:ver-pos} is easy to verify. 
\begin{lemma}[A Sufficient Condition]\label{lem:ver-on-bdry}
Suppose $\vc{x}_1,\ldots, \vc{x}_p$ are on the boundary of a strictly convex set $C$ (i.e., all boundary points of $C$ are exposed. It is enough to assume $\vc{x}_1,\ldots, \vc{x}_p$ are exposed boundary points of $C$.) 
Then, for all $i\in[p]$, we have $C \cap (\vc{x}_i - T(\vc{x}_i; \op{conv}(\mt{X}))) = \{\vc{x}_i\}$. Moreover, \autoref{condn:ver-pos} holds for $\vc{x}_1,\ldots, \vc{x}_p$ and any $\vc{y} \in  C \setminus \op{conv}(\mt{X})$. 
\end{lemma}
\begin{proof}[Proof of \autoref{lem:ver-on-bdry}]
Observe that $\vc{x}_1\in C \cap (\vc{x}_1 - T(\vc{x}_1; \op{conv}(\mt{X})))$. 
Contrapositively, suppose there exists $\vc{z}\neq \vc{x}_1$ where $\vc{z}\in C \cap (\vc{x}_1 - T(\vc{x}_1; \op{conv}(\mt{X})))$. Therefore, similar to the proof of \autoref{lem:negTangent}, we can conclude that $\vc{x}_1$ is in the convex hull of $\{\vc{z}, \vc{x}_2, \ldots, \vc{x}_p\}$. On the other hand, since $\vc{x}_1$ is an exposed point of $C$, there exists a hyperplane that separates it from the rest of $C$, and specifically from $\vc{z}, \vc{x}_2, \ldots, \vc{x}_p$. This is a contradiction, which establishes the first claim. 
For the second part, we take a longer route to demonstrate the role of tangent cones and the previous claim. Take any $\vc{y} \in  C \setminus \op{conv}(\mt{X})$ and assume that $\vc{x}_1$ is not a vertex of $\op{conv}(\vc{y}, \vc{x}_1, \ldots, \vc{x}_p)$. Therefore, there exists $\vc{\lambda}\geq \vc{0}$ with $\vc{1}^T\vc{\lambda}=1$ for which $\vc{x}_1 = \lambda_1 \vc{y} + \sum_{i=2}^p \lambda_i \vc{x}_i$. Note that $\vc{x}_1$ cannot be in the convex hull of $\vc{x}_2, \ldots, \vc{x}_p$ due to the exposedness assumption; therefore, $\lambda_1$ has to be positive. 
Moreover, we equivalently have $\vc{x}_1 + \lambda_1 (\vc{x}_1-\vc{y}) = \sum_{i=1}^p \lambda_i\vc{x}_i \in \op{conv}(\mt{X})$. Therefore, since $\lambda_1>0$, $\vc{x}_1-\vc{y} \in T(\vc{x}_1; \op{conv}(\mt{X}))$ which implies $\vc{y}\in C \cap (\vc{x}_1-T(\vc{x}_1; \op{conv}(\mt{X})))=\{\vc{x}_1\}$ where we used the first part. Recall that $\vc{y}\not\in \op{conv}(\mt{X})$, hence this is a contradiction, establishing the second claim.
\end{proof}
Note that \autoref{lem:ver-on-bdry} is not necessary for \autoref{condn:ver-pos}; e.g., see the illustrations in \autoref{sec:illust}. 
We provide necessary conditions and further sufficient conditions for \autoref{condn:ver-pos} in \autoref{sec:originT}.

\paragraph{The Main Theorem.} 
With considerations in \autoref{cond:f-star} and \autoref{condn:ver-pos}, we state our main guarantee.

\begin{theorem}[Main Theorem]\label{thm:main-detailed}
Given $\mt{X} = \begin{bmatrix} \vc{x}_1 & \cdots & \vc{x}_p \end{bmatrix}\in\mathbb{R}^{n\times p}$, $f:\mathbb{R}^n\to\Rext$, and $\eta>0$, consider \eqref{prob:loss-reg-lin-pos}. 
Suppose $f$ satisfies \autoref{cond:f-star} with some $\vc{y}\in\mathbb{R}^n$ and $\psilow\in\mathbb{R}$. 
Moreover, suppose that there exists $\gamma\in[0,1]$ and $\vc{h}\in \op{dom}((f_\gamma)^\star)$, where 
\[
f_\gamma(\vc{u}) \coloneqq f( \vc{u} ) - \gamma \cdot \dd f(\vc{u})(\vc{u}),
\]
for which $\eta$ satisfies 
\begin{align}\label{eq:cond-eta}
	(1-\gamma)\eta \leq  
	-\min_{i\in[p]} \langle\vc{x}_i,\vc{h}\rangle
	<\alpha \cdot \left[ - (f_\gamma)^\star (\vc{h})  - \psilow \right] -  \gamma \eta \,,
\end{align} 
for some $\alpha >0$. 
Assume that $\mt{X}$ and $\alpha \vc{y}$ satisfy \autoref{condn:ver-pos}. 
Moreover, assume $f$ and $\mt{X}$ satisfy the assumptions of \autoref{lem:opt-conds-pos}.  
Consider the convex cone 
\[
\T \coloneqq 
\bigl\{ \sum_{i=1}^p \lambda_i (\vc{x}_i - \alpha \vc{y} ) :~ \lambda_i \geq 0   \bigr\},
\]
and assume for some $\cal{I}\subseteq [p]$, if $\{t(\vc{x}_i- \alpha\vc{y}):~ t\geq 0\}$ is an extreme ray of $\T$ then $i\in\cal{I}$. 
Then, the two optimization problems in \eqref{prob:loss-reg-lin-pos} and \eqref{prob:loss1F-I} have the same optimal values and the same sets of optimal solutions; i.e., $\cal{I}$ is a superset of the supports of all optimal solutions for \eqref{prob:loss-reg-lin-pos}.  
\end{theorem}
\begin{proof}[Proof of \autoref{thm:main-detailed}]
If $\vc{\beta}^\star=\vc{0}$ is the unique solution, then the claim holds trivially; the second optimization problem has additional constraints compared to the first problem (entries not in $\cal{I}$ required to being zero) which are satisfied by the optimal solution to the first problem. Therefore, the two problems are equivalent. Therefore, suppose there exists a nonzero optimal solution $\vc{\beta}^\star$.

Consider any such optimal solution and suppose $\beta_i^\star\neq 0$. By \autoref{lemma:normless1-genloss}, if $\vc{x}_i - \vc{y} \not\in \op{ext\,ray\,}\T$ then 
$\cal{L}(\vc{\beta}^\star) \leq \psilow $ and $0< \langle \vc{1}, \vc{\beta}^\star \rangle \leq 1$. 
Combining these with the assumptions on $\eta$ in the statement of the theorem as well as the inequality in \autoref{lem:opt-pos-gen} yields a contradiction. Therefore, $\vc{x}_i - \vc{y} \in \op{ext\,ray\,}\T$ which implies $i\in \cal{I}$. Equivalently, if $i\not\in\cal{I}$, we can guarantee that $\beta^\star_i=0$ for any optimal solution $\vc{\beta}^\star$. This proves a special case of the the theorem in which $\alpha$ has been set to $1$. Then, use \autoref{lem:perspective-props} to state the theorem for any $\alpha>0$. 
\end{proof}
A few remarks are in order. 
\begin{itemize}

\item 
There is no need for computing $(f_\gamma)^\star$ in \autoref{thm:main-detailed}. Any lower bound on $f_\gamma \geq F$ provides a lower bound for $-(f_{\gamma})^\star \geq -F^\star$ which can be used in \eqref{eq:cond-eta} for a weaker result. See \autoref{sec:f-bnd}; or \autoref{cor:Lqq-regression} for an example.  

\item When $f$ grows fast out of $\vc{y}$, a smaller value of $\gamma$ helps keep $f_\gamma$ bounded from below by a `large' enough envelope function. This in turn affects $f_\gamma^\star$ and helps us get a wider interval in \eqref{eq:cond-eta}. See \autoref{condn:f-bnd}, and the discussion right after, for a simplified statement.

\item 
We exploit an important property of support recovery in \autoref{thm:main-detailed}; that for any minimizer of $\cal{L}$ there exists a minimizer of $(\alpha\star \cal{L})$ with the same support, as shown in the first assertion of \autoref{lem:perspective-props}. This is the role of $\alpha$ in the statement of the theorem.

\item Note that if $f$ is convex and defined everywhere on $\op{col}_+(\mt{X})$ and $r$ is definite, then by \autoref{lem:eta-0-convex} there exists a threshold for $\eta$ above which $\vc{0}$ is the unique optimal solution and the claim of the theorem holds trivially. Therefore, whenever we use \autoref{condn:f-bnd} and choose $\cal{H} = (1-\gamma) \partial f(\vc{0})$, as in some of the examples we present below, the left-hand side inequality in \eqref{eq:cond-eta} is not needed.

\item 
The provided interval for $\eta$ is not sharp, and it is possible for our guarantee to hold for all $\eta$, for some choices of data $\mt{X}$ and $\vc{y}$ and loss function $f$. 
Note that if the interval in \eqref{eq:cond-eta} is nonempty for some value of $\alpha$ then a larger value of $\alpha$ also satisfies these inequalities. Therefore, in situations where $\mt{X}$ and $\alpha\vc{y}$ satisfy \autoref{condn:ver-pos} for all $\bar{\alpha}\geq \alpha$ (e.g., see the illustrations in \autoref{sec:illust}), and the loss function is as in the previous item (hence the left-hand side inequality is not needed), then the superset guarantee of the theorem will hold for all $\eta$.

\item 
The support of the optimal solutions of \eqref{prob:loss-reg-lin-pos} need not be monotonic (nested) with respect to $\eta$, i.e., entries may fall in and out of support as $\eta$ grows. For example, this is well-known for lasso. Therefore, our result does not generalize for all $\eta$ in general even though increasing $\eta$ (above the stated upper bound) increases regularization and generally yields smaller supports. Nonetheless, we do not believe that the current requirement is necessary (when considering all $\mt{X}$ and $\vc{y}$ and not as in the previous item), as discussed in \autoref{rem:eta-bound} for the quadratic loss.

\item 
Our result is in a regime of small regularization parameter; not in a limit sense, and only in the sense of being below an explicit threshold. See the experimental results reported in \autoref{plt:sph-compare-etacv-add} which compares the upper bound in \autoref{thm:main} with a cross-validated value of $\eta$.

\item The assumptions and the proof of \autoref{thm:main-detailed} in some sense decouple $f$ from the gauge function corresponding to $\op{conv}(\mt{X})$ in the following alternative representation for \eqref{prob:loss-reg-lin-pos}; 
\begin{align}\label{eq:gauge-reform}
\min_{\vc{w}} ~~ f(\vc{w}) + \eta \cdot \gauge(\vc{w}; \op{conv}(\mt{X})) , 
\end{align}

\item 
Note that the way we define $\cal{I}$ allows for having two distinct indices $i,j\in\cal{I}$ representing the same ray (when they are aligned.)

\end{itemize}

As promised on page \pageref{prob:loss1F-I}, we now provide an explicit description of the class $\FC(\mt{X},\vc{y})$. We avoid this statement within the theorems (e.g., those in \autoref{sec:cor}) to make the results  more accessible. 

\begin{definition}[The Main Class]\label{def:class}
Given any $\mt{X}\in\mathbb{R}^{n\times p}$ and any $\vc{y}\in\mathbb{R}^n$, define a class of extended-real valued functions $\FC(\mt{X},\vc{y})$ over $\mathbb{R}^n$ as in the following. For any $\overline{f}:\mathbb{R}^n\to\Rext$, we have $\overline{f}\in\FC(\mt{X},\vc{y})$ if and only if there exists, 
\[
	f:\mathbb{R}^n\to\Rext,~ 
	\eta>0,~
	\alpha>0,~
	\psilow\in\mathbb{R},~
	\gamma\in[0,1],~
	\vc{h}\in\op{dom}(f_\gamma)^\star , 
\]
for which
\begin{itemize}
\item $\overline{f} = f/\eta$, 
\item $f$, $\vc{y}/\alpha$, and $\psilow$ satisfy \autoref{condn:ver-pos},  
\item $(1-\gamma)\eta 
	\leq  \rc(- \mt{X}^T\vc{h}) <
	\alpha \cdot \left[ - (f_\gamma)^\star (\vc{h})  - \psilow \right] -  \gamma \eta$, where $f_\gamma(\vc{u}) \coloneqq f( \vc{u} ) - \gamma \cdot \dd f(\vc{u})(\vc{u})$, and,
\item $f$ and $\mt{X}$ satisfy the assumptions of \autoref{lem:opt-conds-pos}. 	
\end{itemize}
\end{definition}

As an example, using the results of the next section, we can show that $\FC_{1}(\mt{X},\vc{y}) \subset \FC(\mt{X},\vc{y})$ for all $\mt{X}$ and $\vc{y}$, where $\FC_1$ is defined next. 
\begin{definition}\label{def:class-cvx}
Given any $\mt{X}\in\mathbb{R}^{n\times p}$ and any $\vc{y}\in\mathbb{R}^n$, define a class of extended-real valued convex functions $\FC_{1}(\mt{X},\vc{y})$ over $\mathbb{R}^n$ as in the following. For any $\overline{f}:\mathbb{R}^n\to\Rext$, we have $\overline{f}\in\FC_{1}(\mt{X},\vc{y})$ if and only if there exists, 
\[
	f:\mathbb{R}^n\to\Rext,~ 
	\eta>0,~
	\alpha>0,~
	\gamma\in[0,1],
\]
for which
\begin{itemize}
\item $\overline{f} = f/\eta$, 
\item $f$ is star-convex with respect to $\vc{y}/\alpha$, 

\item Interval $\left[ 
	\, (1-\gamma)\eta 
	\,,\,
	 \alpha (f_\gamma)^{\star\star}(\vc{0})  - \alpha f(\vc{y}/\alpha)  -  \gamma \eta \,\right)$ has a nonempty intersection with the interval $\{ \rc(- \mt{X}^T\vc{h}) :~ \vc{h}\in \partial (f_\gamma)^{\star\star}(\vc{0}) \}$, where $f_\gamma(\vc{u}) \coloneqq f( \vc{u} ) - \gamma \cdot \dd f(\vc{u})(\vc{u})$, and,
\item $\op{dom}f$ cannot be separated from the range of $\mt{X}$. 	
\end{itemize}
\end{definition}

\subsection{Corollaries and Special Cases}\label{sec:cor}

We list a few corollaries and special cases of \autoref{thm:main-detailed}.

First, as mentioned before, 
if the columns of $\mt{X}$ have been already standardized, as common in pre-processing for regression, we can choose $\alpha = 1/\norm{\vc{y}}_2$. In such case, \autoref{condn:ver-pos} will be satisfied if $\vc{y}$ is not a multiple of any of the columns of $\mt{X}$. 

\subsubsection{Corollary; Bounding the Conjugate} \label{sec:f-bnd}
Note that $g_1\leq g_2$ implies $g_2^\star \leq g_1^\star$ and $\op{dom}g_1^\star \subseteq \op{dom}g_2^\star$, for any two functions $g_1$ and $g_2$. Therefore, if $F\leq f_\gamma$, we can choose $\vc{h}\in \op{dom}F^\star$ and replace $(f_\gamma)^\star (\vc{h})$ with $F^\star (\vc{h})$ in \eqref{eq:cond-eta} to get a new (weaker) result. For example, consider \autoref{condn:f-bnd} which makes use of a convex envelope function.

\begin{setup}\label{condn:f-bnd}
For a function $f:\mathbb{R}^n\to\Rext$, assume that the following holds for all $\vc{u}\in\op{dom} f$: 
\begin{align*}
f_\gamma(\vc{u}) \coloneqq f( \vc{u} ) - \gamma \cdot \dd f(\vc{u})(\vc{u}) \geq H(\vc{u}) + \psiup
\end{align*}
for some $\gamma\in [0,1]$, $\psiup\in\mathbb{R}$, as well as a proper lsc convex function $H$ whose conjugate is non-positive on its domain. In other words, $H$ can be represented as 
	\begin{align*}
	H(\vc{u})  = \sup_{\vc{h}\in \cal{H}}\, \langle \vc{h}, \vc{u}\rangle - \hat{H}(\vc{h}) 
	\end{align*}
	for some convex function $\hat{H}:\mathbb{R}^n\to\mathbb{R}$ and some closed convex set $\cal{H}\subset \mathbb{R}^n$, where $\hat{H}(\vc{h})\leq 0$ for all $\vc{h}\in\cal{H}$. 
\end{setup}
Suppose $f$ is regular and the convex hull of $f_\gamma$ is proper; through the choice of $\gamma$. 
Let us provide some high-level recipes in relation to \autoref{condn:f-bnd}. 
For example, for any value of $\psiup \leq -\inf (f_\gamma)^\star$ we can define $\cal{H}\coloneqq \{\vc{h}:~ (f_\gamma)^\star (\vc{h}) \leq - \psiup  \} \neq \emptyset$, a sublevel set, and 
\begin{align*}
H^\star(\vc{h}) \coloneqq \begin{cases}
(f_\gamma)^\star (\vc{h}) + \psiup & \vc{h}\in\cal{H}, \\
+\infty & \vc{h}\not\in\cal{H}.
\end{cases}
\end{align*}
Observe that the above satisfies $(f_\gamma)^\star \leq H^\star-\psiup$ hence $f_\gamma \geq (f_\gamma)^{\star\star} \geq  H + \psiup$, and $H^\star$ is non-positive on its domain. We can simplify the matters here and, with the same $\cal{H}$, choose $H^\star(\vc{h}) = \indic(\vc{h}; \cal{H})$. In that case, we get $H(\vc{u}) = \sigma(\vc{u}; \cal{H})$, where $\sigma$ denotes the support set. This choice of $H$ is sublinear (convex and positively homogeneous). 
As yet another special case, we can choose $\psiup = -\inf(f_\gamma)^\star = (f_\gamma)^{\star\star}(\vc{0})$, $\cal{H} = \Argmin (f_\gamma)^\star = \partial (f_\gamma)^{\star\star}(\vc{0})$, and $H(\vc{u}) = \sigma(\vc{u}; \cal{H})= \dd (f_\gamma)^{\star\star}(\vc{0})(\vc{u})$, where we used \cite[Theorem 11.8, Theorem 8.30]{rockafellar2009variational}. We have used this last choice in defining $\FC_1$ in \autoref{def:class-cvx}. 
See \autoref{fig:example-H-gamma} for illustrations of the relationship between $\gamma$ and $\psiup$.

If \autoref{condn:f-bnd} holds, then $- (f_\gamma)^\star (\vc{h})$ in \eqref{eq:cond-eta} can be replaced by the constant $\psiup$, leading to a weaker result. In such case, the requirement on $\eta$ in \autoref{thm:main-detailed} can be stated as in \autoref{condn:eta}.

\begin{setup}
\label{condn:eta}
Consider a function $f$ that satisfies \autoref{condn:f-bnd} with some $\gamma\in [0,1]$, $\psiup\in\mathbb{R}$, and convex function $H$, and satisfies \autoref{cond:f-star} with some $\psilow$. 
Consider $r(\vc{\beta}) = \vc{1}^T\vc{\beta} + \indic (\vc{\beta}; \mathbb{R}_+^p)$. Assume the following two intervals intersect:
\[
\big[\, (1-\gamma)\eta \,,\,  \alpha\psiup-\alpha\psilow - \gamma \eta \,\big) 
\cap 
\big\{ \rc(- \mt{X}^T\vc{h}):~ \vc{h}\in \cal{H} \big\}
\neq \emptyset\, . 
\]
\end{setup}
Consider $A \coloneqq \{ \rc(- \mt{X}^T\vc{h}):~ \vc{h}\in \cal{H}\}$ which is an interval. 
When $\gamma\in (0,1)$, consider $\etabnd \coloneqq \max \{ 
\min ( \etaone/(1-\gamma) \,,\, (\alpha\psiup-\alpha\psilow- 
\etaone )/\gamma ):~
\etaone \in A \}$. Note that the objective function in the maximization is concave in $\etaone$ (a scalar.) Therefore, if $\alpha(1-\gamma)(\psiup-\psilow) \in A$, then $\etabnd = \alpha(1-\gamma)(\psiup-\psilow)$. Otherwise, the maximum will be attained at one of the two extremes of $A$. 

Using \autoref{lem:perspective-props}, if $\etabnd>0$ then it is possible to transform the original optimization problem \eqref{prob:loss-reg} into a new one, through the choice of $\alpha$, without changing the support of the optimal solutions, so that the new $\etabnd$ becomes $\max(A)$. The price we pay is that this transformation scales $\vc{y}$ to $\alpha\vc{y}$ and we will need to verify \autoref{condn:ver-pos} for $\alpha\vc{y}$. As we will show in \autoref{lem:perspective-props}, \autoref{cond:f-star} and \autoref{condn:f-bnd} are invariant under these transformations.

\begin{figure}
	\centering
	\includegraphics[clip, trim=4.9cm 8cm 4.2cm 7cm, width=.19\textwidth]{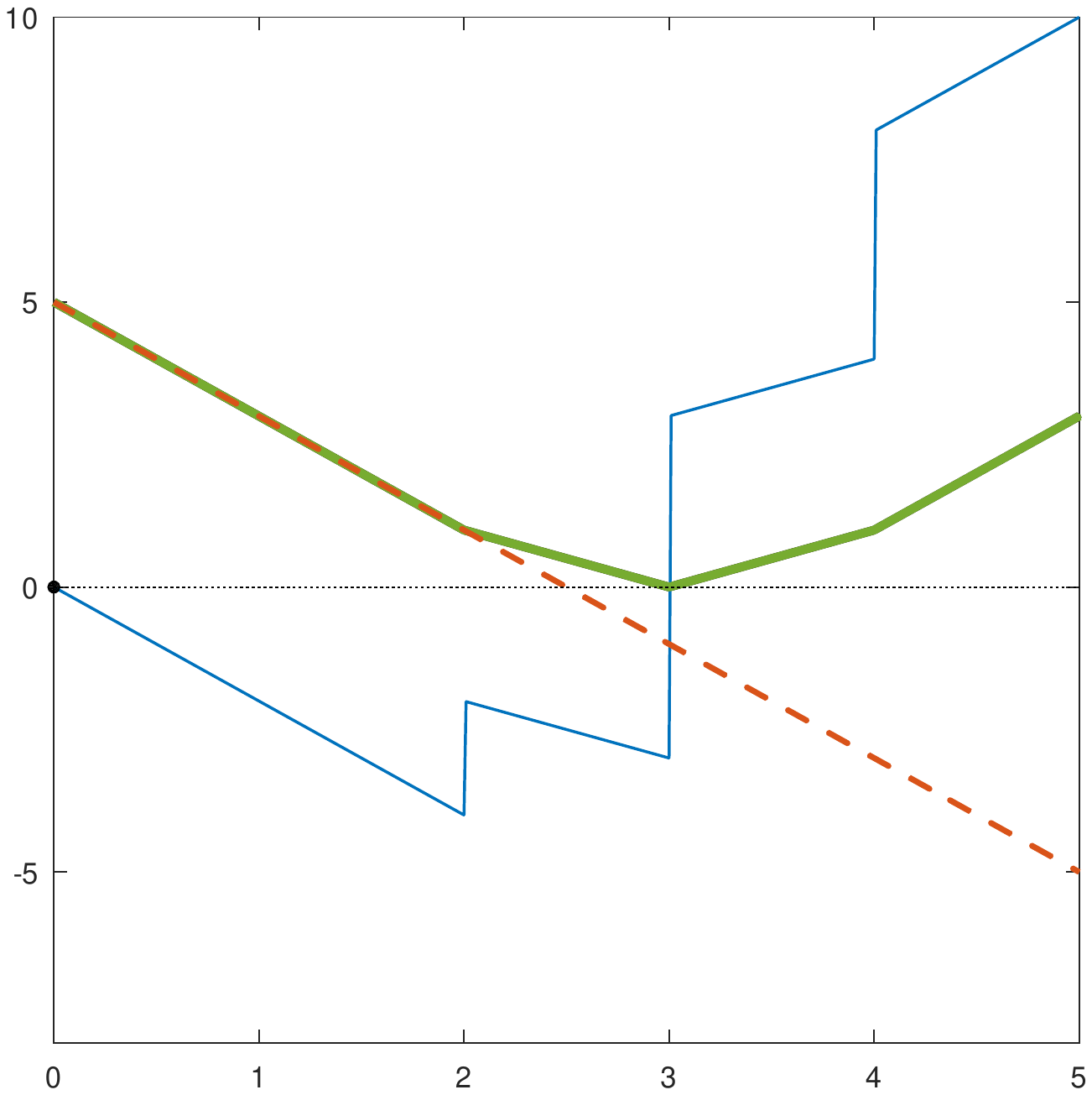}
	\includegraphics[clip, trim=4.9cm 8cm 4.2cm 7cm, width=.19\textwidth]{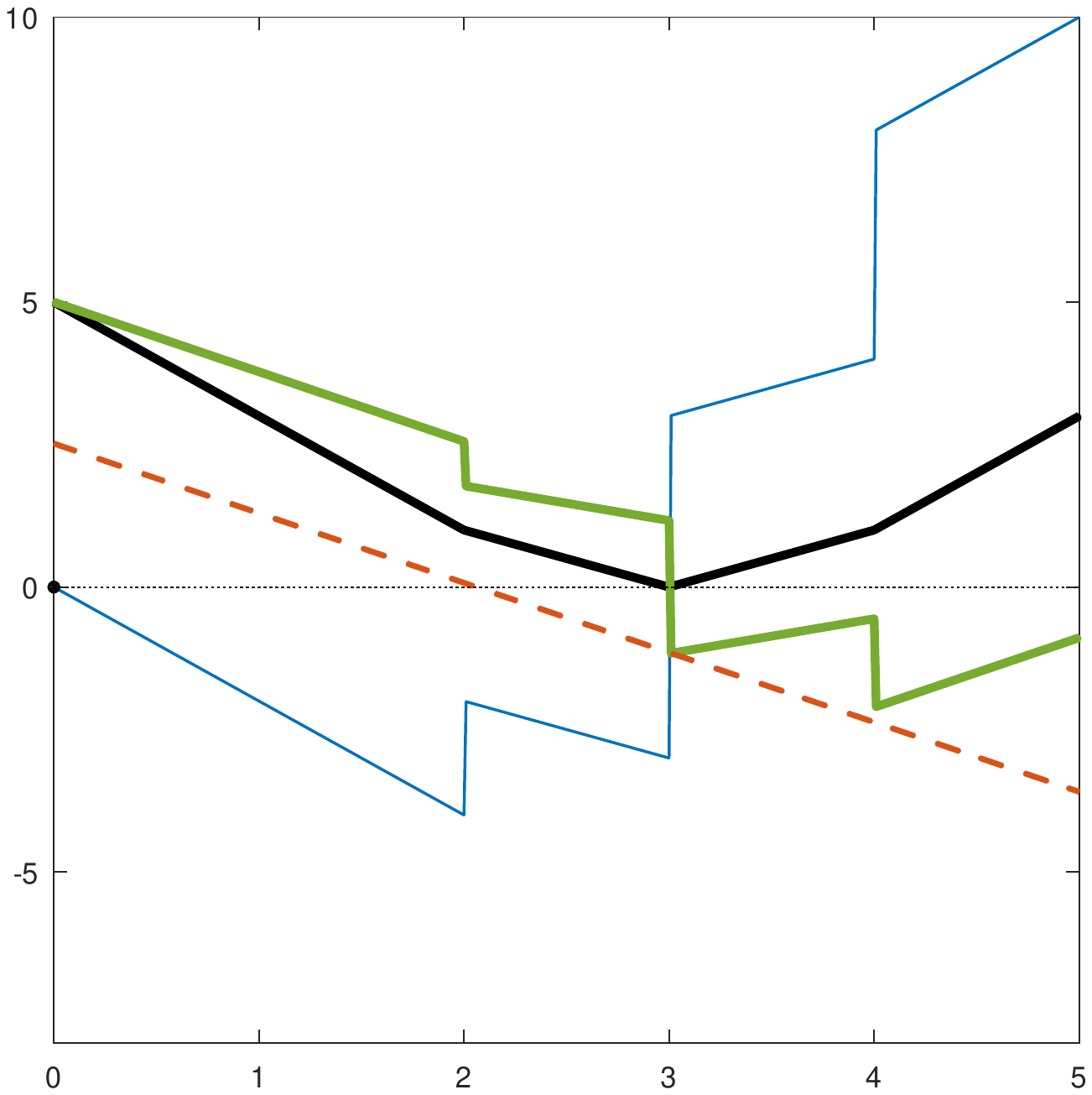}
	\includegraphics[clip, trim=4.9cm 8cm 4.2cm 7cm, width=.19\textwidth]{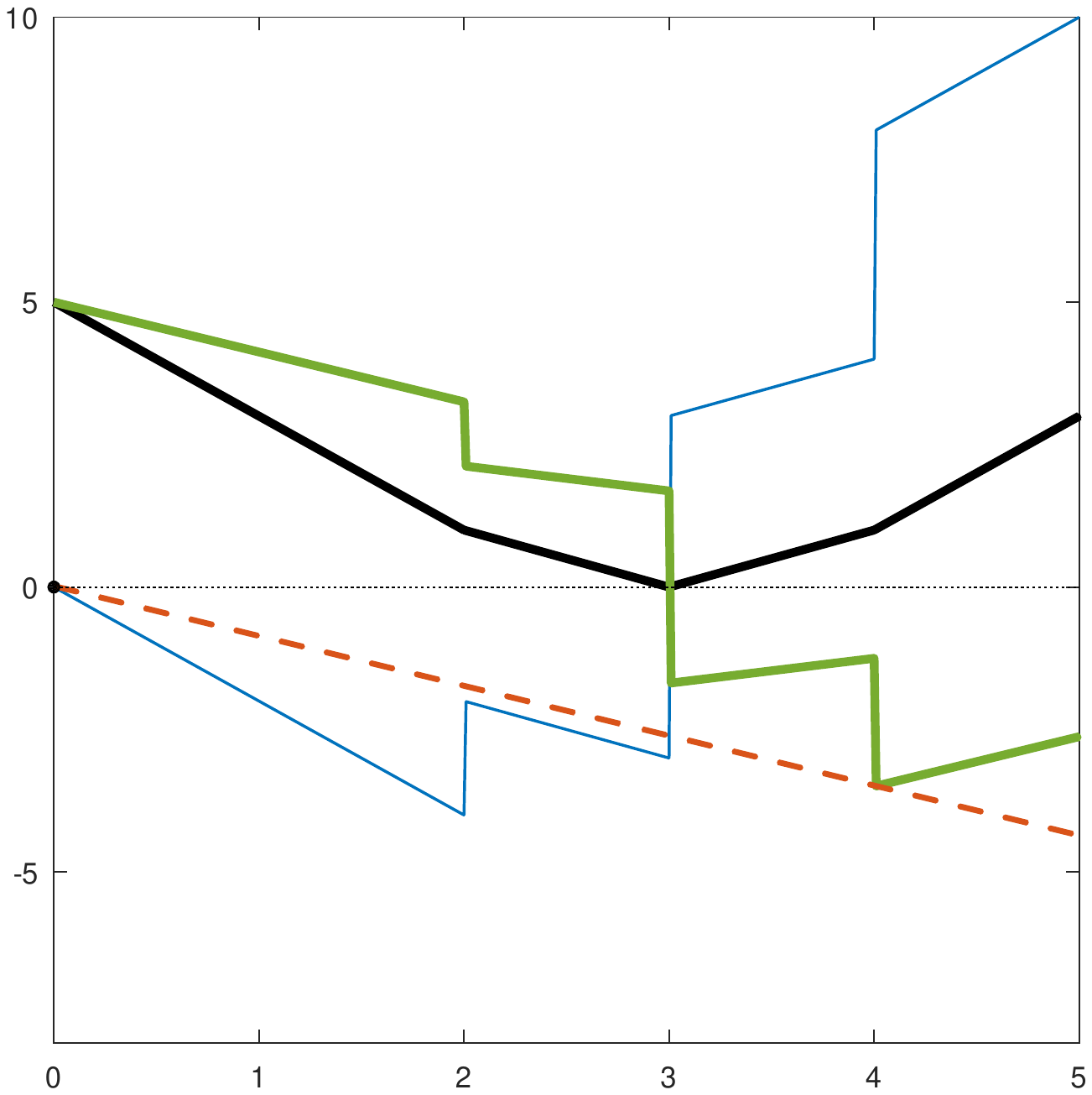}
	\includegraphics[clip, trim=4.9cm 8cm 4.2cm 7cm, width=.19\textwidth]{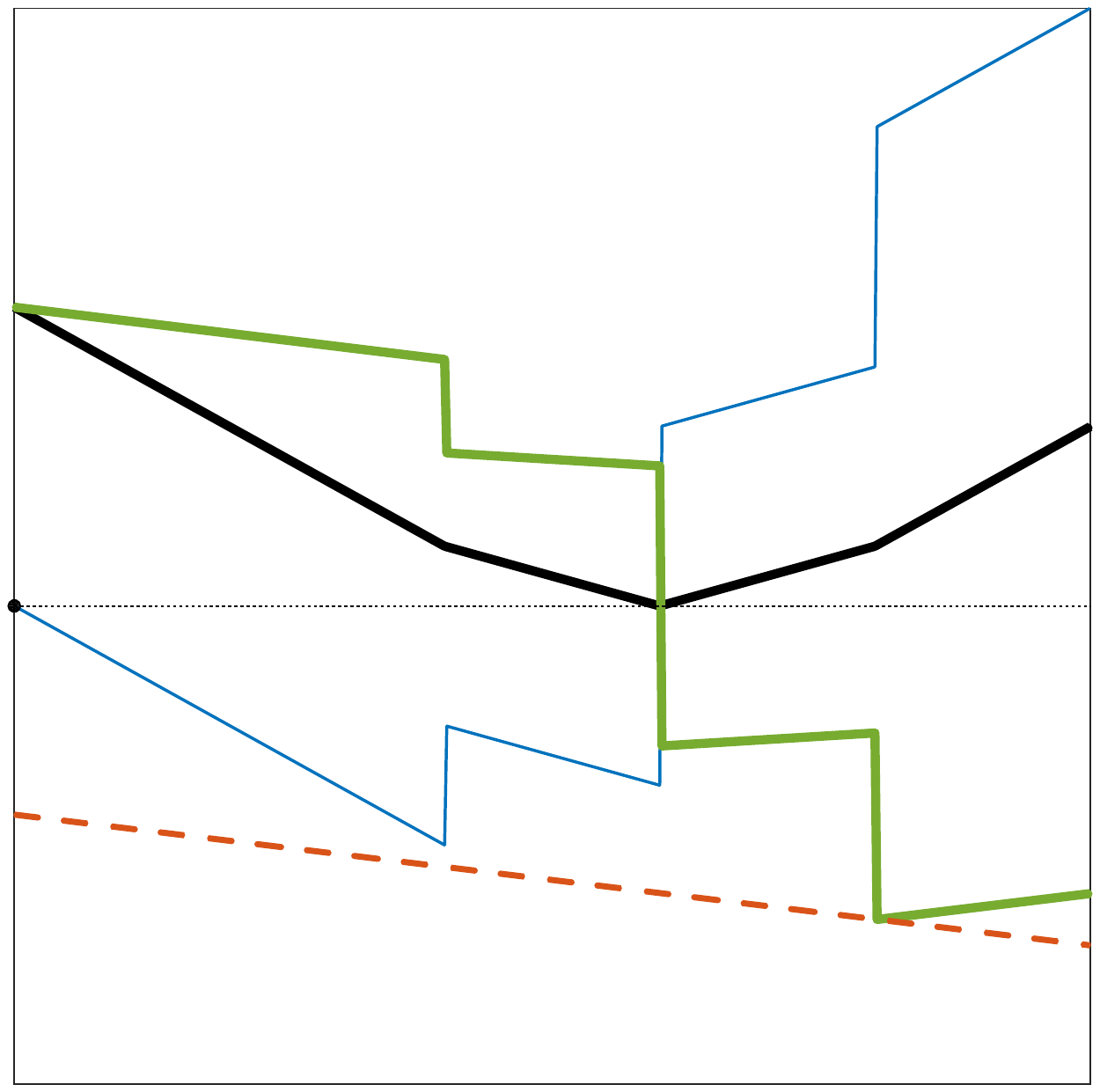}
	\includegraphics[clip, trim=4.9cm 8cm 4.2cm 7cm, width=.19\textwidth]{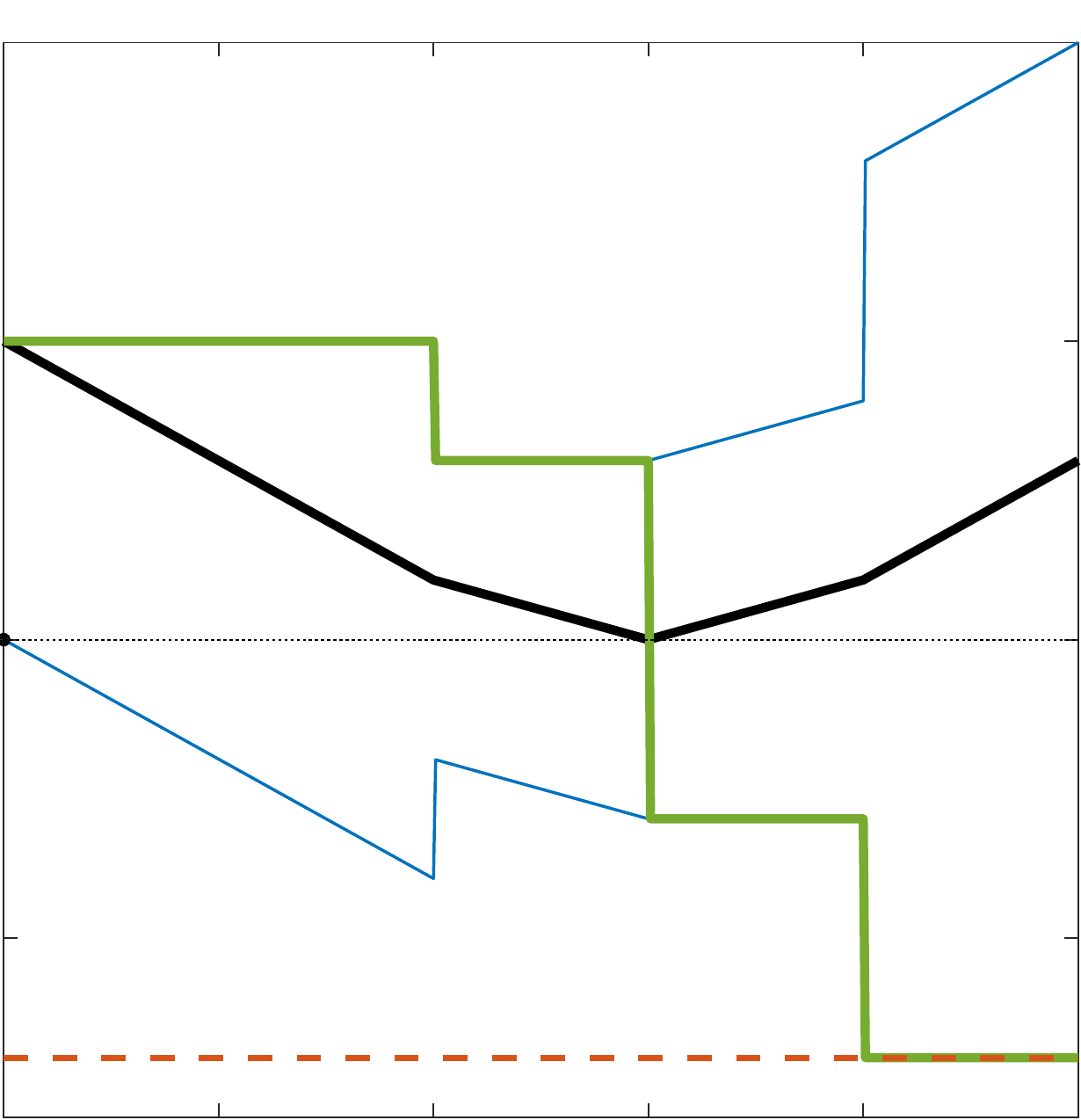}	
\\
	\includegraphics[clip, trim=4.9cm 8cm 4.2cm 7cm, width=.19\textwidth]{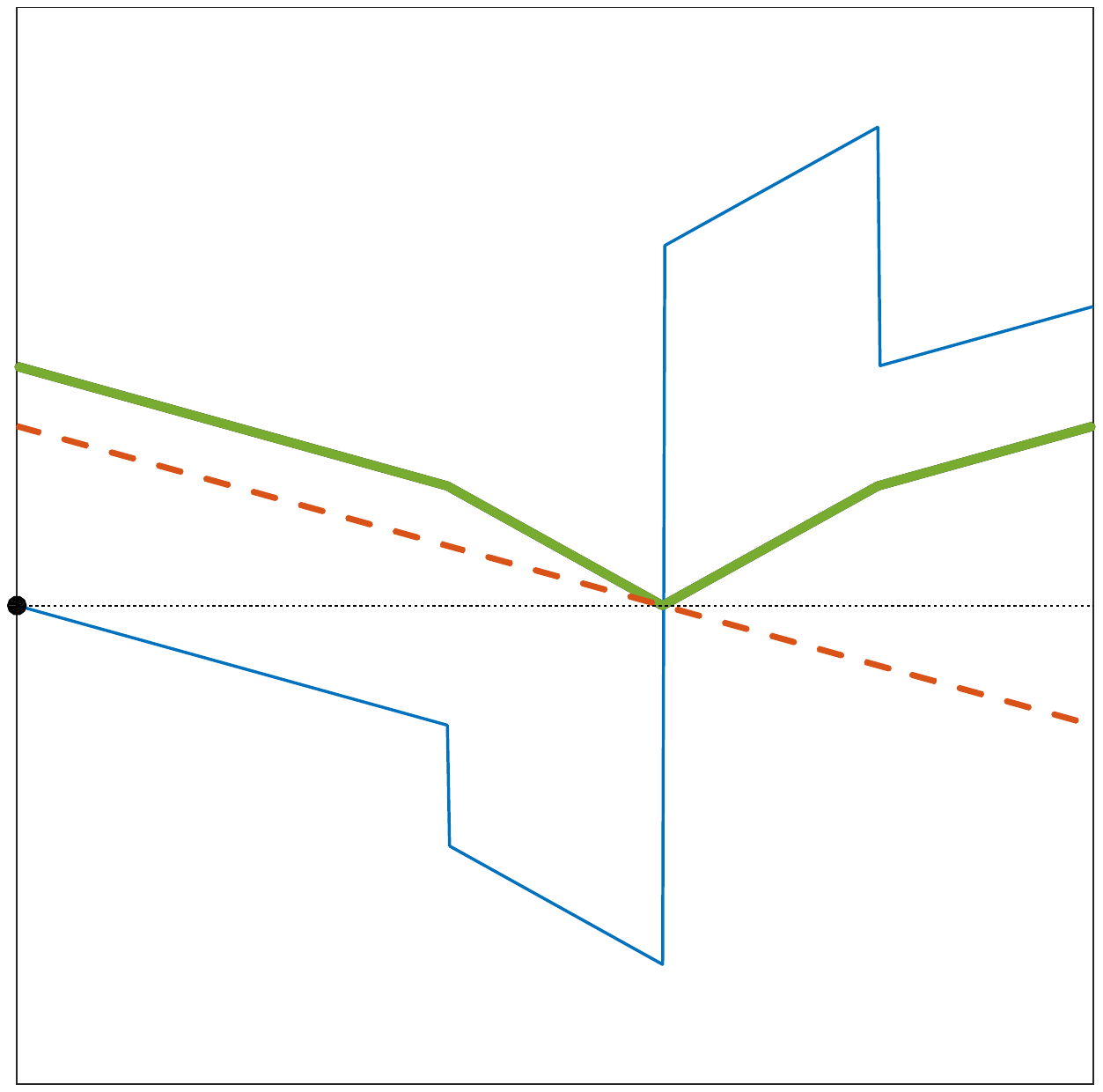}
	\includegraphics[clip, trim=4.9cm 8cm 4.2cm 7cm, width=.19\textwidth]{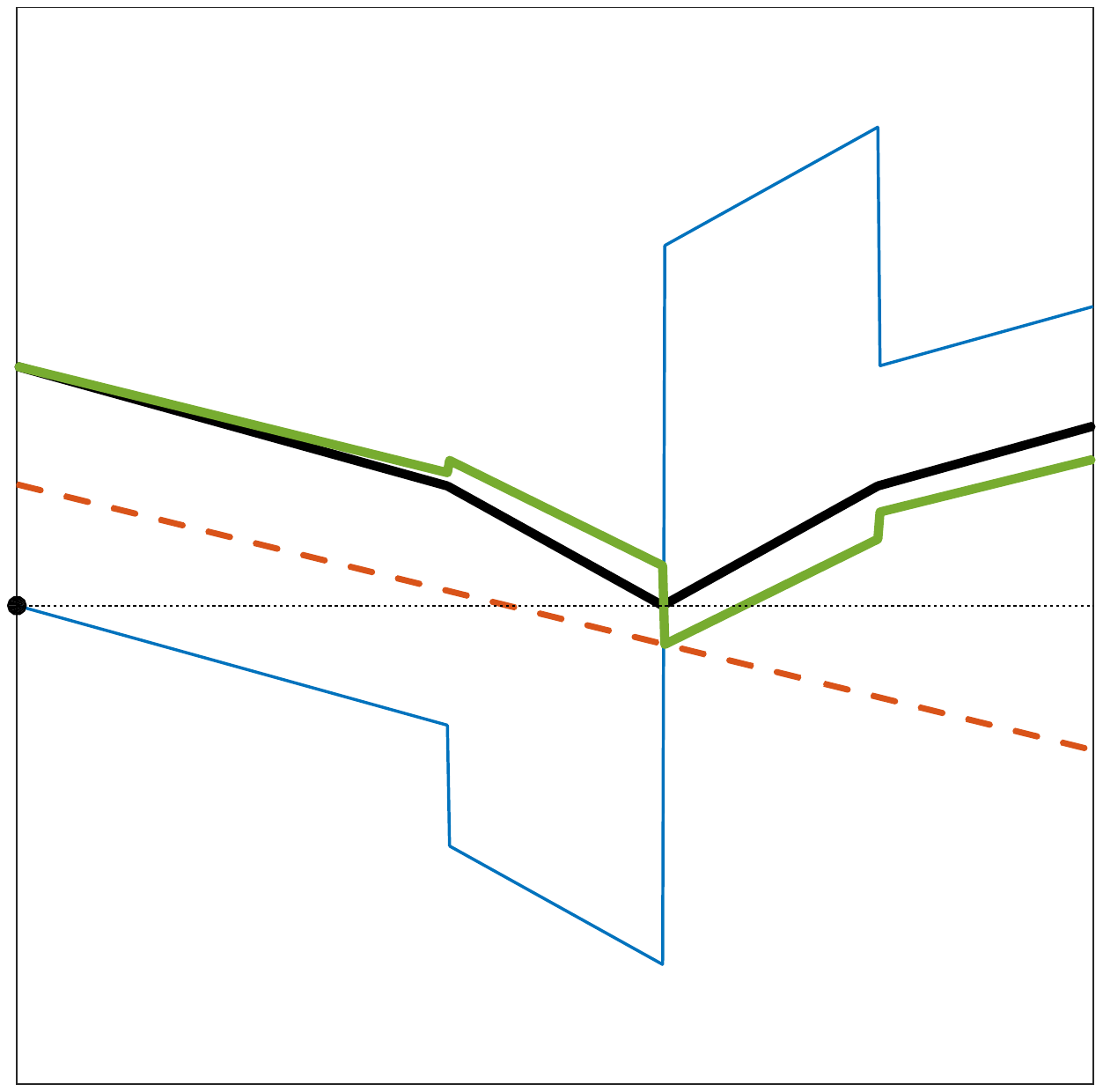}
	\includegraphics[clip, trim=4.9cm 8cm 4.2cm 7cm, width=.19\textwidth]{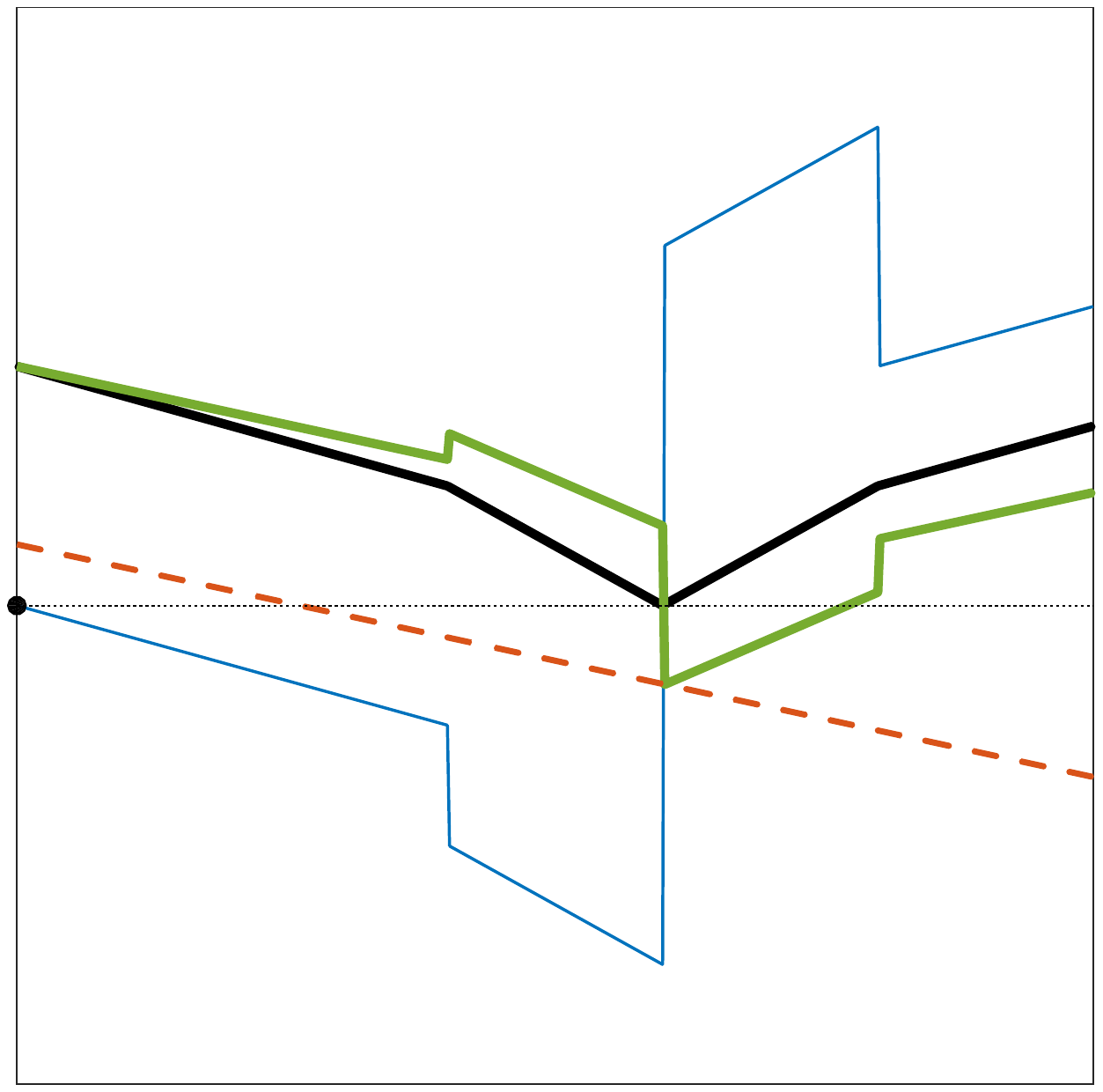}
	\includegraphics[clip, trim=4.9cm 8cm 4.2cm 7cm, width=.19\textwidth]{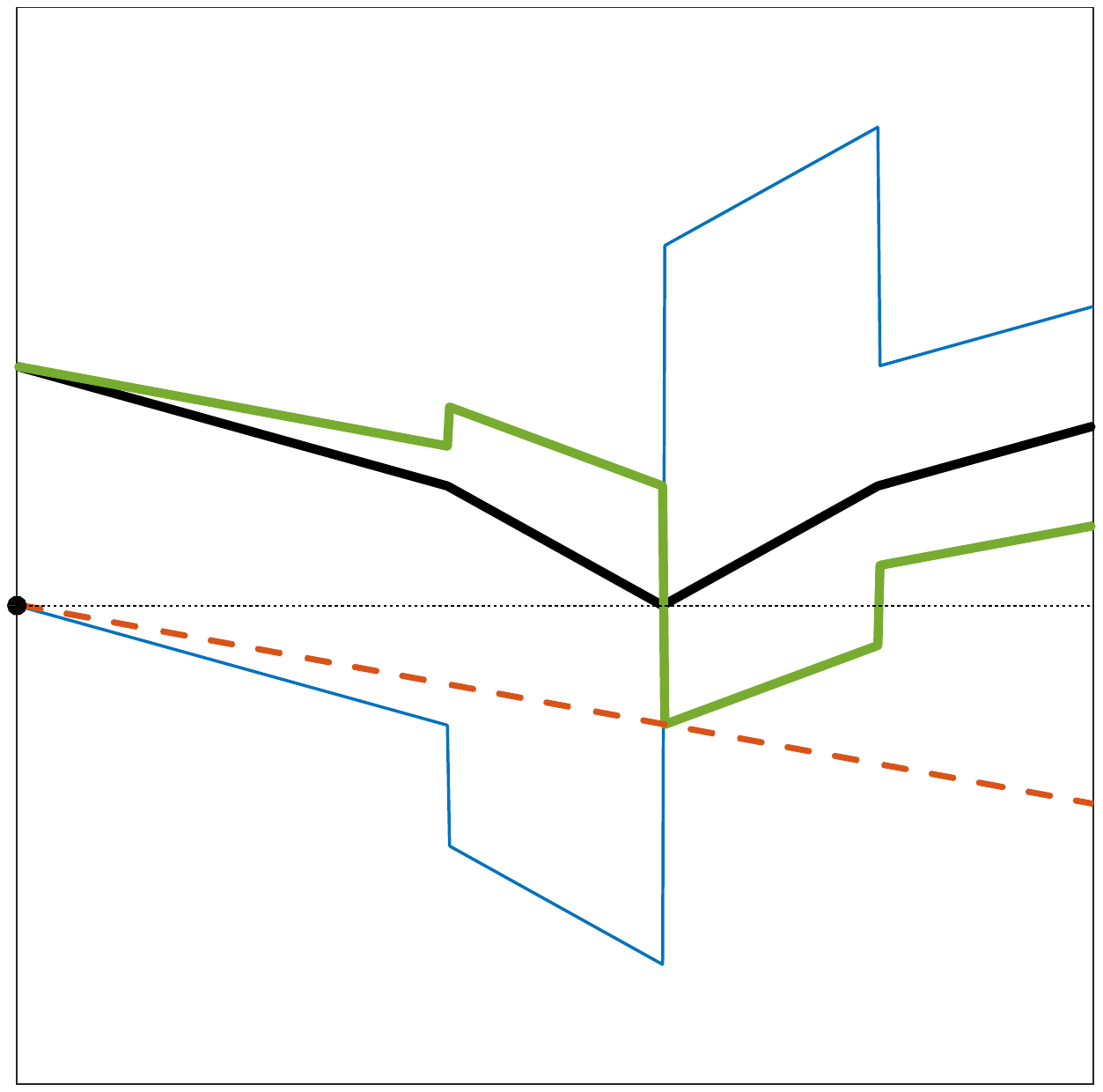}
	\includegraphics[clip, trim=4.9cm 8cm 4.2cm 7cm, width=.19\textwidth]{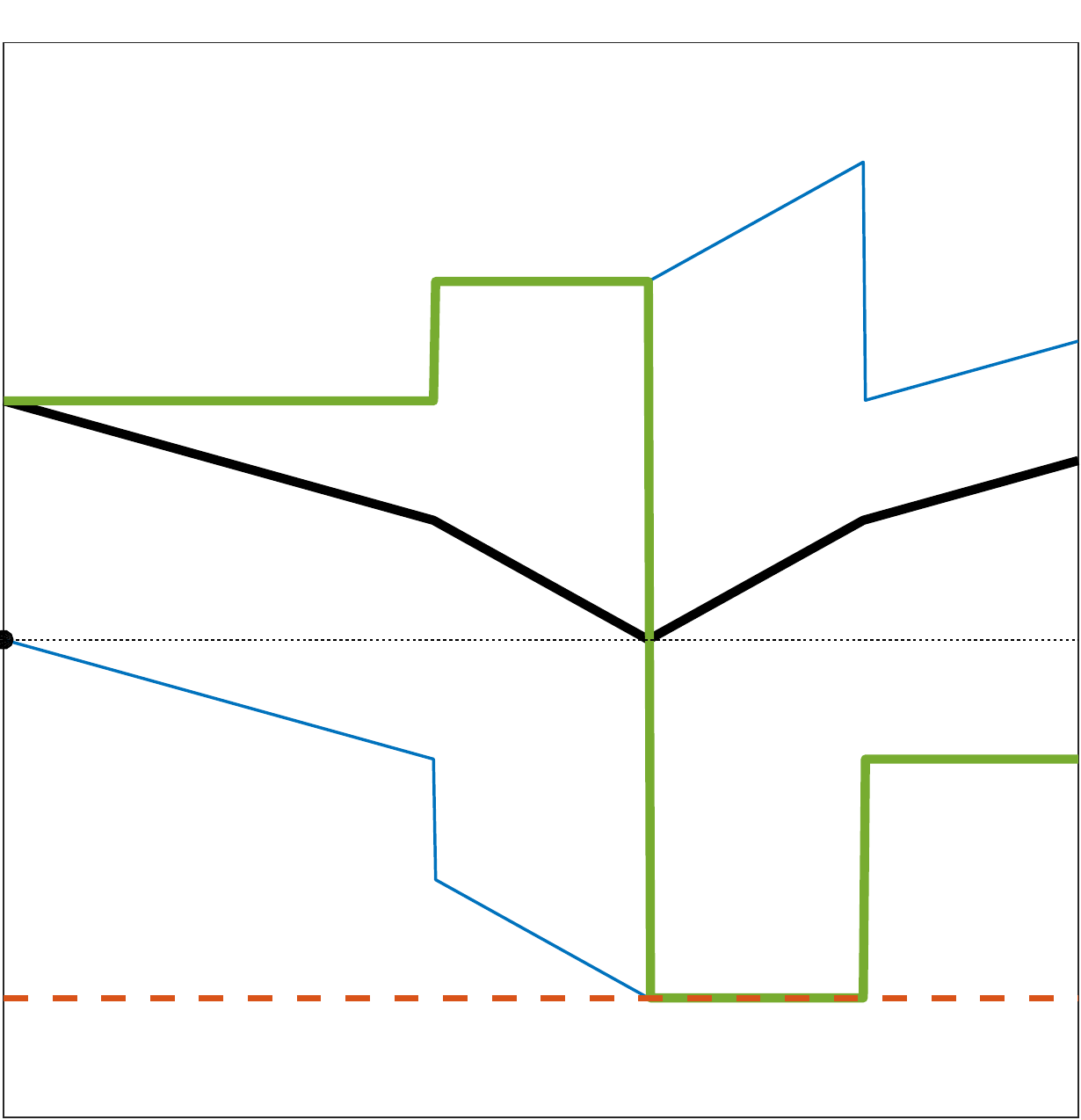}	
	\caption{$\gamma$ versus $\psiup$. 
	In each picture, the origin is marked by a black circle on the left. $f$ is the piecewise linear function with $f(3)=\inf f=0$, breakpoints at $\{2,3,4\}$, and slopes in $\{\pm1,\pm2\}$, in black. 
	$\dd f(u)(u)$ is in blue (not changing across the pictures in each row) and $f_\gamma$ is in green. $H(u)+\psiup$ is the red dashed line and it lower bounds $f_\gamma$. 
	Top row: $f(2)=f(4)=1$, $\psilow=0$, and from left to right, 
	$\gamma=0,\nicefrac{7}{18},\nicefrac{9}{16},\nicefrac{25}{32},1$ and 
	$\psiup = 5,\nicefrac{5}{2}, 0, \nicefrac{-7}{2},-7$. 
	Bottom row: $f(2)=f(4)=2$, $\psilow=1$, and from left to right, 
	$\gamma=0,\nicefrac{1}{9},\nicefrac{2}{9},\nicefrac{1}{3},1$, and 
	$\psiup=4,2,1,0,-6$. 
	In the first three columns, $\psiup\geq\psilow$. 	
	In all of the pictures, $\hat{H}\equiv 0$ and $\cal{H}=\{\gamma-1\}$. 		
	}
	\label{fig:example-H-gamma}
\end{figure}

\subsubsection{Example; Regularized Least-Squares}
Consider $f(\vc{u}) = \norm{\vc{u}-\vc{y}}_2^2$. This function is convex and has a minimum at $\vc{y}$, hence it satisfies \autoref{cond:f-star} with the same $\vc{y}$ and with $\psilow=0$. 
Moreover, 
$f_\gamma(\vc{u}) \coloneqq f( \vc{u} ) - \gamma \cdot \dd f(\vc{u})(\vc{u}) = \norm{\vc{u}-\vc{y}}_2^2 - \gamma \langle \vc{u}, 2(\vc{u}-\vc{y})\rangle = (1-2\gamma) \norm{\vc{u}}_2^2 + \norm{\vc{y}}_2^2 - 2(1-\gamma) \langle \vc{u}, \vc{y}\rangle$. Choosing $\gamma=1/2$ yields $f_{1/2}(\vc{u}) = \norm{\vc{y}}_2^2 -\langle \vc{u}, \vc{y}\rangle$ and $f_{1/2}^\star(-\vc{y}) = -  \norm{\vc{y}}_2^2$ with $\op{dom} f_{1/2}^\star =\{-\vc{y}\}$. 
Therefore, we can state a corollary of \autoref{thm:main-detailed} for the case of least-squares. To strengthen the following result, we assume without any loss in generality, that $\vc{y}\in \op{col}(\mt{X})$; otherwise, extract the orthogonal component from the squared norm. 
\begin{theorem}\label{thm:LS-gendata}
Given $\mt{X} = \begin{bmatrix} \vc{x}_1 & \cdots & \vc{x}_p \end{bmatrix}\in\mathbb{R}^{n\times p}$ and $\vc{y}\in \op{col}(\mt{X})$, assume that $\mt{X}$ and $\alpha \vc{y}$ satisfy \autoref{condn:ver-pos} for some $\alpha >0$. 
Furthermore, suppose $\eta>0$ satisfies 
\begin{align}\label{eq:eta-condn-LS}
	 \eta    &< 2\alpha \norm{\vc{y}}_2^2  -  2 (\max_{i\in[p]}\, \vc{y}^T\vc{x}_i) \,.
\end{align}
Consider the convex cone 
$
\T \coloneqq 
\bigl\{ \sum_{i=1}^p \lambda_i (\vc{x}_i - \alpha \vc{y} ) :~ \lambda_i \geq 0   \bigr\}
$ 
and assume for some $\cal{I}\subseteq [p]$, if $\{t(\vc{x}_i- \alpha\vc{y}):~ t\geq 0\}$ is an extreme ray of $\T$ then $i\in\cal{I}$. 
Then, the two optimization problems in \eqref{prob:loss-reg-lin-pos} and \eqref{prob:loss1F-I}, with $f(\vc{u})=\norm{\vc{u}-\vc{y}}_2^2$, have the same optimal values and the same sets of optimal solutions; i.e., $\cal{I}$ is a superset of the supports of all optimal solutions for \eqref{prob:loss-reg-lin-pos}.  
\end{theorem}

Specializing \autoref{thm:LS-gendata} to when all of the columns of $\mt{X}$ are on the unit sphere and $\alpha = 1/\norm{\vc{y}}_2$ leads to \autoref{thm:main}. Note that, in \autoref{thm:main}, the bound on $\eta$ prohibits $\vc{y}$ to coincide with any of the columns of $\mt{X}$ as required in \autoref{condn:ver-pos}.

In \autoref{thm:LS-gendata}, we choose $\gamma=1/2$ which then leads to the optimal choice of $\cal{H}=\{-\vc{y}\}$. In \autoref{rem:eta-bound}, we examine the case where $\gamma \in (0,1/2)$ and we show that it is possible to arrive at a larger interval for $\eta$ compared to the one in \autoref{thm:LS-gendata}. It can be seen from the calculations in \autoref{rem:eta-bound} that a tighter analysis, resulting in a weaker requirement on $\eta$, is possible. Such analysis would incorporate properties of $\mt{X}$ in choosing $\gamma$. 

\subsubsection{Example; Regularized Mahalanobis Distance}
Consider $f(\vc{u}) = \vc{u}^T \mt{A} \vc{u} -2 \vc{u}^T\vc{b} + c$ for some positive definite matrix $\mt{A}$, vector $\vc{b}\neq\vc{0}$, and scalar $c$. The unique minimizer of $f$ is $\vc{y}\coloneqq \mt{A}^{-1}\vc{b}$ and $f(\vc{y})= c- \vc{b}^T\mt{A}^{-1}\vc{b}$. Moreover, the smallest scalar $\psilow$ for which \autoref{cond:f-star} holds is $\psilow = f(\vc{y})$. As mentioned before, since $\psilow=\inf f$, the second condition of \autoref{cond:f-star} is implied by the first one.

Next, consider $\gamma = 1/2$ and $f_{1/2}(\vc{u}) \coloneqq f( \vc{u} ) - \frac{1}{2} \cdot \dd f(\vc{u})(\vc{u}) 
= c- \vc{u}^T\vc{b}$ and $f_{1/2}^\star(-\vc{y}) = -c$ and $\op{dom} f_{1/2}^\star =\{-\vc{b}\}$. 
Therefore, we can state a corollary of \autoref{thm:main-detailed} for the case of Mahalanobis distance. To strengthen our result, we assume without any loss in generality, that $\vc{b}\in \op{col}(\mt{A})$. Then, we can state a similar result to \autoref{thm:LS-gendata} for regularized Mahalanobis distance minimization but with the following conditions on $\eta$ replacing \eqref{eq:eta-condn-LS}:
\[
\eta < 2\alpha \vc{b}^T\mt{A}^{-1}\vc{b} -  2 (\max_{i\in[p]}\, \vc{b}^T\vc{x}_i) \,.
\]

\subsubsection{Example; Regularized $\ell_q^q$ Regression} 
Consider $f(\vc{u}) = \norm{\vc{u}-\vc{y}}_q^q$ with $q\geq 2$. Let us work with the stronger condition in \autoref{condn:f-bnd} instead of $f_\gamma$ itself. Let us choose $\cal{H} = (1-\gamma)\partial f(\vc{0})$ and $\hat{H}\equiv 0$. 
Since the function is separable, it suffices to study each of its parts separately. For the moment, consider $f(u) = \abs{u-y}^q$ where $y\in\mathbb{R}$. 
\begin{lemma}\label{lem:lqq-scalar}
For any $q\geq 2$ and $y\in\mathbb{R}$, consider $f(u) = \abs{u-y}^q$ and 
\[F(u) \coloneqq f(u) - \gamma u f'(u) - (1-\gamma) u f'(0).\]  
For $\gamma = 1/(2q-2)$, we have $uF'(u)\geq 0$ and $F(u)\geq F(0)$ for all $u$.  
\end{lemma}
\begin{proof}[Proof of \autoref{lem:lqq-scalar}]
Observe that 
\begin{align*}
F(u) &= 
\abs{u-y}^{q-2} [(1-\gamma q)u-y] \cdot (u-y)  
+ (1-\gamma) q \abs{y}^{q-2} yu \,,\\
F'(u) &= 
\abs{u-y}^{q-2} [ (1-\gamma q)u - (1-\gamma)y ] \cdot q 
+ (1-\gamma) q \abs{y}^{q-2} y  \,,\\
F''(u) &=
\abs{u-y}^{q-4} [(1-\gamma q)u - (1-2\gamma)y]\cdot q(q-1)(u-y)\,.
\end{align*}
Moreover, $F(0)=\abs{y}^q$, $F'(0)=0$. 
It is easy to verify the claim for $q=2$ and $\gamma = 1/2$. Therefore, consider $q>2$. Furthermore, if $y=0$ and $\gamma q <1$ then both of the claims hold. Therefore, assume $y\neq 0$ and $q>2$. 
 
Define $u_+\coloneqq y(1-2\gamma)/(1-\gamma q)$. Note that $F''$ has roots only at $y$ and $u_+$. Therefore, $F'$ has stationary points only at these roots. Since $yF'(y)>0$, if we choose $\gamma$ such that $yF'(u_+)\geq 0$, $(1-2\gamma)/(1-\gamma q)>0$, and $\gamma q<1$, then $u=0$ will be a global optimum for $F$; i.e., $F(u)\geq F(0)$ for all $u$. We have 
\[
yF'(u_+) = q \abs{y}^q \left[ 1-\gamma - \gamma ( \frac{1-2\gamma}{1-\gamma q}-1 )^{q-2} \right]. 
\]
Algebraic manipulations establish that $\gamma = 1/(2q-2)$ guarantees $yF'(u_+)\geq 0$. 
\end{proof}
\begin{corollary}\label{cor:Lqq-regression}
Consider $f(\vc{u}) = \norm{\vc{u}-\vc{y}}_q^q$ with $q\geq 2$. 
Consider $\gamma = 1/(2q-2)$. 
Then, $f_\gamma(\vc{u}) 
\geq (1-\gamma)\dd f(\vc{0})(\vc{u}) 
+ \norm{\vc{y}}_q^q
= -q(1-\gamma) \norm{\vc{y}}_q^{q-2} \langle\vc{y},\vc{u}\rangle
+ \norm{\vc{y}}_q^q$ for all $\vc{u}$. 
\end{corollary}
With the above at hand, we can state a parallel to \autoref{thm:LS-gendata} for regularized $\ell_q^q$ regression, with the only difference in the bound on $\eta$, namely
\begin{align*}
\eta < 
\norm{\vc{y}}_q^{q-2} \cdot 
\left( 
\alpha (2q-2) \norm{\vc{y}}_q^2 - q(2q-3) (\max_{i\in[p]}\, \vc{y}^T\vc{x}_i)
\right) 
\end{align*}
in stead of \eqref{eq:eta-condn-LS}. 
With $q=2$, the above reduces to \eqref{eq:eta-condn-LS}.

\subsubsection{Example; Regularized Bregman Divergence.}

Consider a differentiable convex function $F:\mathbb{R}^n\to\mathbb{R}$ and a vector $\vc{y}$ and define $f(\vc{u}) =  F(\vc{u}) - F(\vc{y}) - \langle \nabla F(\vc{y}) , \vc{u}-\vc{y}\rangle $. Note that $\vc{y}$ is a minimizer of $f$ for which $f(\vc{y})=0$. Moreover, $\psilow=0$. Furthermore, for any choice $\gamma$, we get
\[
f_\gamma (\vc{u})
= F(\vc{u}) - F(\vc{y}) 
- \langle \gamma \nabla F(\vc{u}) + (1-\gamma)\nabla F(\vc{y}) , \vc{u}\rangle + \langle \nabla F(\vc{y}) , \vc{y}\rangle. 
\]
For example, for any $q\geq 2$, consider the Bregman divergence corresponding to $F(\vc{u}) = \norm{\vc{u}}_q^q=\sum_{i=1}^n u_i^q$ and grounded at $\vc{y}$. Note that $\langle \nabla F(\vc{u}), \vc{u}\rangle = q \cdot F(\vc{u})$ for all $\vc{u}$. Therefore, 
we can choose $\gamma = 1/q$, $\cal{H} = (1-\gamma) (\nabla F(\vc{0})-\nabla F(\vc{y}))$, $\hat{H}\equiv 0$, and $\psiup = (q-1)\norm{\vc{y}}_q^q$, in \autoref{condn:f-bnd}. Then, we can state a parallel to \autoref{thm:LS-gendata} for the regularized Bregman divergence minimization, corresponding to $F(\vc{u}) = \norm{\vc{u}}_q^q$, with the only difference in the bound on $\eta$, namely
\begin{align*}
\eta < 
q(q-1)\norm{\vc{y}}_q^{q-2} \cdot 
\left( 
\alpha  \norm{\vc{y}}_q^2 -  \max_{i\in[p]}\, \vc{y}^T\vc{x}_i
\right) . 
\end{align*}
in stead of \eqref{eq:eta-condn-LS}. 
With $q=2$, the above reduces to \eqref{eq:eta-condn-LS}.

\section{Identifying the Extreme Rays}\label{sec:ray-id}
The characterization provided in \autoref{sec:main} can be directly turned into an algorithm, as hinted in the statement of \autoref{thm:main-detailed} and presented as the procedure in \autoref{proc:pseudo}. More specifically, we can form the cone $\T$, as a conic combination of shifted columns of $\mt{X}$, and identify its extreme rays, followed by solving a smaller optimization problem. The main challenge in such approach is an efficient identification of the extreme rays. 

This problem is equivalent to vertex identification for the convex hull of a given set of points. For each ray, we test whether it belongs to the conic hull of the other rays. For example, we can solve a linear feasibility problem $\min_{\vc{\lambda}} \{ 0 :~ \vc{b} = \mt{A}\vc{\lambda} ,~\vc{\lambda}\geq \vc{0}\}$ where $\vc{b}$ is any one of $\vc{x}_i-\alpha\vc{y}$, $i\in [p]$, and $\mt{A}$ collates the rest of such vectors. We can alternatively use a quadratic program for minimizing the distance of the query point to the convex hull of the other points; e.g., see \cite{clarkson2010coresets,kalantari2015characterization} among others. In any case, this way, finding all of the extreme rays takes 
\[p \cdot \cc(p-1,n),\] 
where $\cc(p-1,n)$ is the cost of solving the aforementioned linear feasibility problem or some other conic hull membership test.  

\paragraph{Output-Sensitivity.} 
On the other hand, we can employ the output-sensitive approach of \cite{clarkson1994more,Ottmann1995Enumerating}. Suppose we have access to a cone membership test which also provides a separating hyperplane when the queried point does not belong to the conic hull; e.g., as for the aforementioned linear feasibility program or distance minimization oracles. Then, instead of testing each ray against {\em all other rays}, we can maintain a list of {\em extreme rays identified so far}, and for each next ray, determine whether it is in their conic hull or not. If it is, then we discard this ray (some care is needed; see \autoref{line:if-in-pos} in \autoref{alg:clarkson1994more}), and if it is not, we can get an infeasibility certificate (a separating hyperplane) which then allows for identifying an extreme ray using a maximum angular search (a modification of the maximum inner product search); some care in needed, see \autoref{line:extrayid-Ij} in \autoref{alg:clarkson1994more}. This way, the overall cost reduces to 
\[p\cdot \cc(s, n),\] 
plus the linear cost of maximum angular searches, where $s$ is the {\em actual number of extreme rays of the cone}. This, for a fixed $n$, is $O(sp)$ time. Further improvements 
to these original output-sensitive methods exist \cite{chan1996output}. Several approximation and rounding techniques can be used to make the feasibility problems or the step of maximum angular search faster; e.g., see \cite{kumar2013fast} for various implementation details. As another example, as proposed by \cite{jalali2017subspace}, the linear programs can be solved only approximately using any feasible method while yielding the exact feasibility/infeasibility result as well as a correct infeasibility certificate (hence maintaining the output-sensitivity property of the aforementioned algorithm.) See the optimization program (P3) in \cite[Section 7]{jalali2017subspace} as well as the discussions in there. Even the order in which we process the points affects the count of extreme rays so far; we can first seek a small but diverse set of extreme rays and then process the columns of $\mt{X}$ that are least likely to correspond to extreme rays, such as those with the smallest inner products with $\vc{y}$ when $\mt{X}$ is symmetric;  resembling ideas in correlation screening. We leave this direction for now as future work.

\paragraph{The Angular Maximization Oracle.}
It is easy to extend \cite{clarkson1994more}'s algorithm to the case of conic hulls instead of convex hulls. While the first oracle mentioned above can be similarly implemented as a linear feasibility program, the second oracle has to be modified. In essence, instead of linear optimization over generating rays, we need to work with a base of the cone; e.g. see \cite{kumar2013fast}. Any $\vc{g}$ in the relative interior of the dual cone defines a base for a pointed convex cone. With such $\vc{g}$, define the {\em angular maximization oracle} (AMO) as 
\begin{align}\label{eq:AMO}
\op{AMO} (\vc{v}; \cal{Z},\vc{g})	\coloneqq \op{argmax}_{\vc{z}\in\cal{Z}} \frac{ \langle \vc{v} , \vc{z} \rangle }{ \langle \vc{g} , \vc{z} \rangle }, 
\end{align}
where we assume $\op{pos}(\op{conv}(\cal{Z}))$ is a pointed cone, $\vc{0}\not\in \cal{Z}$, and $\vc{g}$ is any point in $\op{rel\,int}( \cal{Z}^\star)$. 
Observe that the above is equivalent to 
\begin{align}\label{eq:AMO-equiv}
\max_{\vc{z}\in\cal{Z},\theta} \big\{ \theta :~ \langle \vc{v}- \theta \vc{g} , \vc{z} \rangle \geq 0 \big\}. 
\end{align}
The optimal value of \eqref{eq:AMO-equiv} is equal to $-\lambda_{\min}(-\vc{v}; \cal{Z}^\star)$ for the function $\lambda_{\min}$ defined in \cite{renegar2016efficient}.

\paragraph{More on \autoref{alg:clarkson1994more}.}
The required input $\vc{g}\in\op{rel\,int}(\cal{Z}^\star)$ means that \autoref{alg:clarkson1994more} requires $\op{pos}\op{conv}(\cal{Z})$ to be a pointed cone. 
Note that $\vc{g}\in\op{rel\,int}(\cal{Z}^\star)$ is equivalent to having a positive inner product with all $\vc{z}\in\cal{Z}$. When $\theta=1$, we have a certificate that $\vc{v}$ has a positive inner product with at least one member of $\cal{Z}$; i.e., $\vc{v}\not\in \cal{Z}^\circ$. Therefore, we can use $\vc{v}$ to identify an extreme ray among such members of $\cal{Z}$. Note that Maximum Cosine Similarity Search (MCSS) is not useful for our purposes, as $\vc{v}$ might belong to $\op{pos}\op{conv}(\cal{Z})$ in which case MCSS will not output an extreme ray. Instead, we use the aforementioned Angular Maximization Oracle, which finds an extreme point of a linear base of $\op{pos}\op{conv}(\cal{Z})$, namely the base defined by normal $\vc{g}$. However, the AMO may result in more than one ray within a face of $\op{pos}\op{conv}(\cal{Z})$; it is easy to see that we in fact get a face. In such case, we feed these maximizers into $\extrayid$ to identify the extreme rays. Note that the extreme rays of a face of a cone are extreme rays for the cone itself. Moreover, $\vc{g}$ is a valid input for this subset of $\cal{Z}$. 
Furthermore, note that in both of the unions in \autoref{line:Ij} and \autoref{line:IIj} of \autoref{alg:clarkson1994more} are disjoint unions; no repetitive elements are being added. Finally, initializing $\cal{I}$ is helpful. For example, when $\cal{I}=\emptyset$ and all points are multiples of each other, then all points will be maximizers on \autoref{line:Jj} of \autoref{alg:clarkson1994more} and we would call the function on the same set; which is problematic.

Given $\mt{X}$ and $\vc{y}$, \autoref{condn:ver-pos} guarantees that $\T$ is a pointed cone. Therefore, the relative interior of its dual cone has a nonzero member. However, it is not necessarily always trivial to generate such $\vc{g}$, for $\vc{z}_i \coloneqq \vc{x}_i-\vc{y}$, $i\in[p]$. On the other hand, in some practical cases of interest such task is easy. 
For example, when all of $\vc{x}_i$ and $\vc{y}$ have unit $\ell_2$ norms, or when $\norm{\vc{y}}_2^2 > \max_{i\in[p]} \langle \vc{y}, \vc{x}_i \rangle$ then we can use $\vc{g}= -\vc{y}$.


\section{The Three Main Ingredients of \autoref{thm:main-detailed}}\label{sec:props}

In this section, we prove the three main propositions used in the proof of \autoref{thm:main-detailed}. Recall the definition $\cal{L}(\vc{\beta}) \coloneqq f(\mt{X}\vc{\beta}) + \eta r(\vc{\beta}) - \eta$. 

\begin{proposition}\label{lemma:normless1-genloss}
Assume \autoref{cond:f-star} holds for the extended real valued function $f:\mathbb{R}^n\to\Rext$ with some $\psilow$ and some $\vc{y}\in\mathbb{R}^n$. Moreover, assume $\mt{X}\in\mathbb{R}^{n\times p}$ and the same $\vc{y}$ satisfy \autoref{condn:ver-pos}. Define
\[
\T \coloneqq \op{pos}\op{conv}\bigl\{ \vc{x}_i - \vc{y} :~ i\in[p]   \bigr\}. 
\]
Consider any nonzero global minimizer $\vc{\beta}^\star$ for the optimization program in \eqref{prob:loss-reg-lin-pos}, and without loss of generality assume $\beta^\star_1 \neq 0$. If $\vc{x}_1 - \vc{y} \not\in \op{ext\,ray\,}\T$ then $\cal{L}(\vc{\beta}^\star) \leq \psilow $ and $0< \langle \vc{1}, \vc{\beta}^\star \rangle \leq 1$. 
\end{proposition}
We provide a constructive proof of \autoref{lemma:normless1-genloss}. 
\begin{proof}[Proof of \autoref{lemma:normless1-genloss}]
First, we examine the data in light of \autoref{condn:ver-pos} and the assumption on $\vc{x}_1-\vc{y}$ not being an extreme ray. 
Note that, in \autoref{condn:ver-pos}, $\vc{y}$ is assumed to be outside of the convex hull of the columns of $\mt{X}$. Therefore, $\T$ is a pointed convex polyhedral cone. 
The assumption on $\vc{x}_1-\vc{y}$ implies that there exists $ \vc{\lambda} \geq \vc{0}$ such that 
$\vc{x}_1 - \vc{y} = \sum_{i=2}^p \lambda_i (\vc{x}_i-\vc{y})$. 
Setting $\lambda_1  = 0$, the above implies 
\begin{align}\label{eq:x1-repr}
\vc{x}_1 = (1-\vc{1}^T\vc{\lambda})\vc{y} + \sum_{i=2}^p \lambda_i \vc{x}_i.
\end{align}
If $1-\vc{1}^T\vc{\lambda}\geq 0$ then $\vc{x}_1$ is in the convex hull of $\vc{y}$ and other columns of $\mt{X}$ which contradicts \autoref{condn:ver-pos}. 
Therefore, $\vc{1}^T\vc{\lambda} > 1$. 

Next, define $\vc{e}^\star\coloneqq \vc{y} - \mt{X}\vc{\beta}^\star$. 
Considering \eqref{eq:x1-repr}, we get
\begin{align*}
\vc{y} 
= \vc{e}^\star + \beta^\star_1 \vc{x}_1 + \sum_{i=2}^p  \beta^\star_i \vc{x}_i 
= \vc{e}^\star + \beta^\star_1 \left( (1-\vc{1}^T\vc{\lambda})\vc{y} + \sum_{i=2}^p \lambda_i \vc{x}_i \right) + \sum_{i=2}^p  \beta^\star_i \vc{x}_i 
\end{align*}
which, for $\tau\coloneqq 1 + \beta^\star_1 (\vc{1}^T\vc{\lambda}-1)> 1$, gives 
\begin{align}\label{eq:rep1}
	\vc{y} 
	= \frac{1}{\tau} \vc{e}^\star + \sum_{i=2}^p  \frac{\beta^\star_i + {\beta^\star_1} \lambda_i}{\tau}   \vc{x}_i	
	= \frac{1}{\tau} \vc{e}^\star + \mt{X} \widetilde{\vc{\beta}},
\end{align}
where $\widetilde{\beta}_1 \coloneqq 0$ and $\widetilde{\beta}_i \coloneqq  (\beta^\star_i + {\beta^\star_1} \lambda_i)/\tau \geq 0$ for $i=2,\ldots,p$. 
Note that $\widetilde{\vc{\beta}}\in\mathbb{R}_+^p$ is feasible (regardless of $\vc{\lambda})$: it is non-negative and $\mt{X}\widetilde{\vc{\beta}} = \vc{y}- \frac{1}{\tau} \vc{e}^\star$ lies on the open line segment between $\vc{y}$ and another point in the domain of $f$, namely $\mt{X}\vc{\beta}^\star$. 
Moreover, $\tau(\langle \vc{1}, \widetilde{\vc{\beta}}\rangle -1) = \langle\vc{1},\vc{\beta}^\star\rangle - 1$. 
\autoref{eq:rep1} is another noisy representation of $\vc{y}$ in terms of the columns of $\mt{X}$ in addition to 
\begin{align}\label{eq:rep2}
\vc{y} = \vc{e}^\star + \sum_{i=1}^p \beta^\star_i \vc{x}_i. 
\end{align}

With these representations of the data, we turn into the optimization program and the properties of the function. 
By optimality of $\vc{\beta}^\star$ (implying $\cal{L}(\vc{\beta}^\star)\leq \cal{L}(\widetilde{\vc{\beta}})$), and combining the two representations in~\eqref{eq:rep1} and~\eqref{eq:rep2}, we have 
\begin{align*} 
\cal{L}(\vc{\beta}^\star)
&=
f(\mt{X}\vc{\beta}^\star) + \eta \langle \vc{1}, \vc{\beta}^\star \rangle  -\eta \\
&=
f(\vc{y} - \vc{e}^\star) + \eta \langle \vc{1}, \vc{\beta}^\star \rangle  -\eta \\
& 
\leq   
f(\vc{y} - \frac{1}{\tau}\vc{e}^\star)
+ \eta \langle \vc{1}, \widetilde{\vc{\beta}} \rangle  -\eta \\
&= 
f(\vc{y} - \frac{1}{\tau}\vc{e}^\star)
+ \frac{\eta}{\tau}
(\langle \vc{1}, \vc{\beta}^\star \rangle +{\beta^\star_1}(\vc{1}^T\vc{\lambda} - 1)) -\eta  \\
&= 
f(\vc{y} - \frac{1}{\tau}\vc{e}^\star)
+ \frac{\eta}{\tau} \langle \vc{1}, \vc{\beta}^\star \rangle - \frac{\eta}{\tau} . 
\end{align*}
Therefore, 
\begin{align}\label{eq:f-vs-beta}
f(\vc{y} - \vc{e}^\star)
- f(\vc{y} - \frac{1}{\tau}\vc{e}^\star)
\leq 
\eta \cdot (1-\frac{1}{\tau}) \cdot(1-\langle \vc{1}, \vc{\beta}^\star \rangle )
\end{align}
Recall that $\tau>1$. 
By our second assumption on $f$ in \autoref{cond:f-star}, the left-hand side of \eqref{eq:f-vs-beta} is non-negative. 
Therefore, we get 
$\langle \vc{1}, \vc{\beta}^\star \rangle \leq 1$. 
Moreover, we have 
\begin{align*}
(1-\frac{1}{\tau}) \cal{L}(\vc{\beta}^\star) 
\leq  
f(\vc{y} - \frac{1}{\tau}\vc{e}^\star) - \frac{1}{\tau} f(\vc{y}-\vc{e}^\star)
\leq 
(1-\frac{1}{\tau})  \psilow
\end{align*}
where the first inequality is a re-arrangement of \eqref{eq:f-vs-beta} and the second inequality is an implication of our our first assumption on $f$ in \autoref{cond:f-star}.  
Therefore, $\cal{L}(\vc{\beta}^\star)\leq \psilow$. 
\end{proof}

\begin{remark}[An Alternative Proof for \autoref{lemma:normless1-genloss}, assuming convexity]\label{rem:pf-normless1-genloss-subdiff}
When $f$ is proper, lsc, convex, and finite everywhere, and $\vc{y}$ is a minimizer of $f$ over any superset of 
$\op{col}_+(\mt{X})$, we can simplify the proof. Consider the aforementioned proof of \autoref{lemma:normless1-genloss} up to the point where we established $\vc{1}^T\vc{\lambda}>1$ right after \eqref{eq:x1-repr}. 

Consider the optimality conditions: there exists $\vc{g}\in (\partial f)(\mt{X}\vc{\beta}^\star)$ for which $\mt{X}^T\vc{g} + \eta \vc{1} \geq \vc{0}$, $\langle \vc{g}, \mt{X}\vc{\beta}^\star \rangle+\eta \vc{1}^T\vc{\beta}^\star = 0$, and, $\beta_i^\star >0$ implies $\vc{x}_i^T\vc{g} = -\eta$. Taking the inner product of both sides of \eqref{eq:x1-repr} with $\vc{g}$ yields
\[
-\eta = \langle \vc{x}_1 , \vc{g} \rangle 
= (1-\vc{1}^T\vc{\lambda}) \langle \vc{y} , \vc{g} \rangle
+ \sum_{i=2}^p \lambda_i \langle \vc{x}_i , \vc{g} \rangle 
\geq (1-\vc{1}^T\vc{\lambda}) \langle \vc{y} , \vc{g} \rangle - \eta \vc{1}^T \vc{\lambda}, 
\]
where we used $\beta_1^\star >0$, $\lambda_i\geq 0$ for $i\in[p]$, $\lambda_1 =0$, and the aforementioned optimality condition. Re-arrangement, and the fact that $\vc{1}^T\vc{\lambda}-1 >0$ yields $\langle \vc{y} , \vc{g} \rangle  \geq -\eta$. 
Finally, consider a first-order expansion of $f$ at $\mt{X}\vc{\beta}^\star$ (from convexity) to get
\begin{align*}
	f(\vc{y}) 
	\geq f(\mt{X}\vc{\beta}^\star)  + \langle \vc{g}, \vc{y} - \vc{X}\vc{\beta}^\star\rangle 
	\geq f(\mt{X}\vc{\beta}^\star) - \eta + \eta \vc{1}^T\vc{\beta}^\star = \cal{L}(\vc{\beta}^\star). 
\end{align*} 
Combining the above with $f(\vc{y}) \leq f(\mt{X}\vc{\beta}^\star)$ and $\eta>0$ yields $\vc{1}^T\vc{\beta}^\star\leq 1$. Note that $\psilow= f(\vc{y})$ here. 
\end{remark}

As a result of \autoref{lemma:normless1-genloss}, some of the results of this paper may be stated only on $\{\vc{\beta}\in\mathbb{R}_+^p:~\vc{1}^T\vc{\beta}\leq 1\}$ to enable us to deal with unbounded sublevel sets. 

\begin{proposition}
\label{lem:opt-pos-gen}
Suppose $f$ satisfies the conditions of \autoref{lem:opt-conds-pos}. Then, for any $\gamma\in[0,1]$ and for any local minimum of \eqref{prob:loss-reg-lin-pos}, we have
\begin{align*}
    \cal{L}(\vc{\beta}^\star)  
    \geq 
    \big[  
	- (f_\gamma)^\star(\vc{h})    
    -  \etaone - \gamma \eta \big] 
    +  (1-\vc{1}^T\vc{\beta}^\star) \big[ \etaone - (1-\gamma)\eta \big],
\end{align*}
where $\etaone = \rc(- \mt{X}^T\vc{h})$ for any choice of $\vc{h}\in \op{dom}(f_\gamma^\star) \neq \emptyset$. 
\end{proposition}
\begin{proof}[Proof of \autoref{lem:opt-pos-gen}]
The optimality condition in \autoref{lem:opt-conds-pos} reads as  
\[
\eta \vc{1}^T\vc{\beta}^\star + \dd f( \mt{X}\vc{\beta}^\star )(\mt{X}\vc{\beta}^\star ) \geq 0.
\] 
Therefore, for any $\gamma\in[0,1]$, 
\begin{align*}
    \gamma \cdot \eta \vc{1}^T\vc{\beta}^\star 
    &\geq - \gamma \cdot \dd f( \mt{X}\vc{\beta}^\star )(\mt{X}\vc{\beta}^\star )
    \\
    &= -\langle -\mt{X}^T\vc{h}, \vc{\beta}^\star \rangle      
    -  f( \mt{X}\vc{\beta}^\star )  
    + \left[ f( \mt{X}\vc{\beta}^\star )
    	- \gamma \cdot \dd f (\mt{X}\vc{\beta}^\star)(\mt{X}\vc{\beta}^\star) 
    	- \langle\vc{h}, \mt{X}\vc{\beta}^\star \rangle  \right] 
    \\
    &\geq 
    - (\vc{1}^T\vc{\beta}^\star) \cdot \etaone
    -  f( \mt{X}\vc{\beta}^\star ) 
     - (f_\gamma)^\star(\vc{h})
\end{align*}
where we used the gauge polarity (since $\vc{\beta} \in \mathbb{R}_+^p = \op{dom} r$) in \eqref{eq:gauge-polar-ineq} as well as the definition of the conjugate function at $\vc{h}\in \op{dom}(f_\gamma)^\star$. The domain of the conjugate function is non-empty as $f_\gamma$ is proper. 
Rearrangement yields 
$f(\mt{X}\vc{\beta}^\star) + (\gamma \eta +\etaone) (\vc{1}^T\vc{\beta}^\star) \geq - (f_\gamma)^\star(\vc{h})$ which in turn is equivalent to the claimed inequality without further bounding. 
\end{proof}

\begin{proposition}\label{lem:perspective-props}
Suppose $r:\mathbb{R}^n\to\Rext$ is positively homogeneous and $f:\mathbb{R}^n\to\Rext$. Consider $\cal{L}(\vc{\beta}) \coloneqq f(\mt{X}\vc{\beta}) + \eta r(\vc{\beta}) - \eta$. 
For any fixed value of $\alpha>0$, the following statements hold: 
\begin{itemize}

\item $\vc{\beta}^\star$ is a minimizer of $\cal{L}$ if and only if $\alpha \vc{\beta}^\star$ is a minimizer of $(\alpha\star \cal{L})$.

\item If $f$ satisfies \autoref{cond:f-star} with $\vc{y}$ and $\psilow$, then $(\alpha\star f)$ also satisfies \autoref{cond:f-star} with $\alpha\vc{y}$ and $\alpha\psilow$.

\item $(\alpha\star f)_\gamma = (\alpha\star f_\gamma)$ for all $\gamma\in\mathbb{R}$. 

\item If $f$ satisfies \autoref{condn:f-bnd} with $\gamma$, $\psiup$, and $H$ where $\op{dom}(H^\star) = \cal{H}$, then $(\alpha\star  f)$ also satisfies \autoref{condn:f-bnd} with $\gamma$, $\alpha\psiup$, and $(\alpha\star H)$ where $\op{dom} ((\alpha\star H)^\star) = \cal{H}$.

\end{itemize}
\end{proposition}
\begin{proof}[Proof of \autoref{lem:perspective-props}]
First, optimality of $\vc{\beta}^\star$ for $\cal{L}$ implies 
$f(\mt{X}\vc{\beta}^\star) +\eta r(\vc{\beta}^\star) - \eta \leq f(\mt{X}\vc{\beta}) +\eta r(\vc{\beta}) - \eta$ for all $\vc{\beta}$, which in turn implies 
$\alpha f(\frac{1}{\alpha}\mt{X}(\alpha\vc{\beta}^\star)) + \eta r(\alpha\vc{\beta}^\star) - \alpha\eta \leq \alpha f(\frac{1}{\alpha}\mt{X}(\alpha\vc{\beta})) + \eta r(\alpha\vc{\beta}) - \alpha\eta$ for all $\vc{\beta}$, which by a change of variable $\vc{\beta}$ to $\vc{\beta}/\alpha$ (since $\alpha>0$ is fixed) implies
\[
(\alpha\star \cal{L})(\alpha\vc{\beta}^\star) = 
\alpha f(\frac{1}{\alpha}\mt{X}(\alpha\vc{\beta}^\star)) + \eta r(\alpha\vc{\beta}^\star) - \alpha\eta \leq \alpha f(\frac{1}{\alpha}\mt{X}\vc{\beta}) + \eta r(\vc{\beta}) - \alpha\eta
= (\alpha\star \cal{L})(\vc{\beta}), 
\] 
for all $\vc{\beta}\in\op{dom}(\alpha\star f) = \alpha \cdot \op{dom} f$. Since $\op{dom}f$ is star-convex with respect to $\vc{y}$, we have $\op{dom}(\alpha\star f)$ is star-convex with respect to $\alpha\vc{y}$. If $f$ satisfies the first condition in \autoref{cond:f-star} with $\vc{y}$ and $\psilow$ then it is easy to see that $(\alpha\star f)$ satisfies the same condition with $\alpha \vc{y}$ and $\alpha \psilow$. Similarly, the second part of \autoref{cond:f-star} holds for $(\alpha\star  f)$ with $\alpha \vc{y}$ if it holds for $f$ with $\vc{y}$. 

From the definition of subderivatives (\cite[Definition 8.1]{rockafellar2009variational}), it is clear that $\dd (\alpha\star f) (\alpha\vc{u})(\alpha\vc{w}) = \alpha \cdot \dd f(\vc{u})(\vc{w})$. $(\alpha\star f)_\gamma = (\alpha\star f_\gamma)$ for all $\gamma\in\mathbb{R}$. Moreover, if $f$ satisfies the inequality in \autoref{condn:f-bnd} with $\gamma$, $\psiup$, and $H$, then $(\alpha\star f)$ satisfies the same inequality with $\gamma$, $\alpha\psiup$, and $(\alpha\star H)$. 
Note that epi-multiplication by $\alpha>0$ scales the conjugate function by $\alpha$; \cite[Equation 11(3)]{rockafellar2009variational}. Therefore, the conjugate of $(\alpha\star H)$ has the same domain as the conjugate of $H$ and is non-positive there.
\end{proof}


\section{More On the Geometric Condition}
\label{sec:originT}
In this section, we provide further insight on \autoref{condn:ver-pos} as well as on the cone $\T$. While \autoref{proc:pseudo}, in combination with \autoref{thm:main-detailed}, provides a `computational' approach for identifying a superset of supports, the results in this section are mostly `theoretical' (see the remarks right after \autoref{cor:Ky1WT}) and pertain to providing further insight. More specifically, we define a few more geometric objects, in relation to $\op{conv}(\mt{X})$ and $\vc{y}$, and using these objects, we provide (i) a sufficient condition for \autoref{condn:ver-pos}, as well as, (ii) a complete characterization of $\T$ in a restricted setup. In this restricted setup, we can establish an interesting connection between the solutions to the regularized problem in \eqref{prob:loss-reg-lin-pos} and solutions to the constrained problem (gauge minimization subject to affine constraints) in \eqref{eq:minL1-constrained-pos} corresponding to $f(\vc{u}) = \indic(\vc{u}; \{\vc{y}\})$. \autoref{def:K-defs} gathers these definitions and \autoref{cor:Fy-subset-T-ifnoncover} and \autoref{cor:Ky1WT} on page \pageref{cor:Fy-subset-T-ifnoncover} state a summary of the results in this section.

\begin{definition}\label{def:K-defs}
For a given (bounded) polytope $\Kc\subset\mathbb{R}^n$; 
\begin{itemize}
\item Consider the corresponding gauge function defined as $\gauge(\vc{y}; \Kc) \coloneqq \inf\{\alpha\geq 0:~ \vc{y}\in \alpha \Kc\}$ over the extended real line where $\op{dom}\gauge(\cdot\,;\Kc)= \op{pos}(\Kc)$. By definition, $\vc{y}=\vc{0}$ if and only if $\gauge(\vc{y}; \Kc)=0$. Moreover, $\vc{y}/\gauge(\vc{y}; \Kc) \in \Kc$ for all $\vc{0}\neq \vc{y}\in \op{pos}(\Kc)$. 

\item 
Consider a facet description for $\Kc$ as 
\[
\Kc = \bigcap_{i=1}^K \{\vc{z}:~ \langle \vc{h}_i, \vc{z}\rangle \leq b_i \} ,
\] 
where $\vc{h}_i\neq \vc{0}$ and $b_i\in\mathbb{R}$, for all $i\in[K]$. Note that $\vc{0}\in\Kc$ implies $b_i\geq 0$ for all $i\in[K]$.

\end{itemize}
For $\vc{0}\neq \vc{y}\in \op{pos}(\Kc)$ (i.e., all $\vc{y}$ with $0<\gauge(\vc{y}; \Kc)<\infty$); 
\begin{itemize}

\item 
Denote by $\cal{J}(\vc{y})\subset [K]$ the set of indices for active constraints at $\vc{y}$; i.e., $i \in \cal{J}(\vc{y})$ if and only if $\langle \vc{h}_i, \vc{y}\rangle = b_i \cdot \gauge(\vc{y}; \Kc)$. 
Moreover, define $\cal{J}_1(\vc{y}) \coloneqq \{i\in \cal{J}(\vc{y}) :~ b_i\neq 0\}$. 

\item 
Define
\begin{align*}
	\msh(\vc{y}) \coloneqq \bigcup_{ i\in \cal{J}_1(\vc{y}) } \left\{ 
		\vc{x}\in \op{ver}(\Kc):~ \langle \vc{h}_i, \vc{x}\rangle = b_i 
	\right\} 
	= \left\{ 
		\vc{x}\in \op{ver}(\Kc):~ \cal{J}_1(\vc{x})\cap \cal{J}_1(\vc{y}) \neq \emptyset \right\} . 
\end{align*}

\item 
Define
\begin{align*}
	\mshbar(\vc{y}) \coloneqq 
	\left\{ \vc{x}\in \op{ver}(\Kc)\setminus\{\vc{0}\}:~ \cal{J}(\vc{x})\cap \cal{J}(\vc{y}) \neq \emptyset \right\}. 
\end{align*}

\item 
Remove the hyperplanes indexed by $\cal{J}(\vc{y})$ from the description of $\Kc$ and denote the new polyhedral set by $\Kc[{\backslash \vc{y}}]$. In other words, 
\begin{align*}
	\Kc[{\backslash \vc{y}}] 
	&\coloneqq 
	\{\vc{z}:~ \langle \vc{h}_i, \vc{z}\rangle \leq b_i ~~ \forall \, i\not\in \cal{J}(\vc{y})\} . 
\end{align*} 
Note that $\Kc[{\backslash \vc{y}}]$ may not be polytope; e.g. when $\Kc$ is the unit cube and the entries of $\vc{y}$ have distinct absolute values. In other words, $\gauge(\vc{z}; \Kc[{\backslash \vc{y}}])$ can be zero even if $\vc{z}$ is nonzero.

\item Define $\cal{F}(\vc{y}) = \cal{H}(\vc{y}) \cap \Kc$ where 
\[
	\cal{H}(\vc{y}) \coloneqq \bigl\{ \vc{z}:~  	
	\langle \vc{h}_i, \vc{z}\rangle = b_i 
	~~ \forall i\in\cal{J}(\vc{y})
	\bigr\}.
\]
Note that, $\vc{z}\in \cal{F}(\vc{y}) $ implies $\cal{J}(\vc{y})\subseteq \cal{J}(\vc{z})$.

\end{itemize}
\end{definition}

Note that a facet description of $\Kc$ may be derived by considering all the extreme points and extreme rays of the set $\overline{\Kc}^\circ$ where $\overline{\Kc} \coloneqq \op{conv}(\{\vc{0}\}\cup \Kc) = \{\vc{u}:~ \gauge(\vc{u};\Kc)\leq 1\}$. Similarly, $\cal{J}(\vc{y})$, hence all of the objects above, can also be defined through $\partial \gauge(\vc{y}; \Kc)$. In fact, $\msh(\vc{y})$ and $\mshbar(\vc{y})$ are closely related to the union of subdifferentials for the polar gauge at subgradients of $\gauge(\vc{y}; \Kc)$. 

\begin{remark}\label{rem:0int-JJ1}
When $\vc{0}\in \op{int}\Kc$, or equivalently when $b_i>0$ for all $i\in[K]$, we have $\cal{J}_1(\vc{y})=\cal{J}(\vc{y})$ which in turn implies $\msh(\vc{y}) = \mshbar(\vc{y})$, for all $\vc{0}\neq \vc{y}\in\op{pos}(\Kc)$. 
\end{remark}

\begin{lemma}\label{lem:J1y-nonempty}
For $\vc{0}\neq \vc{y}\in \op{pos}(\Kc)$ holds: $\cal{J}(\vc{y}) \neq \emptyset$, $\cal{J}_1(\vc{y}) \neq \emptyset$, $\vc{0}\not\in \cal{F}(\vc{y})$, and $\op{ver}(\cal{F}(\vc{y}))\subseteq \msh(\vc{y})$. 
\end{lemma}
\begin{proof}[Proof of \autoref{lem:J1y-nonempty}]
Since $\vc{y}/\gauge(\vc{y}; \Kc)\in\Kc$, we have $\cal{J}(\vc{y}) \neq \emptyset$. If $\cal{J}_1(\vc{y}) = \emptyset$, then $\langle \vc{h}_i,\vc{y} \rangle < b_i \cdot \gauge(\vc{y}; \Kc)$ for $i \not\in \cal{J}(\vc{y})$, while $\langle \vc{h}_i,\vc{y} \rangle=b_i\cdot \gauge(\vc{y}; \Kc)=0$ for $i \in \cal{J}(\vc{y})$. Since $\vc{0}\neq \vc{y}\in \op{pos}(\Kc)$, there exists $0<\epsilon< \gauge(\vc{y}; \Kc)$ where $\langle \vc{h}_i,\vc{y} \rangle \leq b_i \cdot (\gauge(\vc{y}; \Kc)-\epsilon)$ for all $i$. This implies $\vc{y}/(\gauge(\vc{y}; \Kc)-\epsilon)\in\Kc$ and contradicts the optimality criterion defining $\gauge(\vc{y}; \Kc)$. 
\end{proof}

\begin{remark}\label{rem:Fyfacet}
When $\cal{F}(\vc{y})$ is a facet of $\Kc$, and not a lower dimensional face, then $\cal{J}(\vc{y})=\cal{J}_1(\vc{y})$ is a singleton; assuming a non-redundant face description for $\Kc$. In such case, when $\vc{0}\neq \vc{y}\in\op{pos}(\Kc)$, we have $\op{ver}(\cal{F}(\vc{y})) =
\msh(\vc{y}) =
\mshbar(\vc{y})$. 
\end{remark}

\begin{lemma}\label{lem:Fy-facet-def}
For $\vc{0}\neq \vc{y}\in \op{pos}(\Kc)$, $\cal{F}(\vc{y})$ is a face of $\Kc$ and $\vc{y}/\gauge(\vc{y}; \Kc) \in \cal{F}(\vc{y})$. Moreover, $\cal{F}(\vc{y})$ is the unique proper face of $\Kc$ whose relative interior contains $\vc{y}/\gauge(\vc{y}; \Kc)$. 
\end{lemma}
\begin{proof}[Proof of \autoref{lem:Fy-facet-def}]
Recall the definition of $\cal{F}(\vc{y})$. Note that $\cal{H}(\vc{y})$ is an intersection of halfspaces defining $\Kc$. Therefore, $\cal{H}(\vc{y}) \cap \Kc$ is an intersection of facets of $\Kc$. Therefore, $\cal{F}(\vc{y})$ is a face of $\Kc$. 
Regarding $\vc{y}/\gauge(\vc{y}; \Kc) \in \cal{F}(\vc{y})$; it is clearly in $\Kc$. On the other hand, for any $i\in \cal{J}(\vc{y})$, we have $\langle \vc{h}_i, \vc{y}/\gauge(\vc{y}; \Kc)\rangle = b_i$. Therefore, it is in $\cal{H}(\vc{y})$. The claim is established. 

Since the face lattice is a complete lattice ordered with respect to inclusion, it remains to show that there are no smaller faces of $\Kc$ than $\cal{F}(\vc{y})$ that contain $\vc{y}/\gauge(\vc{y}; \Kc)$. But if that is the case, there would exist $j\not\in \cal{J}(\vc{y})$ for which the inequality is tight which is a contradiction with the definition of $\cal{J}(\vc{y})$. This finishes the proof.  
\end{proof}

\subsection{A Sufficient Condition}

\begin{lemma}\label{lem:0K-Ky1}
Suppose $\vc{0}\neq \vc{y}\in \op{pos}(\Kc)$ and $\vc{0}\in\Kc$. If $\gauge(\vc{y}; \Kc[{\backslash \vc{y}}] )<1$ then $\langle\vc{h}_i, \vc{y}\rangle < b_i$ (strict) for all $i\not\in \cal{J}(\vc{y})$. 
\end{lemma}
\begin{proof}[Proof of \autoref{lem:0K-Ky1}]
Since $\vc{0}\in\Kc$, we have $b_i\geq 0$ for all $i\in[K]$. Then, $\gauge(\vc{y}; \Kc[{\backslash \vc{y}}] )<1$ implies that for some small $\epsilon>0$ we have $\gauge(\vc{y}/(1-\epsilon); \Kc[{\backslash \vc{y}}] )<1$, which is equivalent to: $\langle \vc{h}_i , \vc{y}/(1-\epsilon) \rangle \leq b_i$ for all $i\not\in \cal{J}(\vc{y})$. If $b_i\neq 0$ then we get $\langle \vc{h}_i , \vc{y} \rangle < b_i$. If $b_i=0$, since $i\not\in \cal{J}(\vc{y})$, we get $\langle \vc{h}_i , \vc{y} \rangle < b_i \cdot \gauge(\vc{y}; \Kc) = 0$. Therefore, $\langle \vc{h}_i , \vc{y} \rangle < b_i$ for all $i\not\in \cal{J}(\vc{y})$. 
\end{proof}

The following provides a sufficient condition for \autoref{condn:ver-pos}. 

\begin{proposition}[A Sufficient Condition]\label{lem:K-y-norm-notcover}
Suppose $\vc{0}\neq \vc{y}\in \op{pos}(\Kc)$ and $\vc{0}\in\Kc$. 
If $\gauge(\vc{y}; \Kc[{\backslash \vc{y}}] )<1$, then 
\[
\op{ver}(\Kc) \setminus \{ \vc{y}/\gauge(\vc{y}; \Kc) \}
\subseteq \op{ver}( \op{conv}(\Kc\cup\{\vc{y}\}) ) .
\]
Furthermore, if $\vc{y}$ is not a multiple of a vertex of $\Kc$ and $\gauge(\vc{y}; \Kc)>1$ then, 
\[
\op{ver}(\Kc) \cup \{ \vc{y} \}
= \op{ver}( \op{conv}(\Kc\cup\{\vc{y}\}) ) .
\]
\end{proposition}
\begin{proof}[Proof of \autoref{lem:K-y-norm-notcover}]
Denote by $\vc{x}_1,\ldots,\vc{x}_p$ the list of vertices of $\Kc$ (without repetition.) 
Consider any vertex of $\Kc$ that is not equal to $\vc{y}/\gauge(\vc{y}; \Kc)$, namely $\vc{x}_1$. Contrapositively, suppose $\vc{x}_1$ is not a vertex of $\op{conv}(\Kc\cup\{\vc{y}\})$. Therefore, there exists $\vc{\lambda}\geq \vc{0}$ with $\vc{1}^T\vc{\lambda}=1$ where 
$\vc{x}_1 = \lambda_1 \vc{y}+ \sum_{i=2}^p \lambda_i\vc{x}_i$.  
Since $\vc{x}_1$ is a vertex of $\Kc$, $\lambda_1$ has to be nonzero. 
Next, 
we claim that $\cal{J}(\vc{x}_1) \cap \cal{J}(\vc{y})^c \neq \emptyset$. Otherwise, $\cal{J}(\vc{x}_1) \subseteq \cal{J}(\vc{y})$ which implies that $\vc{y}/\gauge(\vc{y}; \Kc)$ is a vertex and equal to $\vc{x}_1$ which is contradiction. 

Finally, 
take any $j\in \cal{J}(\vc{x}_1)$, $j\not\in \cal{J}(\vc{y})$, and take the inner product of the expression in the first paragraph with $\vc{h}_j$ to get 
$
b_j = \langle \vc{h}_j, \vc{x}_1 \rangle
=  \lambda_1 \langle \vc{h}_j, \vc{y} \rangle 
+ \sum_{i=2}^p  \lambda_i \langle \vc{h}_j, \vc{x}_i \rangle
< \lambda_1 b_j +  \sum_{i=2}^p  \lambda_i b_j = b_j 
$, 
where we used $\lambda_1\neq 0$ and the assumption $\gauge(\vc{y}; \Kc[{\backslash \vc{y}}] )<1$. This is a contradiction which finishes the proof for the first assertion. The second assertion is straightforward as $\gauge(\vc{y}; \Kc)>1$ implies that $\vc{y}$ is a vertex of $\op{conv}(\Kc\cup\{\vc{y}\})$. Moreover, $\op{ver}( \op{conv}(\Kc\cup\{\vc{y}\}) ) \subseteq \op{ver}(\Kc) \cup \{ \vc{y} \}$. 
\end{proof}

\subsection{A Characterization of the Extreme Rays}

\begin{proposition}[Extreme rays via a union of facets]\label{lem:extrays-union-faces}
For $\vc{0}\neq \vc{y}\in \op{pos}(\Kc)$ with $\gauge(\vc{y}; \Kc)>1$, assume 
\begin{align}\label{cond:verKY-notcover}
\op{ver}(\op{conv}(\Kc \cup \{\vc{y}\})) = \op{ver}(\Kc) \cup \{\vc{y}\}.
\end{align}
Then, for any $\vc{x}\in \msh(\vc{y})$, direction $\vc{x}-\vc{y}$ defines an extreme ray of $\T(\vc{y}) \coloneqq \op{pos}(\Kc-\{\vc{y}\})$.
\end{proposition} 
Note that \eqref{cond:verKY-notcover} does not necessarily imply $\gauge(\vc{y}; \Kc)>1$ unless $\vc{0}\in\Kc$. 
\begin{proof}[Proof of \autoref{lem:extrays-union-faces}]
Since $\gauge(\vc{y}; \Kc)>1$, we have $\vc{x}\neq \vc{y}$. 
Contrapositively, suppose $\vc{x}-\vc{y}$ is not an extreme ray. Therefore, there exists $K\in[n]$, and $\vc{x}_1, \ldots, \vc{x}_K\in\Kc$, all different from $\vc{x}$, and $\vc{\lambda}\in\mathbb{R}^K_+$, such that $\vc{x}-\vc{y} = \sum_{k=1}^K \lambda_k (\vc{x}_k-\vc{y})$, or equivalently, 
\[
\vc{x} = (1- \vc{1}^T\vc{\lambda}) \vc{y} + \sum_{k=1}^K \lambda_k \vc{x}_k . 
\]
Observe that $\vc{\lambda} \neq \vc{0}$ since $\vc{x}\in\Kc$ implies $\gauge(\vc{x};\Kc)\leq 1$ while the first assumption requires $\gauge(\vc{y};\Kc)>1$. If $\vc{1}^T\vc{\lambda} \leq 1$, then the above implies that $\vc{x}$, a vertex of~$\Kc$, is in the convex hull of $\vc{y}$ and $\vc{x}_1, \ldots, \vc{x}_K\in \Kc$, which contradicts \eqref{cond:verKY-notcover}. 
Therefore, $\vc{1}^T\vc{\lambda} >1$. Taking the inner product of both sides with 
$\vc{z}= \vc{h}_i/b_i$, for any $i\in\cal{J}_1(\vc{y})$ corresponding to $\vc{x}$ in the definition of $\msh(\vc{y})$, yields 
\[
1 = (1- \vc{1}^T\vc{\lambda}) \langle \vc{y}, \vc{z}\rangle + \sum_{k=1}^K \lambda_k \langle \vc{x}_k, \vc{z} \rangle
< 1- \vc{1}^T\vc{\lambda} + \sum_{k=1}^K \lambda_k = 1, 
\]
using $\langle \vc{y}, \vc{z}\rangle >1$ and $\langle \vc{x}_k, \vc{z} \rangle \leq 1$ for $k\in[K]$. 
The strict inequality implies a contradiction.
\end{proof}

\begin{proposition}\label{prop:Wbar-eq-erT}
Suppose $\vc{0}\in\Kc$, $\vc{0}\neq \vc{y}\in\op{pos}(\Kc)$, $\gauge(\vc{y}; \Kc[{\backslash \vc{y}}] )<1$, $\gauge(\vc{y}; \Kc)>1$, and $\vc{0}\neq \vc{x}\in\op{ver}(\Kc)$. 
If $\vc{x}-\vc{y}$ defines an extreme ray of $\T(\vc{y}) \coloneqq \op{pos}(\Kc-\{\vc{y}\})$ then $\vc{x}\in \mshbar(\vc{y})$. 
\end{proposition}
\begin{proof}[Proof of \autoref{prop:Wbar-eq-erT}]
Note that $\op{int}\Kc = \emptyset$ implies $\mshbar(\vc{y}) = \op{ver}(\Kc)\setminus \{\vc{0}\}$ and the claim holds trivially. Suppose otherwise. 
We claim that, with the given assumptions, if $\vc{x}\not\in \mshbar(\vc{y})$ then there exists $\epsilon\in(0,1)$ for which $\epsilon \vc{y}+ (1-\epsilon) \vc{x}\in \op{int}\Kc$. In such case, $(\epsilon \vc{y}+ (1-\epsilon) \vc{x})-\vc{y}  = (1-\epsilon)(\vc{x}-\vc{y})$ cannot be an extreme ray of $\op{pos}(\Kc-\{\vc{y}\})$. This is equivalent to $\vc{x}-\vc{y}$ not being an extreme ray.

Let us prove the claim. 
For a value of $\epsilon$, to be chosen, define $\vc{z}\coloneqq \epsilon\vc{y}+ (1-\epsilon)\vc{x}$. 
If $i\in \cal{J}(\vc{x})$ then $i\not\in \cal{J}(\vc{y})$; by our assumptions on $\vc{x}$ and the definition of $\mshbar(\vc{y})$. Therefore, using \autoref{lem:0K-Ky1}, we get $\langle\vc{h}_i, \vc{y}\rangle < b_i = \langle\vc{h}_i, \vc{x}\rangle$ which implies $\langle\vc{h}_i, \vc{z}\rangle < b_i$. 
On the other hand, if $i\not\in \cal{J}(\vc{x})$ then $\langle\vc{h}_i, \vc{x}\rangle < b_i$. Moreover, $\vc{y}\in\op{pos}(\Kc)$ implies $\gauge(\vc{y};\Kc)$ is finite. Therefore, for a small enough $\epsilon>0$ we get $\langle\vc{h}_i, \vc{z}\rangle < b_i$. 
All in all, we have $\langle\vc{h}_i, \vc{z}\rangle < b_i$ for all $i$. 
Now, by $\vc{0}\in\Kc$ we have (i) $b_i\geq 0$ for all $i$ and (ii) $\vc{z}\not\in \op{int}\Kc$ implies $\gauge(\vc{z}; \Kc)\geq 1$. Therefore, $\vc{z}\not\in \op{int}\Kc$ implies $\langle\vc{h}_i, \vc{z}\rangle < b_i \leq b_i \cdot \gauge(\vc{z}; \Kc)$, for all $i$, which implies $\vc{z}\in \op{int}\Kc$, a contradiction. Therefore, $\vc{z}\in \op{int}\Kc$ as claimed.
\end{proof}

All in all, we can state the following summary (of \autoref{lem:J1y-nonempty}, \autoref{lem:K-y-norm-notcover}, \autoref{lem:extrays-union-faces}, and \autoref{prop:Wbar-eq-erT}) regarding the objects defined in \autoref{def:K-defs}. Note that some of the above results have been weakened in the following; e.g., see \autoref{cond:verKY-notcover} and illustrations in \autoref{sec:illust}, as well as \autoref{lem:extrays-union-faces}, for some cases outside of the scope of \autoref{cor:Fy-subset-T-ifnoncover}.   
\begin{proposition}\label{cor:Fy-subset-T-ifnoncover}
Suppose $\vc{0}\in\Kc$, $\vc{0}\neq \vc{y}\in\op{pos}(\Kc)$,  
$\vc{y}\not\in \Kc$, 
and $\vc{y}$ is not a multiple of a vertex of $\Kc$. Consider \autoref{def:K-defs}. Moreover, assume 
\begin{align}\label{eq:Kyless1}
\gauge(\vc{y}; \Kc[{\backslash \vc{y}}] )<1.
\end{align}
Then, for $\cal{X}(\vc{y}) \coloneqq \{ \vc{x}\in \op{ver}(\Kc)\setminus\{\vc{0}\} :~ \vc{x}-\vc{y}\in \op{ext\,ray} \op{pos}(\Kc-\vc{y}) \}$, we have 
\begin{align}\label{cond:verKY-notcover-2}
\op{ver}(\Kc) \cup \{ \vc{y} \} = \op{ver}( \op{conv}(\Kc\cup\{\vc{y}\}) ) ,
\end{align}
and, 
\begin{align}
\op{ver}(\cal{F}(\vc{y})) \subseteq
\msh(\vc{y}) \subseteq
\cal{X}(\vc{y}) \subseteq
\mshbar(\vc{y}). 
\end{align}
\end{proposition}
\begin{corollary}[A complete characterization of $\T$]\label{cor:Ky1WT}
Suppose $\vc{0}\in\op{int}\Kc$, $\vc{y} \not\in \Kc$, and $\vc{y}$ is not a multiple of a vertex of $\Kc$. Consider \autoref{def:K-defs}. Moreover, assume $\gauge(\vc{y}; \Kc[{\backslash \vc{y}}] )<1$. 
Then, $\op{ver}(\Kc) \cup \{ \vc{y} \} = \op{ver}( \op{conv}(\Kc\cup\{\vc{y}\}) )$, and, 
\begin{align}
\op{ver}(\cal{F}(\vc{y})) \subseteq
\msh(\vc{y}) =
\cal{X}(\vc{y}) =
\mshbar(\vc{y}). 
\end{align}
Furthermore, if $\cal{F}(\vc{y})$ is a facet (and not a lower dimensional face), then
\begin{align}
\op{ver}(\cal{F}(\vc{y})) =
\msh(\vc{y}) =
\cal{X}(\vc{y}) =
\mshbar(\vc{y}). 
\end{align}
\end{corollary}
\begin{proof}[Proof of \autoref{cor:Ky1WT}]
As mentioned in \autoref{rem:0int-JJ1}, when $\vc{0}\in \op{int}\Kc$, the two sets $\msh(\vc{y})$ and $\mshbar(\vc{y})$ coincide, for all $\vc{y}\neq\vc{0}$. Note that since $\vc{0}\in\Kc$, $\vc{y}\not\in\Kc$ is equivalent to $\gauge(\vc{y}; \Kc)>1$. Then, apply \autoref{cor:Fy-subset-T-ifnoncover} to establish the first claim. 
The second claim can be established by \autoref{rem:Fyfacet}.
\end{proof}
A few remarks are in order.

\begin{itemize}
\item 
\autoref{cor:Ky1WT} provides a complete characterization of the extreme rays of $\T$ under the assumption in \eqref{eq:Kyless1}. Moreover, it establishes that \eqref{eq:Kyless1} is a sufficient condition for \autoref{condn:ver-pos}; i.e., for \eqref{cond:verKY-notcover-2} with the appropriate definition of $\Kc$. However, (i) \eqref{eq:Kyless1} is not necessary for \autoref{condn:ver-pos} (e.g., see \autoref{sec:illust}) and, (ii) as established in \autoref{thm:main-detailed}, when \autoref{condn:ver-pos} holds, we can identify a superset of the support of optimal solutions by finding the extreme rays of $\T$, for example using \autoref{alg:clarkson1994more}, and without reliance on a characterization as in \autoref{cor:Ky1WT}. Nonetheless, \autoref{cor:Ky1WT} should be viewed as a useful understanding for $\T$ under a restricted setup described by \eqref{eq:Kyless1}. 

\item 
As we will see in the next section, $\cal{F}(\vc{y})$ appears in the characterization for the support of the solutions of the constrained problem (gauge minimization subject to affine constraints) in \eqref{eq:minL1-constrained-pos}; \autoref{lem:constrained-Fy-supp}. On the other hand, $\T$ appears in our characterization for the solutions of \eqref{prob:loss-reg-lin-pos} and of \eqref{eq:minL1-constrained-pos}; in \autoref{thm:main-detailed} and \autoref{thm:constrained}, respectively. While computing $\cal{F}(\vc{y})$ is equivalent to solving the original problem, the characterization in terms of $\T$ provides a way to compute a superset of the support without solving the original problem. \autoref{cor:Fy-subset-T-ifnoncover} provides a framework to understand how these geometric objects could be understood in relation to each other, in a restricted setup described by \eqref{eq:Kyless1}.  

\item 
As an example, for the second part of \autoref{cor:Ky1WT}, when vertices of $\Kc$ and $\vc{y}$ are in so-called `general position,' $\cal{F}(\vc{y})$ is a facet. 

\item 
To gain intuition about $\msh$, we can examine it for a random polytope. For example, \cite[Theorem~10]{Reitzner2003Random} states that when $\vc{x}_1,\ldots,\vc{x}_p$ and $\vc{y}$ are random points on the boundary of any sufficiently smooth and curved convex set (see the statement of his theorem), then the expected number of facets of $\op{conv}(\{\vc{x}_1,\ldots,\vc{x}_p\})$ seen by $\vc{y}$ tends to a constant, as $p\to\infty$, that only depends on the dimension $n$ and is independent of the original convex set or the density function used for drawing the random points. 
Note that when $\vc{0}$ is in the interior of the polytope, $\msh(\vc{y})$ is contained in the union of such facets; by \autoref{lem:ver-on-bdry} and \autoref{lem:extrays-union-faces}. 

\item 
See \autoref{sec:illust} for illustrations and discussion of some corner cases relevant to the results of this section. For example, these examples illustrate that (i) all of the inclusion relations in \autoref{cor:Fy-subset-T-ifnoncover} could be strict in general, (ii) without the assumptions of \autoref{cor:Fy-subset-T-ifnoncover}, we may not get the claimed inclusions, and, (iii) the converse of \autoref{lem:K-y-norm-notcover} may not hold. 

\end{itemize}

\subsection{Further Results}

\begin{lemma}[A Necessary Condition]\label{lem:negTangent}
If $\mt{X}$ and $\vc{y}$ satisfy the conditions of \autoref{condn:ver-pos} then 
\begin{align}
\vc{y} \not\in \vc{x}_i - T(\vc{x}_i; \op{conv}(\mt{X})),
\end{align}
for all $i\in [p]$, where $T(\vc{x}; C)$ denotes the tangent cone to $C$ at $\vc{x}$. 
\end{lemma} 
\begin{proof}[Proof of \autoref{lem:negTangent}]
Suppose $\vc{y} \in \vc{x}_1 - T(\vc{x}_1; \op{conv}(\mt{X}))$. Therefore, there exists $\theta>0$ and $\vc{\lambda}\geq \vc{0}$ with $\vc{1}^T\vc{\lambda}=1$ for which $\vc{x}_1 + \theta (\vc{x}_1-\vc{y}) = \sum_{i=1}^p \lambda_i \vc{x}_i$. 
Rearrangements yields 
\[
\vc{x}_1 = \frac{\theta}{1+\theta-\lambda_1}\vc{y} + \sum_{i=2}^p \frac{\lambda_i}{1+\theta-\lambda_1} \vc{x}_i, 
\]
which contradicts \autoref{condn:ver-pos}. In the above, we used $\theta>0$ and $0\leq \lambda_1\leq 1$ to get $1+\theta-\lambda_1>0$.
\end{proof}

Complementary to \autoref{lem:negTangent}, we can state the following result. Note that we have $\vc{x}_i - T(\vc{x}_i; \op{conv}(\mt{X}))$ in \autoref{lem:negTangent}, while we have $\vc{x}_i + T(\vc{x}_i; \op{conv}(\mt{X}))$ in \autoref{lem:negTangent-extray}. \autoref{lem:negTangent-extray} is closely related to $\Kc[{\backslash \vc{y}}]$. 
\begin{lemma}\label{lem:negTangent-extray}
Under \autoref{condn:ver-pos}, if $\vc{y} \not\in \vc{x}_i + T(\vc{x}_i; \op{conv}(\mt{X}))$ then $\vc{x}_i-\vc{y}$ is an extreme ray of~$\T$. 
\end{lemma} 
\begin{proof}[Proof of \autoref{lem:negTangent-extray}]
Suppose $\vc{x}_1-\vc{y}$ is not an extreme ray of $\T$. Therefore, there exists $\vc{\lambda}\geq \vc{0}$ for which $\vc{x}_1-\vc{y} = \sum_{i=2}^p \lambda_i(\vc{x}_i\vc{y})$. \autoref{condn:ver-pos} requires $\vc{1}^T\vc{\lambda}>1$. Therefore, for $\theta\coloneqq 1 - 1/(\vc{1}^T\vc{\lambda})>0$ we have $\vc{x}_1 + \theta (\vc{y}-\vc{x}_1) =\sum_{i=2}^p \lambda_i\vc{x}_i\in \op{conv}(\mt{X})$. This implies $\vc{y}-\vc{x}_1 \in T(\vc{x}_1; \op{conv}(\mt{X}))$. 
\end{proof}

\begin{lemma}\label{lem:y-out-convX}
$\vc{y}\in\op{conv}(\begin{bmatrix}\vc{0}&\mt{X}\end{bmatrix}) 
\implies \norm{\vc{y}}_2^2 \leq \max_{i\in[p]} \langle \vc{y}, \vc{x}_i \rangle
\iff
\vc{y}\not\in \op{rel\,int}\T^\star$.
\end{lemma}
\begin{proof}[Proof of \autoref{lem:y-out-convX}]
The claim holds for $\vc{y}=\vc{0}$; suppose $\vc{y}\neq\vc{0}$. 
Since $\vc{y}\in\op{col}_+(\mt{X})\setminus \{\vc{0}\}$, there exists $\vc{\lambda}\geq \vc{0}$ where $\vc{y}=\mt{X}\vc{\lambda}$. Therefore, $0< \norm{\vc{y}}_2^2 = \langle \vc{y}, \mt{X}\vc{\lambda} \rangle$ which requires $\max_{i\in[p]} \langle \vc{y}, \vc{x}_i \rangle >0$. 

Considering the Lagrangian dual to the optimization problem in the definition of $\gauge(\cdot\,; \op{conv}(\mt{X}))$ (as in \eqref{eq:poly-reg}), and strong duality, yields  
$\gauge(\vc{y}\,; \op{conv}(\mt{X})) = \max_{\vc{\theta}} \{ \langle \vc{\theta},\vc{y}\rangle:~ \mt{X}^T\vc{\theta}\leq \vc{1}\}$. 
Next, 
it is clear from the definition that $\vc{y}\in\op{conv}(\begin{bmatrix}\vc{0}&\mt{X}\end{bmatrix})$ implies $\gauge(\vc{y}\,; \op{conv}(\mt{X}))\leq 1$ which in turn implies $\langle \vc{\theta},\vc{y}\rangle\leq 1$ for all $\vc{\theta}$ with $\mt{X}^T\vc{\theta}\leq \vc{1}$. Plugging $\vc{\theta} = \vc{y} / (\max_{i\in[p]} \langle \vc{y}, \vc{x}_i \rangle)$, which satisfies $\mt{X}^T\vc{\theta}\leq \vc{1}$, provides $\norm{\vc{y}}_2^2 \leq \max_{i\in[p]} \langle \vc{y}, \vc{x}_i \rangle$. The equivalence follows from the definition of $\T$.
\end{proof}

\begin{lemma}\label{lem:on-boundary}
Given $\mt{X}$, consider \eqref{prob:loss-reg-lin-pos}. Then, for any optimal solution $\vc{\beta}^\star$ and any $i\in[p]$ with $\vc{x}_i\neq \vc{0}$, if $\beta^\star_i \neq 0$ then $\vc{x}_i$ is not in the relative interior of $\op{conv}(\mt{X})$.
\end{lemma}
\begin{proof}[Proof of~\autoref{lem:on-boundary}]
For the purpose of \eqref{prob:loss-reg-lin-pos}, $\vc{0}$ can be added as a column of $\mt{X}$ without changing anything; e.g., see \eqref{eq:gauge-reform} together with the discussions right after \autoref{def:K-defs}. Therefore, assume $\vc{0}\in\op{conv}(\mt{X})$. 
Suppose $\vc{x}_1\neq\vc{0}$ is in the relative interior of $\op{conv}(\mt{X})$. Therefore, there exists a small enough ball in the affine hull of $\mt{X}$ and around $\vc{x}_1$ that is inside $\op{conv}(\mt{X})$. In particular, $\epsilon \vc{x}_1$ belongs to such ball for small enough $\epsilon>0$. Therefore, there exists $\vc{\lambda}\geq\vc{0}$ with $\vc{1}^T\vc{\lambda}=1$ and $\lambda_1=0$ where $(1+\epsilon)\vc{x}_1 = \mt{X}\vc{\lambda}$. Using this expression we get, $\mt{X}\vc{\beta}^\star = \mt{X}\widetilde{\vc{\beta}}$ where $\widetilde{\beta}_1\coloneqq 0$ and $\widetilde{\beta}_i\coloneqq \beta_i^\star + \lambda_i \beta_1^\star/(1+\epsilon)$ for $i=2,\ldots,p$. Observe that $\vc{1}^T \widetilde{\vc{\beta}} = \vc{1}^T \vc{\beta}^\star - \epsilon \beta_1^\star /(1+\epsilon)$ which is strictly less than $\vc{1}^T \vc{\beta}^\star$ when $\beta_1^\star>0$. This establishes the claim, using the optimality of $\vc{\beta}^\star$. 
\end{proof}

In light of \autoref{lem:on-boundary}, the process of identifying the interior points for $\op{conv}(\begin{bmatrix}\vc{0} & \mt{X}\end{bmatrix})$ and discarding them, has a persistent reduction property for the problems of the form \eqref{prob:loss-reg-lin-pos}.

\section{Discussions}\label{sec:disc-geom}

\subsection{Other Persistent Reductions}\label{sec:other-PR}
Here, we review two other examples of persistent reductions.

\subsubsection{Polyhedral Gauge Minimization with Affine Constraints}\label{sec:constrained}
Consider the following constrained optimization problem, 
\begin{align}\label{eq:minL1-constrained-pos}
	\min_{\vc{\beta}}& ~\{ \vc{1}^T\vc{\beta} :~
	\mt{X}\vc{\beta} = \vc{y},~ \vc{\beta} \in\mathbb{R}_+^p \}, 
\end{align} 
which can be viewed as an instance of \eqref{prob:loss-reg-lin-pos} with $f(\vc{u}) = \indic(\vc{u}; \{\vc{y}\})$. \cite{jalali2017subspace} (Supplementary Material; Lemma 17) was first to establish a superset guarantee (efficiently computable, in pre-processing) for \eqref{eq:minL1-constrained-pos} but under a specific setup where all of the columns of $\mt{X}$ as well as $\vc{y}$ have unit norms. In \autoref{thm:constrained}, we extend their result under \autoref{condn:ver-pos}. 
While \eqref{eq:minL1-constrained-pos} is an instance of \eqref{prob:loss-reg-lin-pos}, here we provide a geometric proof which combines the following well-known result (a consequence of complementary slackness; e.g., see \cite{soltanolkotabi2012geometric}) and the results of \autoref{sec:originT}. Let us first prove \autoref{lem:constrained-Fy-supp} in our notation.  

\begin{lemma}\label{lem:constrained-Fy-supp}
For any optimal solution of \eqref{eq:minL1-constrained-pos}, namely $\vc{\beta}^\star$, if $\beta^\star_i\neq 0$ then $\vc{x}_i\in\op{ver}(\cal{F}(\vc{y}))$, where $\cal{F}(\vc{y})$ is defined in \autoref{def:K-defs}.  
\end{lemma}
\begin{proof}[Proof of \autoref{lem:constrained-Fy-supp}]
For the optimization problem to be feasible and have a nonzero solution, we need $\vc{0}\neq \vc{y}\in \op{pos}(\Kc)$. Then, consider $\cal{F}(\vc{y})$ and any $i\in \cal{J}_1(\vc{y})$ which is nonempty by \autoref{lem:J1y-nonempty}. Since $\vc{y}/\gauge(\vc{y}; \op{conv}(\mt{X})) \in \op{conv}(\mt{X})$, we have  
\begin{align*}
b_i \cdot \gauge(\vc{y}; \op{conv}(\mt{X})) 
= \langle \vc{h}_i , \vc{y}/ \gauge(\vc{y}; \op{conv}(\mt{X})) \rangle  \cdot \gauge(\vc{y}; \op{conv}(\mt{X})) \qquad\qquad\\
= \langle \vc{h}_i , \vc{y} \rangle  
= \langle \vc{h}_i , \mt{X}\vc{\beta}^\star \rangle  
\leq (\vc{1}^T\vc{\beta}^\star) \cdot \max_{j\in[p]} \langle \vc{h}_i , \vc{x}_j \rangle 
\leq  (\vc{1}^T\vc{\beta}^\star) \cdot b_i, 
\end{align*}
for any optimal solution $\vc{\beta}^\star$. 
Since the right and left hand sides are equal, the inequalities should hold with equality which requires $\beta_i^\star =0 $ whenever $\langle \vc{h}_i , \vc{x}_j \rangle < b_i$; i.e., whenever $\vc{x}_j\not\in\op{ver}(\cal{F}(\vc{y}))$. 
\end{proof}

\begin{theorem}\label{thm:constrained}
Given $\mt{X} = \begin{bmatrix} \vc{x}_1 & \cdots & \vc{x}_p \end{bmatrix}\in\mathbb{R}^{n\times p}$ and $\vc{y}\in\mathbb{R}^n$, assume that $\mt{X}$ and $\alpha\vc{y}$ satisfy \autoref{condn:ver-pos}, for some $\alpha>1/\gauge(\vc{y};\op{conv}(\mt{X}))$. Alternatively, if $\vc{0}\in\op{conv}(\mt{X})$, assume that $\mt{X}$ and $\alpha\vc{y}$ satisfy \autoref{condn:ver-pos}, for some $\alpha>0$.
Consider the convex cone 
\[
\T \coloneqq 
\bigl\{ \sum_{i=1}^p \lambda_i (\vc{x}_i - \alpha \vc{y} ) :~ \lambda_i \geq 0   \bigr\}. 
\]
For any optimal solution $\vc{\beta}^\star$, if $\beta_i^\star\neq 0$ then $\vc{x}_i - \alpha\vc{y}$ is an extreme ray of $\T$. 
\end{theorem}
Note that forming $\cal{F}(\vc{y})$ is equivalent to solving \eqref{eq:minL1-constrained-pos} while $\T$ provides a computationally appealing alternative.
\begin{proof}[Proof \autoref{thm:constrained}]
\autoref{condn:ver-pos} does not preclude $\gauge(\vc{y}; \op{conv}(\mt{X}))<1$. The condition on $\alpha$ is to ensure that $\vc{y}$ is not `between' the origin and $\op{conv}(\mt{X})$. 
On the other hand, observe that scaling $\vc{y}$ in \eqref{eq:minL1-constrained-pos} by a positive number does not change the support of any of the optimal solutions. Now, \autoref{lem:J1y-nonempty} establishes $\op{ver}(\cal{F}(\vc{y}))\subseteq \msh(\vc{y})$. Therefore, applying \autoref{lem:extrays-union-faces} establishes the claim. 
\end{proof}
As mentioned above, the special case where all $\vc{x}_i$ and $\vc{y}$ are distinct and lie on the unit sphere was originally proved in \cite[Supplementary Material; Lemma 17]{jalali2017subspace}. Note that their proof for \eqref{eq:minL1-constrained-pos} can be simply generalized to allow for unnormalized $\mt{X}$ and $\vc{y}$, similar to the proof of \autoref{thm:main-detailed}, and without using the results of \autoref{sec:originT}.

\subsubsection{Monotone Left Unitarily Invariant Regularization} 
The procedure in Section 5.2 
of \cite{jalali2017variational} considers a regularized loss minimization problem, similar to \eqref{prob:loss-reg}, where the latent variable is a matrix (instead of a vector.) When the regularization function is left unitarily invariant and monotone with respect to the Loewner order (as all variational Gram functions \cite{jalali2017variational} are), then it is possible to certify {\em zero rows} in all of the optimal solutions (after a rotation, derived from the datapoints). Suppose the loss is defined on $\mathbb{R}^n$, the latent variable lives in $\mathbb{R}^{p\times m}$, and the measurements linearly map the latent variable to $n$ real numbers. The reduction in \cite{jalali2017variational} allows for reducing the number of rows of the latent variable from $p$ to $\min\{p, mn\}$. This reduction is unrelated to any other properties of the regularization or the loss other than those mentioned above (hence applies to all such problems), is exact, and can be done before any optimization is being performed. Therefore, it fits into the framework of persistent reductions described in this work. The subspace $S$ in this case is being derived via a QR decomposition applied to a certain matrix constructed from the data; see \cite[Section 5]{jalali2017variational} for details.


\subsection{Future Directions}\label{sec:future}
There seems to be room to generalize or apply the current results in a few directions as discussed next. 

It seems to us that the gauge assumption on the regularization could be relaxed (e.g., to gauge-like regularization). 
Moreover, it seems that the presented interval for $\eta$ can be expanded via alternative proof techniques; we also provide some concrete evidence in \autoref{rem:eta-bound}. 
Furthermore, we believe the discussion in \autoref{sec:originT} could be helpful in extending the results of this paper to regularization with non-polyhedral gauges.  
Lastly, it would be very interesting to quantify how deviation from the assumption in \autoref{condn:ver-pos} affects the outcomes. For example, consider attaching a proxy (could be in relation to convex layers \cite{Ottmann1995Enumerating} of $\T$) to each $\vc{x}_i$ which indicates how far the point lies inside $\T$; extreme rays will be assigned $0$. Then, we ask whether it is possible to use these (theoretical) proxies to provide a robust version of the guarantee provided in this paper. Such guarantee might be useful in devising fast approximation algorithms. 

On the other hand, 
the insight provided in this paper might be beneficial in designing better data augmentations (see \cite[Section 3.2.2]{she2010sparse} for the terminology.) 
Moreover, instead of using extreme ray identification for reduction {\em before} optimization, as in \autoref{proc:pseudo}, it might be possible to make efficient use of our result in designing active/working set methods for optimization. 
Finally, fast implementations of the proposed procedure and numerical comparisons with existing screening methods are of interest.

\subsection{Connections}\label{sec:connections}
In this section, we review a few major lines of work with connections to the focus of this paper and discuss similarities and differences.

\para 
{\em Dimensionality reduction} techniques have been long employed for reducing the dimensionality of the input data while retaining useful aspects of it. When employed before a particular downstream task, one often is careful about the effect of such linear or nonlinear transformation (from dimensionality reduction) on the output of the task at hand; e.g., how would a low-distortion embedding affect the outcome of a downstream clustering task. However, these dimensionality reduction techniques adhere to their own optimality criteria and it is rare that they do not distort the downstream estimation. On the other hand, in this terminology, our work can be viewed as an exact dimensionality reduction procedure specifically designed for a broad family of downstream tasks, namely a broad class of regularized regression problems for variable selection.

\para 
{\em Geometric data summarization} techniques, specifically {\em coresets} \cite{badoiu2002approximate, agarwal2002approximation}, are close to our approach and are designed for the specific downstream tasks, but they only aim for a good approximation; \cite{dasgupta2009sampling,clarkson2010coresets, jaggi2011sparse, reddi2015communication, Huggins2016Coresets, phillips2016coresets, samadian2020unconditional, tukan2020coresets, feldman2020core} and many more recent works in this area. Other methods, based on optimizing a submodular utility function (e.g., see \cite{balkanski2016learning}), also could be viewed as providing approximations when the utility function does not directly correspond to the downstream task. We view this as an important point of difference between our work and coresets. 
It is worth mentioning that in designing coresets, there is an important emphasis on their size being independent of the size of the input data set. The output-sensitive property of our reduction could be viewed as a counterpart to this property, although certifying an actual independence from the size of the data is not possible without assuming a generative model for the data.

\para 
{\em `Safe' screening} methods on the other hand aim for identifying subsets of the futures where none of the active features could be discarded; \cite{Ghaoui-feature-elim,xiang2017screening}. 
While screening methods in general (e.g., see \cite{fan2010selective} and references therein), and ``sure'' screening \cite{fan2008sure} as a successful example, aim for a goal similar to ours in this paper, the heuristic or probabilistic success of most of these methods distinguishes them from our approach.

Our reduction procedure has important differences with many safe screening methods, especially primal-dual methods. Unlike these methods, 
we can choose to screen for {\em any feature}, by performing a conic hull membership test (e.g., solving a linear feasibility program), we can handle all values of $\eta$ below an explicit threshold (see see \cite[Section~7]{xiang2017screening} for performance of screening methods with small $\eta$), and we have an output-sensitive guarantee for termination. Moreover, we can work with streaming features and we do not rely on the existence of a duality framework. 
Admittedly, our superset of the optimal support (corresponding to the set of extreme rays) could be loose when $p$ is not much larger than $n$, we cannot guarantee a persistent reduction for large values of $\eta$ in general problems (see the remarks right after \autoref{thm:main}), and the iteration cost for our screening procedure is higher than many of the duality-based screening methods. 
Furthermore, since our reduction is completely disentangled from any iterative optimization of the given problem, it is applicable to ultra-high dimensional problems where full-scale optimization (e.g., as required in safe elimination) is computationally prohibitive. Moreover, our proofs have allowed us to work with a broad class of non-convex, non-smooth, extended real-valued, and discontinuous loss functions, which would not have been possible had we required the use of Lagrangian duality (as in safe elimination.)

\para 
{\em Persistent relaxations} in integer programming and theoretical computer science, refer to relaxations yielding exact (as opposed to approximate) information on the true solution of the original problem. 
Our main result resembles persistency results in integer programming literature. As an example, \cite{nemhauser1975vertex} mentions ``Our most striking result is that those variables which assume binary values in an optimum (VLP) solution retain the same values in an optimum (VP) solution.'' 
\cite[Section 4]{hammer1984roof} mentions ``The practical significance of persistency as a tool for reducing the problem size is enhanced by the fact\dots that best roofs can be computed in polynomial time''. 
However, this notion of persistency is commonly studied as a property of `relaxations' to hard computational problems. Here, on the other hand, we study persistency in relation to a reduction/simplification of the original problem, helping with the runtime and memory requirements for optimization.

\para 
{\em Persistency} or risk consistency in statistical estimation, on a high level, is however different from our focus. We focus on identifying a true superset of the optimal supports while risk consistency is concerned with the optimal value in relation to the choices in the estimator.

\para 
{\em Geometry of the data}
has been studied in the statistics literature in its effect on estimation procedures. For example, \cite{el2018impact} and related works have studied `non-spherical' setups by assuming elliptic distribution for {\em predictors}. Also see \cite{diaconis1984asymptotics, hall2005geometric}. On the other hand, standardization (with various definitions), very common in regression, has an important effect on the geometry of datapoints (in $\mathbb{R}^n$). Our work is related to the latter category as it is clear from our main building block in \autoref{condn:ver-pos}. However, in contrast with many existing works, we do not make any probabilistic assumption on the data.


\bigskip

On one hand, and as partially discussed above, our result shares many properties with the aforementioned lines of work. 
On the other hand, we see our result as aligned more with understanding ``optimality conditions'' for regularized loss minimization problems (see the introductions in \cite{rockafellar1993lagrange,poliquin1998tilt} for discussions related to our goal), enabling a ``persistent'' reduction for these optimization problems. 
In dealing with any of the problems in this class and before performing any optimization or pre-processing, an extreme ray identification subroutine could be employed on the data to reduce the problem size without discarding any of the yet-unknown optimal set of variables. 
Persistent reductions could be viewed as a middle ground in between computing closed-form solutions (e.g., for ridge regression) and the invoking of general-purpose iterative optimization algorithms (e.g., gradient descent).

\paragraph{Acknowledgement.} 
We would like to thank Thomas Rothvoss for pointing out the resemblance between our superset guarantees and the half-integrality guarantees in the minimum vertex cover problem. This led us to the literature on persistency in integer programming mentioned in \autoref{sec:connections}. 


\newpage
\bibliographystyle{alpha}
\bibliography{PR-VS}

\appendix

\pagebreak
\section{Some Examples}\label{sec:facts}

\paragraph{Subdifferential Regularity.}
Here, we provide a non-exhaustive list of examples that are subdifferentially regular.

\begin{itemize}
\item
A proper convex function $f:\mathbb{R}^n\to \Rext$ is subdifferentially regular at any point in the domain of $f$ where it is locally lsc; \cite[Example 7.27]{rockafellar2009variational}. 

\item 
The maximum of a finite collection of (extended-value) smooth functions defined over a set $C$ is regular wherever $C$ is Clarke regular. Therefore, any smooth function is regular; \cite[Example 7.28]{rockafellar2009variational}. 

\item
Sum of separable functions, each of which regular, is regular; \cite[Proposition 10.5]{rockafellar2009variational}. 

\item 
See \cite[Theorem 10.6]{rockafellar2009variational} for regularity of compositions. As an example, consider $f = g \circ F$ for a proper, lsc function $g:\mathbb{R}^m \to \Rext$ and a smooth mapping $F:\mathbb{R}^n \to \mathbb{R}^m$, and consider a point $\vc{u}$ where $f$ is finite and the Jacobian $\nabla F(\vc{u})$ has rank $m$. Then, $f$ is regular at $\vc{u}$ if and only if $g$ is regular at $F(\vc{u})$; \cite[Exercise 10.7]{rockafellar2009variational}. 

\item 
Amenability implies regularity; \cite[Section 10.F]{rockafellar2009variational}. 

\item 
Also see \cite[Figure 8.2]{rockafellar2009variational} (an example of a discontinuous function) and \cite[Figure 8.3]{rockafellar2009variational} (on domain $x>-1$). 
See \cite{walther2019characterizing} for some further discussions on regularity.

\end{itemize}

\paragraph{Star-convexity.}
Any positively homogeneous function of any order $k\geq 1$ is star-convex with respect to the origin. Applying such function to a vector of star-convex function with the same center provides another star-convex function with the same center. For example, \cite{lee2016optimizing} mentions a few examples of this (without mentioning the more general statement) along with many more examples in \cite[Appendix A, B]{lee2016optimizing}\footnote{\url{https://arxiv.org/abs/1511.04466}}: 

\begin{itemize}

\item For any star-convex functions $f, g$, with global minima $f(\vc{0}) = g(\vc{0}) = 0$, the function $h(\vc{u}) = ( f(\vc{u})^q + g(\vc{u})^q )^{1/q}$ is star-convex, for any $q\in\mathbb{R}$; for $q=0$, $h(\vc{u}) = \sqrt{f(\vc{u})  \cdot g(\vc{u})}$. 

\item $\ell_q$ norm for any $q>0$; e.g., $f(x_1,x_2) = (\sqrt{\abs{x_1}}+ \sqrt{\abs{x_2}})^2$.

\item $h(\vc{u})= \norm{\vc{u}}_\alpha \cdot g( \vc{u}/ \norm{\vc{u}}_\beta)$ for any positive function $g$ defined on the boundary of the unit-norm ball for $\norm{\cdot}_\beta$.  
 
\end{itemize}


\section{On the True Interval for $\eta$}\label{rem:eta-bound}
\autoref{thm:main-detailed} relies on the existence of a set of parameters ($\psilow, \gamma, \dots$) which satisfy certain inequalities involving $\eta$; \autoref{condn:eta} provides a cleaner picture of this effect (in a less general situation.) 
Therefore, in general, our choice of parameters, which is a consequence of how well we understand the loss function at hand, determines how powerful the final guarantee will be. In many cases, we end up picking a convenient set of parameters since optimizing these intervals (requirements for $\eta$) is not easy. Nonetheless, this possibility of improvement provides hope for applicability of our reduction in situations where $\eta$ is above the upper bound in the theorem.

In the following, we consider the special case of least-squares loss and show that the true requirement on $\eta$ is looser than what we present in \autoref{thm:LS-gendata}. 
As it can be seen from the following calculations, while we have a convenient choice for the parameters unrelated to any properties of the data ($\mt{X}$ and $\vc{y}$), it might be possible to make better choices and provide stronger guarantees if we exploit properties of the data in our proofs.

Consider $f(\vc{u}) = \norm{\vc{u}-\vc{y}}_2^2$ where $\psilow=0$. Let us slightly change the notation to simplify the following presentation. For any choice of $\theta>0$ and $\vc{v}\in\mathbb{R}^n$, consider \autoref{condn:f-bnd} with $\gamma = 1/(\theta+1) \in (0,1)$ and $\cal{H} = (1-\gamma)\cdot\partial f(\vc{v})=\{2(1-\gamma)(\vc{v}-\vc{y})\}$. Therefore, $(1+\theta)H(\vc{u}) = 2\theta \langle \vc{u},\vc{v}-\vc{y}\rangle$. Moreover, define $\phi=\psiup-\psilow$ and $\hat{\eta}(\vc{v}) = \etaone/(1-\gamma)=r^\circ(2\mt{X}^T(\vc{y}-\vc{v}))$. With this new notation, $(\phi, \theta, \vc{v})$ is our new set of parameters. 
Let us elaborate on the quantities in \autoref{condn:f-bnd}. 
Define $\etamix \coloneqq (1+\theta) \phi - \theta \hat{\eta}$. In proving our main result in \autoref{thm:main-detailed}, we will need to require
\[
\eta \leq \etabnd \coloneqq \min \{ (1+\theta) \phi - \theta \hat{\eta},  \hat{\eta}\}. 
\] 
Note that both $\etamix$ and $\etabnd$ are functions of our choices for the parameters. Therefore, we seek a set of parameters for which the above upper bound is as large as possible; i.e., as non-restrictive on $\eta$ as possible. 
Consider the problem of maximizing $\etabnd$ over $\theta>0$, $\phi$, and $\vc{v}$. 
Note that the inequality in \autoref{condn:f-bnd} can be equivalently expressed as 
    \[
    \inf_\vc{u} \{ (\theta-1)\norm{\vc{u}}_2^2
    + (\theta+1)\norm{\vc{y}}_2^2
    -2\theta \langle \vc{u},\vc{v}\rangle \} \geq (1+\theta)\phi.  
    \]
    For such $\phi$ to exist we need either $(\theta=1, \vc{v}=\vc{0})$ or $\theta>1$. Let us elaborate on these two cases. 

\paragraph{Case 1.}
	In the first case, with $\theta=1, \vc{v}=\vc{0}$, we get $\phi \leq \norm{\vc{y}}_2^2$. Moreover, $\hat{\eta}(\vc{v}) = r^\circ(2\mt{X}^T\vc{y})$ and 
    \[
    \etamix = (1+\theta) \phi - \theta \hat{\eta}
    \leq 2  \norm{\vc{y}}_2^2 - r^\circ(2\mt{X}^T\vc{y}).
    \]
    Therefore, 
    \begin{align}\label{eq:etabnd-L2sq-case1}
    \etabnd 
    = \min \{ \hat{\eta}, \etamix \}
    \leq \min \{r^\circ(2\mt{X}^T\vc{y}),  
    2  \norm{\vc{y}}_2^2 - r^\circ(2\mt{X}^T\vc{y}) \}.
    \end{align}

\paragraph{Case 2.}
    In the second case, we get $\phi \leq \norm{\vc{y}}_2^2 - \frac{\theta^2}{\theta^2-1}\norm{\vc{v}}_2^2$. 
    Moreover, $\hat{\eta}(\vc{v}) = r^\circ(2\mt{X}^T(\vc{y}-\vc{v}))$ and 
    \begin{align}\label{eq:etamix-L2sq-case2}
    \etamix 
    &= (1+\theta) \phi - \theta \hat{\eta} \nonumber\\
    &\leq (\theta+1)\norm{\vc{y}}_2^2 - \frac{\theta^2}{\theta-1}\norm{\vc{v}}_2^2 
    - \theta  r^\circ(2\mt{X}^T(\vc{y}-\vc{v})) \nonumber\\
    &= (\theta+1)(\norm{\vc{y}}_2^2 - \norm{\vc{v}}_2^2) - \theta  r^\circ(2\mt{X}^T(\vc{y}-\vc{v}))
    - \frac{1}{\theta-1}\norm{\vc{v}}_2^2 .
    \end{align}
    Note that the right hand-side of \eqref{eq:etamix-L2sq-case2} is a concave function over $\theta>1$. 
    We have, 
    \begin{align}\label{eq:etabnd-L2sq-case2}
    \etabnd 
    = \min \{ \hat{\eta}, \etamix \}
    \leq \min \{
    &r^\circ(2\mt{X}^T(\vc{y}-\vc{v})),  \\
    &
    (\theta+1)(\norm{\vc{y}}_2^2 - \norm{\vc{v}}_2^2) - \theta  r^\circ(2\mt{X}^T(\vc{y}-\vc{v}))
    - \frac{1}{\theta-1}\norm{\vc{v}}_2^2   \},\nonumber
    \end{align}
    where the right-hand side of \eqref{eq:etabnd-L2sq-case2}, as a function of $\theta$ and defined over $\theta>1$, is a truncated concave function. 
	Define 
	\[
	\cal{V}\coloneqq \{\vc{v}:~ \norm{\vc{y}}_2^2 \geq \norm{\vc{v}}_2^2 + r^\circ(2\mt{X}^T(\vc{y}-\vc{v})) \, \} \ni \{\vc{y}\},
	\]
    which is a convex set. 
    If $\vc{v}\in\cal{V}$, then it is easy to see that there exists $\theta>1$ for which the two terms in the minimization in \eqref{eq:etabnd-L2sq-case2} become equal. Therefore, if $\vc{v}\in\cal{V}$, the maximum of the right-hand side of \eqref{eq:etabnd-L2sq-case2} over $\theta>1$ is equal to $r^\circ(2\mt{X}^T(\vc{y}-\vc{v}))$. 
    On the other hand, if $\vc{v}\not\in\cal{V}$, the first term in the minimization in \eqref{eq:etabnd-L2sq-case2} is inactive for all $\theta>1$. In that case, the maximum of the second term in the right-hand side of \eqref{eq:etabnd-L2sq-case2} over $\theta>1$, after some algebraic manipulations, is 
    \[
    \norm{\vc{y}}_2^2 - ( \norm{\vc{v}} + [  r^\circ(2\mt{X}^T(\vc{y}-\vc{v})) + \norm{\vc{v}}_2^2 - \norm{\vc{y}}_2^2   ]^{1/2} )^2.
    \]

\paragraph{All in all,} in maximizing the bound on $\etabnd$ over $\theta>0$, we get
\begin{align*}
	\max \etabnd 
	= \max_{\vc{v}\in\mathbb{R}^n} \Big\{ 
		&\min \{r^\circ(2\mt{X}^T\vc{y}), 2  \norm{\vc{y}}_2^2 - r^\circ(2\mt{X}^T\vc{y}) \} ~,~ \\
		&\max_{\vc{v}\in\cal{V}}~ r^\circ(2\mt{X}^T(\vc{y}-\vc{v})) ~,~\\
		&\max_{\vc{v}\not\in\cal{V}}~ \norm{\vc{y}}_2^2 - ( \norm{\vc{v}}_2 + [  r^\circ(2\mt{X}^T(\vc{y}-\vc{v})) + \norm{\vc{v}}_2^2 - \norm{\vc{y}}_2^2   ]^{1/2} )^2
	\Big\}.
\end{align*}
It is easy to see that if $\vc{0}\in\cal{V}$ (if $\vc{0}\not\in\cal{V}$), then the second (third) term in the maximization renders the first one inactive. Therefore,     
\begin{align}\label{eq:eta-bound-L2sq}
	\max \etabnd 
	= \max_{\vc{v}\in\mathbb{R}^n} \Big\{ 
		&\max_{\vc{v}\in\cal{V}}~ r^\circ(2\mt{X}^T(\vc{y}-\vc{v})) ~,~\\
		&\norm{\vc{y}}_2^2 - \min_{\vc{v}\not\in\cal{V}}~  (\norm{\vc{v}}_2 + \left[  r^\circ(2\mt{X}^T(\vc{y}-\vc{v})) + \norm{\vc{v}}_2^2 - \norm{\vc{y}}_2^2   \right]^{1/2} )^2
	\Big\}. \nonumber
\end{align}
In our main theorem, we suffice to a non-optimal bound resulting from the choice of $\vc{v}=\vc{0}$ and $\theta=1$; see \eqref{eq:etabnd-L2sq-case1}. Note that choosing $\vc{v}=\vc{y}$ leads to $\phi<0$ which is not useful. However, as it is clear from the above, a better choice for $(\theta,\vc{v})$, equivalently for $(\gamma, \cal{H})$ is possible.

In \autoref{fig:bestv-LS-n1}, we consider $n=1$, $x=1$, and $y=5$, and plot the corresponding value for each $v\in[-5,6]$ from \eqref{eq:eta-bound-L2sq}. Note that $\cal{V} = \{v:~ \abs{v-1}\leq \abs{y-1}\}$. The maximum is $16$ and is achieved at $v=1-\abs{y-1}$ which is different from $v=0$ we can conveniently choose together with $\theta=1$. 

\begin{figure}
	\centering
	\includegraphics[width=.27\textwidth]{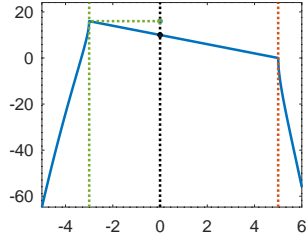}
	\caption{An illustration of the objective in \eqref{eq:eta-bound-L2sq} with $n=1, x=1, y=5$. The horizontal axis denotes $v$. The easy bound ($10$) and the best bound ($16$) have been marked on the vertical axis.}
	\label{fig:bestv-LS-n1}
\end{figure}


\pagebreak
\section{Details for \autoref{plt:sph-compare-etacv-add}}\label{sec:exp-eta-cv-details}
In \autoref{plt:sph-compare-etacv-add}, we consider values $p\in \op{round}(\{1.1^{50}, 1.1^{55}, \ldots, 1.1^{100}\})$ and $n\in \{18,21,\ldots,75\}$ where $\op{round}(\cdot)$ rounds each element of the set to its nearest integer. 

In the first three experiments, for each pair $(n,p)$, in each of the $1000$ random trials, we draw the columns of $\mt{X}$ independently from $N(\vc{0},\mt{I}_n)$. We also draw a noise vector $\vc{\epsilon}$ from $N(\vc{0},\mt{I}_n)$, and a vector $\vc{\beta}\in\mathbb{R}^p$ with $k=\op{round}(\sqrt{p})$ nonzero entries from $N(\vc{0},\mt{I}_k)$, all independent of $\mt{X}$. We normalized $\vc{\beta}$ to have a unit $\ell_2$ norm. We then form $\vc{y} = \mt{X}\vc{\beta}+ \sigma \vc{\epsilon}$ for $\sigma\in\{0.01, 0.1, 1\}$. 
In the fourth experiment, we sample $\vc{y}$ from $N(\vc{0},\mt{I}_n)$, and independent of $\mt{X}$. Then, in all four experiments, we normalize the columns of $\mt{X}$ as well as $\vc{y}$ to have unit $\ell_2$ norm. 
For each instance $(\mt{X},\vc{y})$, we execute the following command in software package R (\texttt{glmnet}, version 4.0-2),
\begin{verbatim}
	cvfit = (cv.glmnet(X, y, alpha = 1, standardize = FALSE, intercept = FALSE, 
                         nlambda = 100, type.measure = `mse', nfolds = 5))
\end{verbatim}
to perform a $5$-fold cross-validation. Due to a difference in constants in the formulation of \texttt{glmnet} and ours, we multiply the optimal output \texttt{cvfit\$lambda.min} by $2n$ to get the cross-validated $\eta_{\rm cv}$. \autoref{plt:sph-compare-etacv-add} illustrates regions of $(n,p)$ for which $\eta_{\rm cv}$ falls within the requirements of \autoref{thm:main-detailed}; the level of black in each pixel corresponds to the fraction of cases in which $\eta_{\rm cv}\leq 2-2\norm{\vc{y}^T\mt{X}}_\infty$; i.e., $\eta_{\rm cv}$ satisfies our requirements.

A few remarks are in order: 

\begin{itemize}

\item It is easy to see that the support to optimal solutions of \eqref{eq:lasso-pos} does not change if both $\vc{y}$ and $\eta$ are multiplied by a scalar $\theta>0$. Observe that $0 < \eta < 2\norm{\vc{y}^T\mt{X}}_\infty$ is equivalent to $0< \theta\eta < 2\norm{ (\theta\vc{y})^T\mt{X}}_\infty$. Therefore, the output of the above command and this experiment stays the same if we input $\vc{y}$ or $\theta\vc{y}$. Therefore, we have made the choice to standardize $\vc{y}$ before running the command. Note that \eqref{eq:cond-eta-thm1} is also invariant under the aforementioned transformation of the problem by $\eta$. 

\item
Similar experiments can be performed to compare our condition on $\eta$ in \autoref{thm:main} with $\eta$ chosen by other methods beyond $K$-fold cross-validation. For example, see the introduction in \cite{chetverikov2020cross} for some pointers to other methods. 

\item
There exist results in the literature on properties of the cross-validated regularization parameter that may aid our comparisons in stead of numerical simulations; e.g., see \cite{arlot2010survey}. However, reviewing such results is out of the scope of this manuscript and we suffice to this numerical demonstration. 

\item 

One can go beyond the comparison of $\eta_{\rm cv}$ and our threshold and compare the support of final solutions from cross-validation with the claimed superset of the supports in our theorem. As we show for one example of a loss, namely the least-squares loss, in \autoref{rem:eta-bound}, the true upper bound on $\eta$ (given in \eqref{eq:eta-bound-L2sq}) is larger than the simple upper bound given in the theorem. More specifically, in proving the theorem, we have chosen to work with a few convenient parameters while better choices may exist. Therefore, it would be informative to directly look at the optimal solution, from cross validation, and compare its support with the set of extreme rays, especially with real data. We leave such comparisons to future work. 

\end{itemize}

The boundary in the left-most plot in \autoref{plt:sph-compare-etacv-add}, corresponding to a low noise regime, appear to follow a behavior of $p= \exp(O(n^\kappa))$ for some $\kappa>1$. In the noiseless or low-noise regime, the only parameter we have is $k$, the sparsity of $\vc{\beta}$. Therefore, in \autoref{plt:sph-compare-etacv-lownoise}, we examine the effect of $k$. In \autoref{plt:sph-compare-etacv-add}, we use $k=\op{round}(\sqrt{p})$ while in \autoref{plt:sph-compare-etacv-lownoise} we use $k=20$ and $k=80$. Observe that the chance of $\eta_{\rm cv}$ falling into the requirements of our theorem slightly decreases as $k$ increases. 

\begin{figure}[h!]
    \centering
	    \begin{tikzpicture}
	  		\node[inner sep=0pt] (A)  {\includegraphics[width=.25\textwidth]{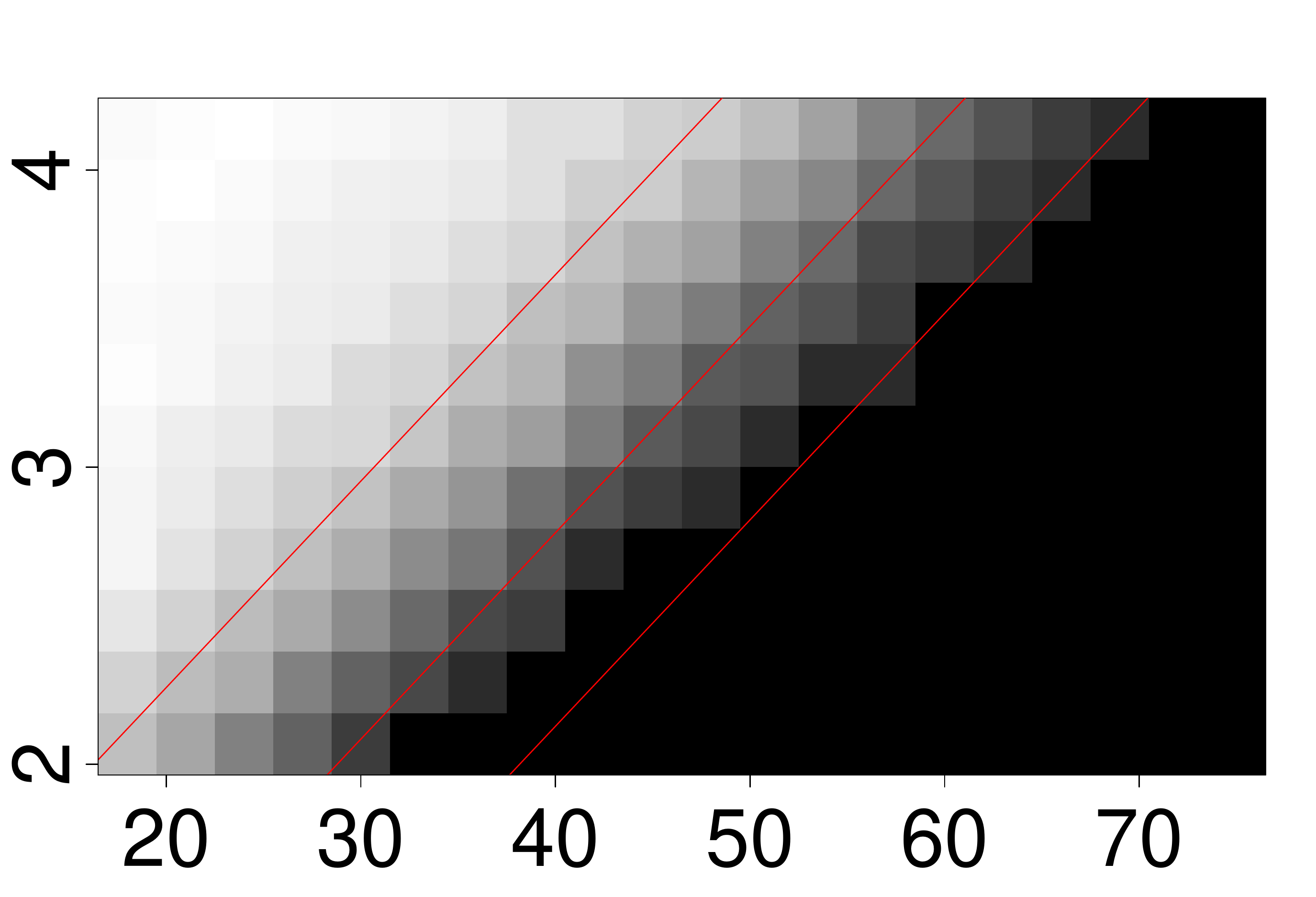}};
	  		\node[xshift=2.5pt,yshift=-14pt] (B) at ($(A.south)!.1!(A.north)$) {\small $n$};
			\node[rotate=90,yshift=15pt] (C) at ($(A.west)!.05!(A.east)$) {\small $\log_{10}(p)$};
	 	\end{tikzpicture}	~
	    \begin{tikzpicture}
	  		\node[inner sep=0pt] (A)  {\includegraphics[width=.25\textwidth]{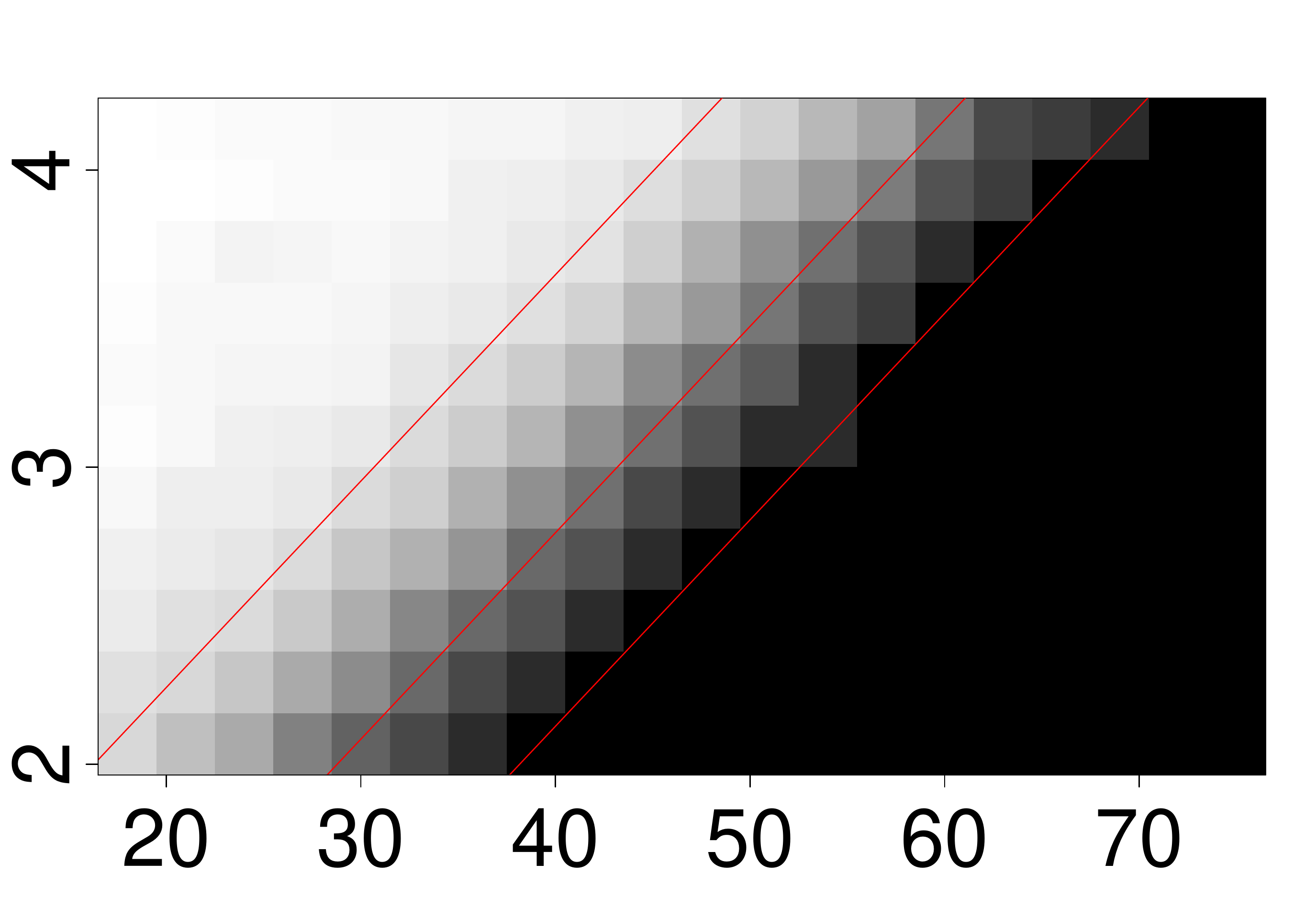}};
	  		\node[xshift=2.5pt,yshift=-14pt] (B) at ($(A.south)!.1!(A.north)$) {\small $n$};
	 	\end{tikzpicture}	
	\caption{The percentage of cases, out of $1000$ random trials, where the $5$-fold cross-validated $\eta_{\rm cv}$ for lasso problem satisfies the requirement of our theorem. Black corresponds to $1$ and white corresponds to $0$. The horizontal and vertical axes correspond to $n$ and $\log_{10}(p)$, respectively. We use $\vc{y} = \mt{X}\vc{\beta}$ and we normalize the columns of $\mt{X}$ and $\vc{y}$ before being fed into the lasso solver. In the left (right) picture $\vc{\beta}$ has $k=20$ ($k=80$) nonzero entries, and is drawn uniformly from the unit sphere.}
	\label{plt:sph-compare-etacv-lownoise}
\end{figure}
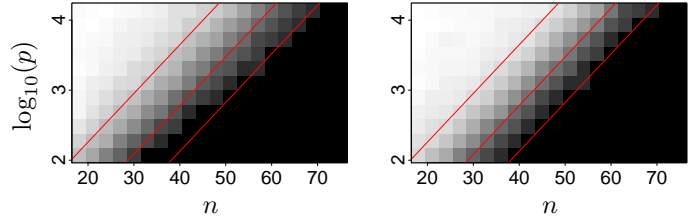

\section{Some Illustrations for \autoref{sec:originT}}\label{sec:illust}

In this section, we provide a few illustrations regarding the results of \autoref{sec:originT}, ans specifically in relation to the summary in \autoref{cor:Fy-subset-T-ifnoncover}.

\begin{figure}[h]
	\centering
	\begin{tikzpicture}
	\node[inner sep=0pt] (c1p) at (0,0)
	{\includegraphics[clip, trim=9cm 3.5cm 6.5cm 5cm, width=.3\textwidth]{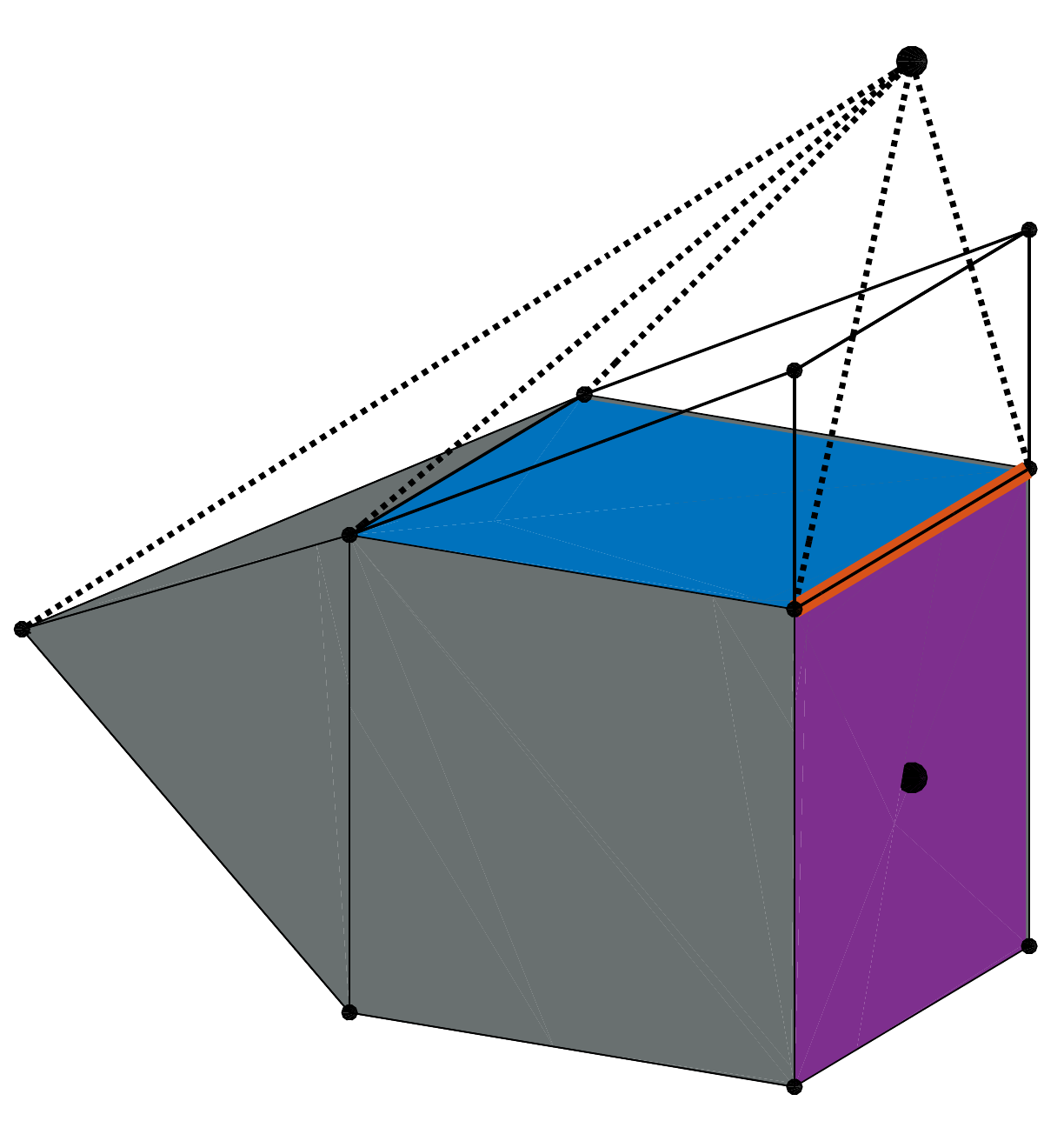}};

	\node[inner sep=0pt] (y) at (c1p.north) [xshift=50pt, yshift=4pt] {$y$};

	\end{tikzpicture}

	\caption{A convex body $\Kc$ with origin on its boundary as well as a point $\vc{y}$ with $\gauge(\vc{y};\Kc)>1$. $\cal{F}(\vc{y})$ is an edge of $\Kc$ in red, $\msh(\vc{y})$ is horizontal facet of $\Kc$ in blue, $\mshbar(\vc{y})$ is the union of $\msh(\vc{y})$ with a vertical facet passing through the origin in purple; $\op{ver}(\cal{F}(\vc{y})) \subsetneq \msh(\vc{y}) \subsetneq \mshbar(\vc{y})$.  \autoref{cond:verKY-notcover} is satisfied for all $\theta\vc{y}$ where $\theta > 1/\gauge(\vc{y};\Kc)$. 
	$\Kc[\backslash\vc{y}]$ is the union of $\Kc$ and the hollow structure above the blue facet. 
	$\gauge(\vc{y}; \Kc[\backslash\vc{y}])>1$, $\msh(\vc{y}) \subsetneq \cal{X}(\vc{y})$, $\cal{X}(\vc{y})\not\subseteq \mshbar(\vc{y})$. For $ \gauge(\vc{y}; \Kc[\backslash\vc{y}]) < 1/\theta < \gauge(\vc{y}; \Kc)$, we have $\op{ver}(\cal{F}(\theta\vc{y})) \subsetneq \msh(\theta\vc{y}) = \cal{X}(\theta\vc{y})\subsetneq \mshbar(\theta\vc{y})$. }
	\label{fig:cube1pyramid}
\end{figure}

\begin{figure}
	\centering
	\begin{tikzpicture}
	\node[inner sep=0pt] (c2p) at (0,0)
	{\includegraphics[clip, trim=4cm 4.8cm 3.3cm 4.7cm, width=.5\textwidth]{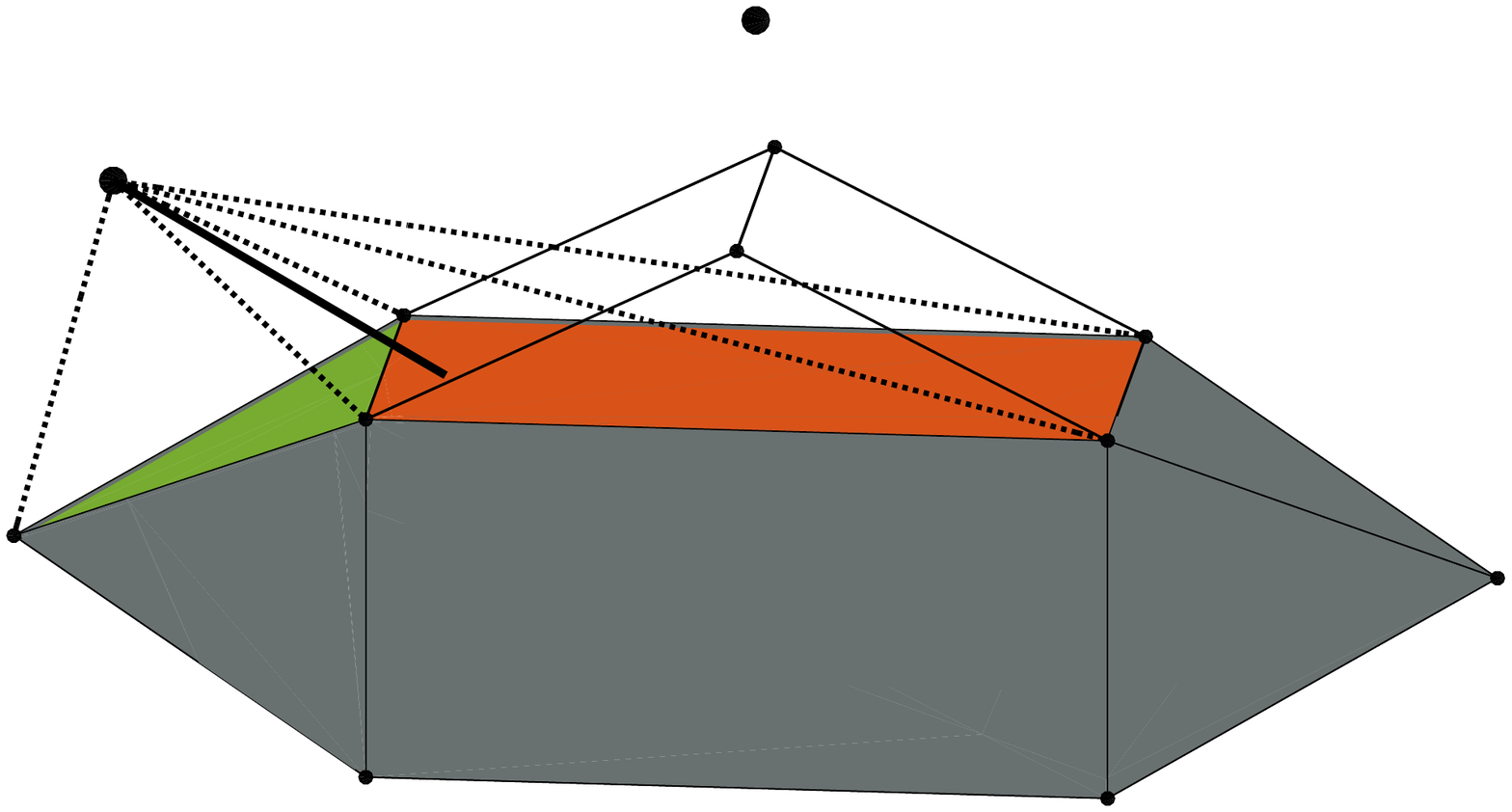}};
	
	\node[inner sep=0pt] (y) at (c2p.north) [xshift=2pt, yshift=-1pt] {$y$};
	\node[inner sep=0pt] (z) at (c2p.north west) [xshift=20pt, yshift=-22pt] {$z$};
	\end{tikzpicture}
	
	\caption{A convex body $\Kc$ symmetric around the origin as well as two points $\vc{y}$ and $\vc{z}$ with $\gauge(\vc{y};\Kc)>1$ and $\gauge(\vc{z};\Kc)>1$. 
	$\cal{F}(\vc{y})=\msh(\vc{y})=\mshbar(\vc{y})
	= \cal{F}(\vc{z})=\msh(\vc{z})=\mshbar(\vc{z})$ is the horizontal facet in red. 
	$\Kc[\backslash\vc{y}]=\Kc[\backslash\vc{z}]$ is the union of $\Kc$ and the hollow structure above the red facet.  
	$\gauge(\vc{z}; \Kc[\backslash\vc{z}])>1$, $\msh(\vc{z}) \subsetneq \cal{X}(\vc{z})\not\subseteq \mshbar(\vc{z})$. Consider lasso with vertices of $\Kc$ as columns of $\mt{X}$, and with $\vc{z}$. For small values of $\eta$, the support of optimal solutions is in the red facet. For large values of $\eta$, the support of optimal solutions is in the green triangular facet. The extreme rays cover vertices of both facets.}
	\label{fig:cube2pyramid}
\end{figure}

\begin{figure}
	\centering
	\begin{tikzpicture}
	\node[inner sep=0pt] (c02p) at (0,0)
	{\includegraphics[clip, trim=4cm 4.8cm 3.3cm 4.7cm, width=.5\textwidth]{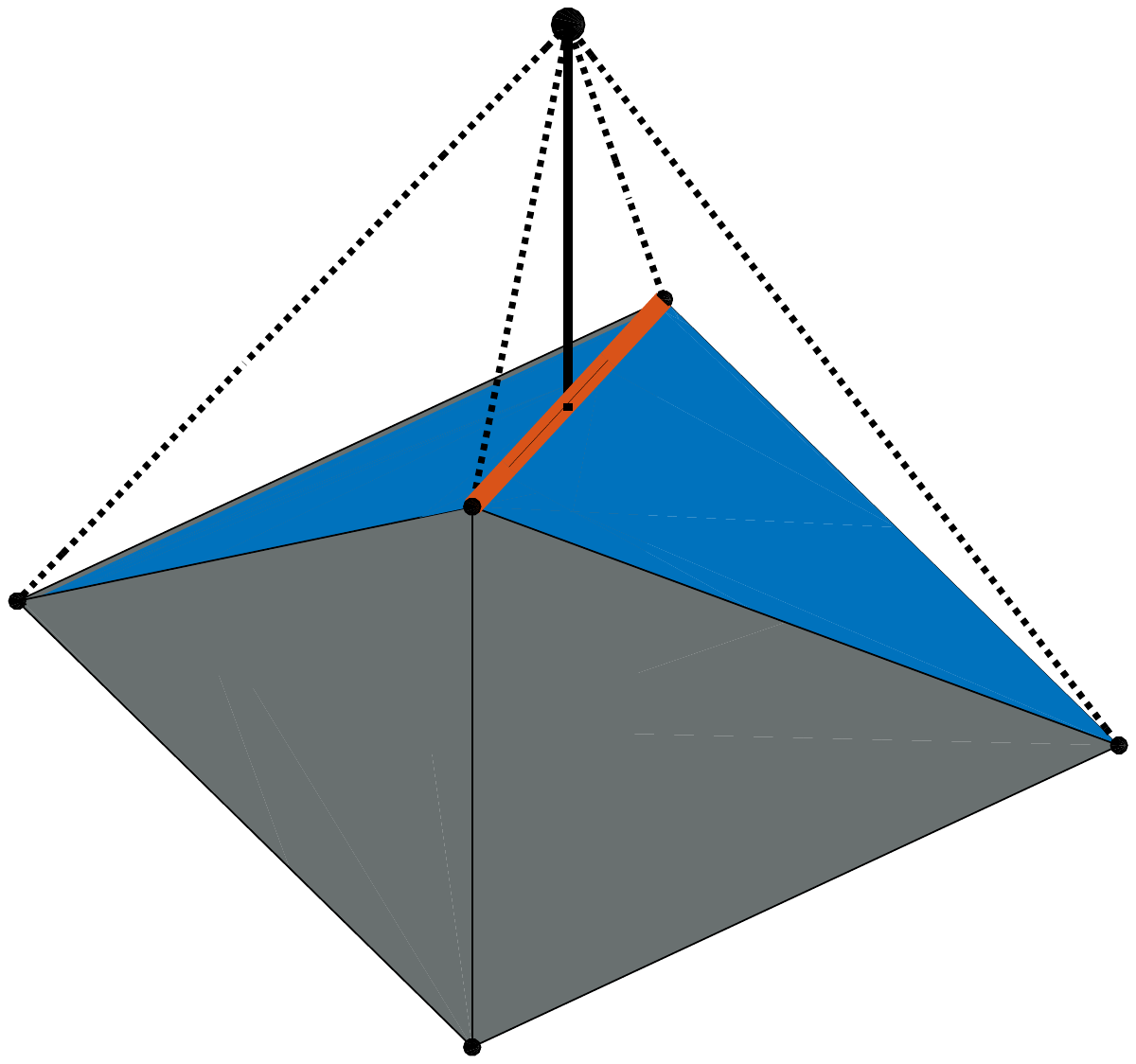}};
	
	\node[inner sep=0pt] (y) at (c02p.north) [xshift=2pt, yshift=0pt] {$y$};
	\end{tikzpicture}
	
	\caption{A convex body $\Kc$ symmetric around the origin as well as a point $\vc{y}$ with $\gauge(\vc{y};\Kc)>1$. 
	$\cal{F}(\vc{y})$ is an edge of $\Kc$ in red, $\msh(\vc{y})=\mshbar(\vc{y})$ is the union of two top facets in blue; $\op{ver}(\cal{F}(\vc{y})) \subsetneq \msh(\vc{y}) = \mshbar(\vc{y})$.  \autoref{cond:verKY-notcover} is satisfied for all $\theta\vc{y}$ where $\theta > 1/\gauge(\vc{y};\Kc)$. 
	$\Kc[\backslash\vc{y}]$ is the upward cone emanating from the bottom vertex of $\Kc$ along the facets of $\Kc$.  
	$\gauge(\vc{y}; \Kc[\backslash\vc{y}]) = 0$, $\msh(\vc{y}) =\mshbar(\vc{y})= \cal{X}(\vc{y})$. In this picture, a small random spherical perturbation of $\vc{y}$ to $\bar{\vc{y}}$ results in, with high probability, $\op{ver}(\cal{F}(\bar{\vc{y}})) = \msh(\bar{\vc{y}}) = \mshbar(\bar{\vc{y}})$ and $\gauge(\bar{\vc{y}}; \Kc[\backslash\bar{\vc{y}}])>1$. 
	 }
	\label{fig:2pyramid}
\end{figure}


\end{document}